\documentclass[11pt]{article}

\usepackage{lmodern}
\usepackage{float}
\usepackage{amsfonts}    
\usepackage{color}
\usepackage{multirow}
\usepackage{graphicx}
\usepackage{amssymb,amsmath,amsthm,bm,bbold}
\usepackage{enumitem}
\usepackage{fullpage}
\usepackage{appendix}
\usepackage{cite}
\usepackage{tikz}
\usetikzlibrary{arrows.meta}

\usepackage{algpseudocode}
\usepackage{algorithm}
\algblock[Init]{Initialize}{EndInitialize}
\usepackage{booktabs}

\usepackage{hyperref}
\usepackage[capitalize,noabbrev]{cleveref}
\Crefname{equation}{Eq.}{Eqs.}
\crefname{equation}{eq.}{eqs.}
\crefname{section}{section}{section}
\Crefname{section}{Section}{Sections}
\crefname{appendix}{appendix}{appendices}
\Crefname{appendix}{Appendix}{Appendices}
\crefname{table}{Table}{Tables}
\Crefname{table}{Table}{Tables}
\crefname{figure}{Fig.}{Fig.}
\Crefname{figure}{Figure}{Figures}
\crefname{algorithm}{Alg.}{Alg.}
\Crefname{algorithm}{Algorithm}{Algorithms}
\crefname{theorem}{Theorem}{Theorems}
\Crefname{theorem}{Theorem}{Theorems}
\crefname{corollary}{Corollary}{Corollaries}
\Crefname{corollary}{Corollary}{Corollaries}
\crefname{lemma}{Lemma}{Lemma}
\Crefname{lemma}{Lemma}{Lemmas}
\crefname{remark}{Remark}{Remarks}
\Crefname{remark}{Remark}{Remarks}
\crefname{proposition}{Proposition}{Propositions}
\Crefname{proposition}{Proposition}{Propositions}
\crefname{algorithm}{Algorithm}{Algorithms}  

\newtheorem{proposition}{Proposition}
\newtheorem{definition}{Definition}
\newtheorem{corollary}{Corollary}
\newtheorem{lemma}{Lemma}

\renewcommand{\vec}[1]{\bm{#1}}
\newcommand{\mat}[1]{\mathbf{#1}}
\newcommand{\red}[1]{\textcolor[rgb]{0.635,0.0780,0.1840}{#1}}
\newcommand{\blue}[1]{\textcolor[rgb]{0,0.33,0.56}{#1}}

\definecolor{shade}{rgb}{0.93,0.93,0.93}

\title{\bf 
 \mbox{Subzero matrix completion for sparse data analysis:}
  large-scale learning of latent low-rank structure}

\author{Lawrence K. Saul$^1$, Ningyuan Huang$^{1}$\thanks{Now at Google Research}, Dennis Bollweg$^2$, \\ Jeff Soules$^1$, 
and  Diana C. Halikias$^3$ \\[1ex]
\small $^1$Center for Computational Mathematics, Flatiron Institute, New York, NY 10010 \\
\small $^2$Scientific Computing Core, Flatiron Institute, New York, NY 10010 \\
\small $^3$Courant Institute, NYU, New York, NY 10012}

\date{August 21, 2026}

\begin{document}

\maketitle

\begin{abstract}
\noindent
We investigate when a sparse nonnegative matrix can be recovered from a real-valued matrix of much lower rank by zeroing out its negative elements. The potential for such decompositions suggests a mathematical connection between sparsity and rank; we analyze a number of sparse matrices with this latent low-rank structure and use them to illustrate the geometric origins of this connection. Previous algorithms have discovered these decompositions via an alternating minimization over the factors of a low-rank matrix, but to do so, they have also needed to compute and store another matrix, neither sparse nor low-rank, that is the size of their product. We develop a stochastic, alternating least-squares algorithm that operates on smaller blocks of this dense matrix and scales as a result to much larger problems. We also show how to further accelerate this algorithm with sparse optimizations and customized CUDA kernels. As one example, we use the algorithm to analyze the sparse matrix of synaptic weights for the recently published \textit{Drosphilia} connectome. The nonzero elements of this matrix, with 139,255 rows and columns, record the
number of synapses between cells in the nervous system of a female fruit fly. Despite a slowly decaying spectrum of singular values, this matrix exhibits a latent low-rank structure that is predictive of cell categories across multiple levels of specificity.
\end{abstract}


\section{Introduction}

There are two types of structure that arise often in large matrices of data. The first is sparsity, where relatively few elements have nonzero values; sparse matrices are used, for example, to record the links in social and biological networks~\cite{newman2003-networks}, the co-occurrence statistics in natural language~\cite{michel2011-ngrams}, the neural responses to natural stimuli~\cite{olshausen1996-emergence}, and the nearest-neighbor relations in large data sets~\cite{papadopoulos2004-nearest}. The second is a deficiency in rank, where not all the rows or columns are linearly independent: many matrices in data science are approximately of low rank~\cite{udell2019-rank,budzinskiy2025-rank}, and low-rank factorizations have been used in many different ways---to analyze text documents~\cite{deerwester1990-indexing} and gene expression levels~\cite{brunet2004-gene}, to learn recommender systems~\cite{koren2009-matrix} and large language models~\cite{hu2022-lora}, and more generally, to reduce the dimensionality of data for downstream tasks~\cite{turk1991-eigenfaces,jolliffe2016-pca}.

It is not obvious that these two types of structure should be mathematically related in any way. Clearly there are sparse matrices, such as the identity matrix, that are full rank, and there are rank-one matrices, such as the matrix of all ones, that are entirely dense. Some recent studies~\cite{saul2022-nmd,saul2022-geometrical}, however, have explored the following connection between sparsity and rank: for a sparse \textit{nonnegative} matrix $\mat{S}$, it is often possible to find a \textit{real-valued} matrix~$\mat{L}$ of much lower rank such that \mbox{$\mat{S}\approx\max(0,\mat{L})$}; here, the max operation is applied elementwise to the entries of $\mat{L}$. An example of this basic idea is illustrated in \cref{fig:sumac}, and note that the sparser the matrix~$\mat{S}$, the more zeros can be replaced by negative values to lower the rank of $\mat{L}$. An even more striking example~\cite{alon2016-sign} (to which we return later) is provided by the $n\!\times\! n$ identity matrix $\mat{I}_n$; for all $n\!\geq\!3$, there exists a corresponding low-rank matrix $\mat{L}_n$ such that $\mat{I}_n\! =\! \max(0,\mat{L}_n)$ and $\text{rank}(\mat{L}_n)\!\leq\! 3$.

\begin{figure}[t]
\centerline{\includegraphics[width=\textwidth]{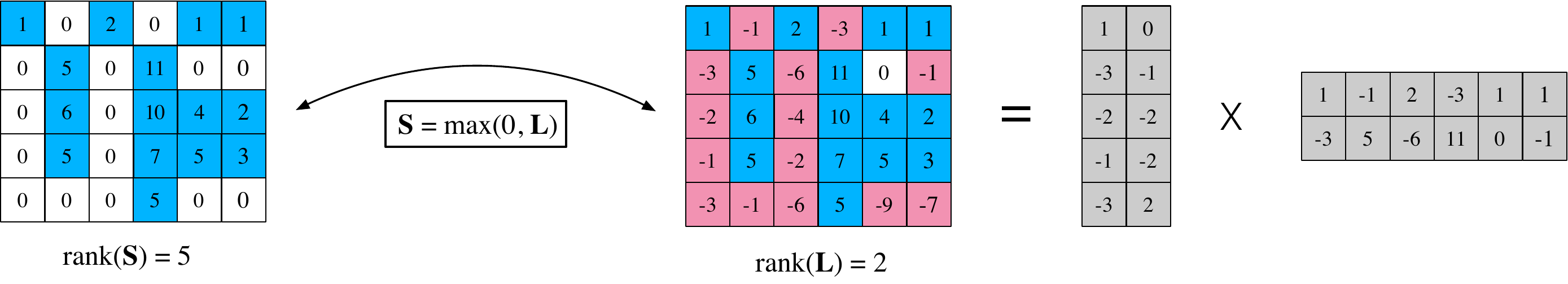}}
\caption{A sparse nonnegative matrix $\mat{S}$ of full rank (\textit{left}) can be transformed into a matrix $\mat{L}$ of lower rank (\textit{right}) by replacing zeros in $\mat{S}$ with negative values in $\mat{L}$ (shown in red). Since $\mat{S}$ and~$\mat{L}$ share the same positive entries (shown in blue), the latter provides a low-rank encoding of the former via  $\mat{S}=\max(0,\mat{L})$. In this case we say that $\mat{L}$ is a \textit{subzero} matrix completion of~$\mat{S}$.}
\label{fig:sumac}
\end{figure}

Subsequently many researchers have developed faster, more efficient, and provably convergent algorithms to learn these decompositions~\cite{seraghiti2023-accelerated,liu2024-symmetric,wang2024-momentum,awari2025-alternating,gillis2025-extrapolated,wang2025-efficient,wang2025-accelerated}. Given a sparse nonnegative matrix~$\mat{S}$, and a prescribed rank $r$, these algorithms attempt to compute a matrix $\mat{L}$ such that $\text{rank}(\mat{L})\!=\!r$ and $\mat{S}\!\approx\! \max(0,\mat{L})$. Despite this ongoing progress, there remain gaps in our understanding and use of these models. Conceptually, these nonlinear decompositions of sparse matrices are more difficult to interpret than linear ones, while on a practical level, current implementations of these algorithms have generally assumed that there is enough memory to store a \textit{dense} matrix of the same size as $\mat{S}$. Such implementations do not scale well to applications involving very large sparse matrices, with (say) hundreds of thousands of rows and columns.

This paper aims to fill some of these gaps. We start by providing an intuitive geometric picture for this model of sparse data analysis and studying various idealized matrices where one can obtain analytical results. Next we describe a stochastic alternating least-squares algorithm for this model that scales well to large problem sizes. We also provide GPU-compatible software packages~\cite{sumac2026-github} in Python and MATLAB that are supported by dedicated and extensively optimized CUDA kernels. As an illustrative example, we use the algorithm to analyze the latent low-rank structure of the recently published \textit{Drosphilia} connectome~\cite{dorkenwald2024-flywire,matsliah2023-codex}; the connectome is stored as a sparse nonnegative matrix, with over 139,000 rows and columns, whose $ij^\text{th}$ element indicates the number of synapses from the $i^\text{th}$ neuron to the~$j^\text{th}$ neuron in the brain of a female fruit fly. On problems of this size, we find not only that stochastic optimizers are much faster than deterministic ones, but also that the customized CUDA kernels lead to further significant speedups (e.g., up to 4x in Python and over 10x in MATLAB).
These implementations should be of use to many researchers in science and engineering. 

The organization of this paper is as follows. In \cref{sec:motivation}, we prove some basic propositions that show when these nonlinear low-rank decompositions exist and what they signify.  We also study some structured sparse matrices where one can develop a further analytical understanding and preview the matrices from image, brain, and text data that we aim to study empirically. In \cref{sec:alg}, we review the basic framework of alternating minimization that is most commonly used to discover these decompositions. In \cref{sec:scaling}, we explore a number of methods for scaling this optimization up to very large problem sizes. In \cref{sec:experiment}, we describe which of these methods are most complementary in practice and demonstrate their ability to discover latent low-rank structure in sparse matrices with hundreds of thousands of rows and columns. Finally, in \cref{sec:related}, we give a broader review of related work, and in \cref{sec:discuss}, we summarize our main findings and discuss directions for future research.


\section{Motivation and examples}
\label{sec:motivation}

In this section we develop a geometric picture for the nonlinear matrix decompositions in \cref{fig:sumac}. We start from our previous understanding of linear decompositions---namely, when a matrix $\mat{L}$ is low rank, it can be written $\mat{L}\!=\!\mat{A}\mat{B}^\top$ where $\mat{A}$ and $\mat{B}$ are tall skinny matrices, or equivalently, we can write $L_{ij}\!=\! \vec{a}_i\cdot\vec{b}_j$, where the vectors~$\vec{a}_i$ and $\vec{b}_j$ refer, respectively, to the $i^\text{th}$ and $j^\text{th}$ corresponding rows in $\mat{A}$ and~$\mat{B}$. Often these vectors can be useful in downstream tasks; intuitively, they provide low-dimensional representations of the data that are indexed by the rows and columns of $\mat{L}$. We can also regard~$\mat{L}$ as a \textit{similarity matrix} for these vectors, where similarities are measured by inner products. Our goal is to prove some elementary propositions that extend these intuitions to the case where $\mat{S}\!=\!\max(0,\mat{L})$ for a sparse nonnegative matrix~$\mat{S}$.

\subsection{Geometric and graphical constructions}

There is a large literature on low-rank completions of real-valued matrices with \textit{missing elements}~\cite{candes2009-exact,candes2010-matrix,keshavan2010-matrix,chatterjee2015-matrix,sun2016-guaranteed,nguyen2019-low,chatterjee2020-deterministic,zilber2022-GNMR}. The zeros of sparse matrices are not missing elements per se, but there are obvious parallels between earlier methods for matrix completion and the ideas in this paper. This motivates the following definition.

\noindent\fcolorbox{black}{shade}{\parbox{0.985\textwidth}{\vspace{-2ex}
\begin{definition} We say that $\mat{L}$ is a \textit{subzero matrix completion} of a sparse nonnegative matrix $\mat{S}$ if $S_{ij}=\max(0,L_{ij})$ for all of the elements in these matrices.
\end{definition}
\vspace{-2ex}}}

 Here we are interested in two questions: (i) when do sparse nonnegative matrices have \textit{subzero} matrix completions of low rank, and (ii) when such completions exist, what geometric picture do they provide? The next definition and propositions give some basic answers to these questions.

\vspace{1ex}
\noindent\fcolorbox{black}{shade}{\parbox{0.985\textwidth}{\vspace{-2ex}
\begin{definition}
Let $\{\vec{\alpha}_i\}_{i=1}^m$ and $\{\vec{\beta}_j\}_{j=1}^n$ be sets of vectors in $\mathbb{R}^d$. We define the thresholded similarity matrix for these vectors, with threshold $\tau\!\in\!(-1,1)$, as the $m\times n$ nonnegative matrix~$\mat{S}$ with elements
\begin{equation}
S_{ij}(\tau) = \max\big(0,\vec{\alpha}_i\!\cdot\!\vec{\beta}_j - \tau\|\vec{\alpha}_i\|\|\vec{\beta}_j\|\big).
\label{eq:tsm}
\end{equation}
\end{definition}
\vspace{-2ex}}}

For brevity in what follows, we write each matrix element in \cref{eq:tsm} as simply $S_{ij}$, rather than~$S_{ij}(\tau)$, omitting the dependence on the threshold $\tau$. Note that this matrix element is positive if and only if the angle between vectors $\vec{\alpha}_i$ and $\vec{\beta}_j$ has a cosine greater than $\tau$. Thus $\mat{S}$ can be viewed as a thresholded similarity matrix based on cosine distances, and the larger the value of $\tau$, the sparser is the matrix $\mat{S}$. In particular, if $\tau$ is close to unity, then the positive elements in $\mat{S}$ indicate precisely those pairs of vectors in $\{\vec{\alpha}_i\}_{i=1}^m$ and $\{\vec{\beta}_j\}_{j=1}^n$ that are nearly parallel. Later we will see many examples of such matrices that are both highly sparse and of full rank even when $d\ll\min(m,n)$.

The geometric intuition behind $\mat{S}=\max(0,\mat{L})$ is based on a correspondence between thresholded similarity matrices and subzero matrix completions. The next proposition establishes one direction of this correspondence.

\vspace{1ex}
\noindent\fcolorbox{black}{shade}{\parbox{0.985\textwidth}{\vspace{-2ex}
\begin{proposition}
\label{thm:tsm1}
Let $\{\vec{\alpha}_i\}_{i=1}^m$ and $\{\vec{\beta}_j\}_{j=1}^n$ be sets of vectors in $\mathbb{R}^d$. If $\mat{S}$ is given by the $m\!\times\! n$ thresholded similarity matrix in \cref{eq:tsm}, then $\mat{S}$ has a subzero matrix completion of rank $r\leq d\!+\!1$.
\end{proposition}
\vspace{-2ex}}}
\vspace{-1ex}

\begin{proof}
The proof is by construction. 
For each row of $\mat{S}$, define a vector $\vec{a}_i$ in $\mathbb{R}^{d+1}$
by \mbox{$\vec{a}_i = (\vec{\alpha}_i,\tau\!\left\|\vec{\alpha}_i\right\|)$},
and for each column of $\mat{S}$, define a vector $\vec{b}_j$ in $\mathbb{R}^{d+1}$ by $\vec{b}_j = (\vec{\beta}_j,-\!\left\|\vec{\beta}_j\right\|)$.
Finally let $L_{ij}\!=\!\vec{a}_i\!\cdot\!\vec{b}_j$. Then $\mat{L}$ is a matrix of rank at most $d\!+\!1$ satisfying $S_{ij}=\max(0,L_{ij})$.
\end{proof}

This result shows how thresholded similarity matrices give rise to subzero matrix completions of very low rank; in particular, there may be a large gap in rank between $\mat{L}$ and $\mat{S}$ in this setting whenever $d\!\ll\!\min(m,n)$ (that is, when the vectors inhabit a space whose dimensionality is far lower than the number of vectors whose similarities are being computed). The next proposition establishes the other direction of this correspondence.

\vspace{1ex}
\noindent\fcolorbox{black}{shade}{\parbox{0.985\textwidth}{\vspace{-2ex}
\begin{proposition}
\label{thm:tsm2}
Let $\mat{L}$ be a subzero matrix completion of $\mat{S}$ with $\textup{rank}(\mat{L})\!=\!r$. Then for any $\tau\!\in\!(-1,1)$, there exist vectors $\{\vec{\alpha}_i\}_{i=1}^m$ and  $\{\vec{\beta}_j\}_{j=1}^n$ in~$\mathbb{R}^{r+1}$ such that $\mat{S}$ equals the thresholded similarity matrix in \cref{eq:tsm}.
\end{proposition}
\vspace{-2ex}}}

\begin{proof}
The proof is by construction. Since $\mat{L}$ is of rank $r$, there exist vectors $\{\vec{a}_i\}_{i=1}^m$ and $\{\vec{b}_j\}_{j=1}^n$ in $\mathbb{R}^r$ such that $L_{ij} = \vec{a}_i\cdot\vec{b}_j$. Let $\gamma$ be a solution of $\gamma^2(1\!-\!|\tau|)=|\tau|$. For each row of~$\mat{S}$, define a vector $\vec{\alpha}_i$ in $\mathbb{R}^{r+1}$ by $\vec{\alpha}_i = \left(\vec{a}_i,\gamma\left\|\vec{a}_i\right\|\text{sign}(\tau)\right)$,
and for each column of~$\mat{S}$, define a vector  $\vec{\beta}_j$ in $\mathbb{R}^{r+1}$ by
$\vec{\beta}_j = \left(\vec{b}_j,\gamma\left\|\vec{b}_j\right\|\right)$.
From these definitions, it follows that
\begin{eqnarray*}
  \vec{\alpha}_i\!\cdot\!\vec{\beta}_j - \tau\|\vec{\alpha}_i\|\|\vec{\beta}_j\|
    & = & \vec{a}_i\!\cdot\!\vec{b}_j\, +\, \gamma^2 \|\vec{a}_i\|\|\vec{b}_j\|\,\text{sign}(\tau)\, -\, \tau\left(1\!+\!\gamma^2\right)\|\vec{a}_i\|\|\vec{b}_j\|, \\
    & = & L_{ij} +\, 
        \text{sign}(\tau)\|\vec{a}_i\|\|\vec{b}_j\|\, [\gamma^2(1\!-\!|\tau|) - |\tau|], \\
     & = & L_{ij}. 
\end{eqnarray*}
Since by assumption $S_{ij}\!=\!\max(0,L_{ij})$, it follows from the preceding equation that $\mat{S}$ is equal to the thresholded similarity matrix in \cref{eq:tsm}.
\end{proof}

The above results provide the following geometric picture for this model of sparse data analysis: a sparse nonnegative matrix $\mat{S}$ may not itself be low rank, but \textit{it will be amenable to a low-rank subzero matrix completion if to each row and column we can associate low-dimensional vectors such that the nonzero elements of $\mat{S}$ encode those pairs of vectors that are most nearly parallel.} 
The above propositions generalize earlier results for subzero matrix completions of symmetric thresholded similarity matrices~\cite{saul2022-geometrical}. We note that the constructions in these proofs are not unique.

Subzero matrix completion is based on a simple idea: the sparser the nonnegative matrix $\mat{S}$, the more zeros we can replace by negative values to lower the rank of $\mat{L}$. This suggests, conversely, that less sparse matrices may be less amenable to these manipulations. The last proposition in this section makes this idea more precise; specifically, it provides a lower bound on the rank of $\mat{L}$ for any subzero matrix completion satisfying $\mat{S} = \max(0,\mat{L})$.

\vspace{1ex}
\noindent\fcolorbox{black}{shade}{\parbox{0.985\textwidth}{\vspace{-2ex}
\begin{proposition}
\label{thm:rankL_bound}
Suppose that $\mat{S} = \max(0,\mat{L})$. Let $\mat{S}_+$ be any submatrix of $\mat{S}$ with all positive elements. Then $\textup{rank}(\mat{L}) \geq \textup{rank}(\mat{S}_+)$.
\end{proposition}
\vspace{-2ex}}}

\vspace{-1ex}
\begin{proof}
The positive elements of $\mat{S}$ are not changed by the process of subzero matrix completion; hence any all-positive submatrix of $\mat{S}$ will also be a submatrix of $\mat{L}$. Since the rank of a matrix is never less than the rank of any submatrix, it follows that $\text{rank}(\mat{L}) \geq \text{rank}(\mat{S}_+)$.
\end{proof}

We make two final observations regarding this result: (i) if $\mat{S}$ is a binary matrix of zeros and ones, then the lower bound in \cref{thm:rankL_bound} is vacuous (because in this case any positive submatrix of~$\mat{S}$ has rank one); (ii) if $\mat{S}$ is a nonnegative matrix with $p$ nonzeros, then the rank of any positive submatrix is at most $\sqrt{p}$. 

Our next goal is to provide some examples that illustrate the geometric content of these propositions. The simplest examples are motivated by graphical constructions, of which the next is well known.

\vspace{1ex}
\noindent\fcolorbox{black}{shade}{\parbox{0.985\textwidth}{\vspace{-2ex}
\begin{definition}
Let $\mat{W}$ be the weighted adjacency matrix of a symmetric graph. The \textit{signless Laplacian} is defined as the matrix $\mat{Q}=\mat{D}\!+\!\mat{W}$, where $\mat{D}$ is the diagonal matrix with elements $D_{ii} = \sum_{j\neq i} W_{ij}$.
\end{definition}
\vspace{-2ex}}}

Many results are known about signless Laplacians from spectral graph theory~\cite{chung1997-spectral,cvetkovic2007-signless}. For our purposes, the most important is that all of the eigenvalues of a signless Laplacian are positive unless the underlying graph is bipartite; in this special case, there is exactly one zero eigenvalue. Of course, the signless Laplacian of a sparse graph is also a \textit{sparse nonnegative matrix}, and from the preceding observation, it is also one of full or nearly full rank.  Thus it is natural to ask when signless Laplacians have subzero matrix completions of especially low rank and to study the properties of sparse graphs that are revealed when they do. The next sections do this for the graphs in \cref{fig:graphs}. 

\begin{figure}[t]
\centerline{\includegraphics[width=0.9\textwidth]{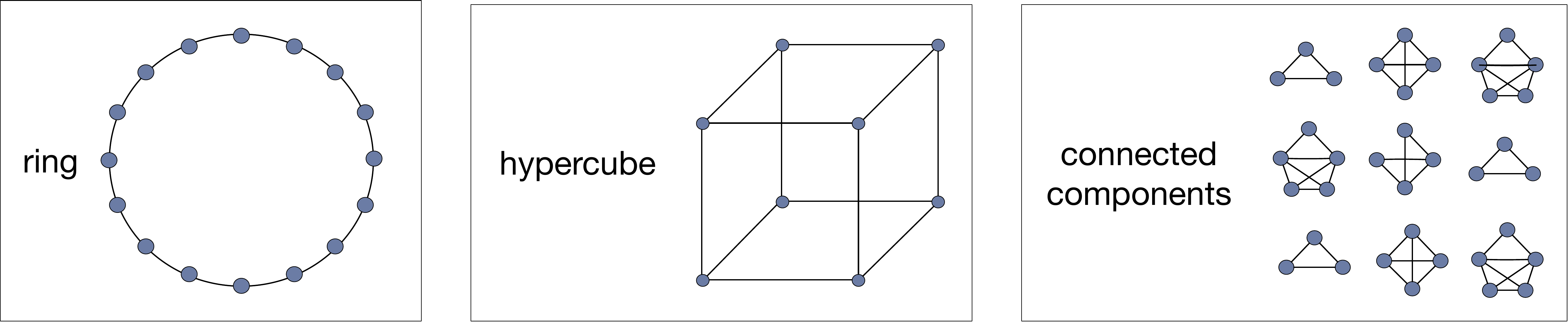}}
\caption{Graphs whose signless Laplacians have subzero matrix completions of much lower rank than the number of nodes (\textit{left}, \textit{middle}) or the number of connected components (\textit{right}).}
\label{fig:graphs}
\end{figure}

\subsection{Example: ring graph}
\label{sec:ring}

\vspace{-1ex}
The geometric picture in the previous theorems is nicely illustrated by the following example. Let $\delta\!\in\big[0,\frac{1}{2}\big]$, and consider the $n\!\times\!n$ matrix given by
\begin{equation}
\label{eq:ring}
S_{ij} = \left\{
 \begin{array}{ll}
  1 & \mbox{if $i\!=\!j$}, \\
  \delta & \mbox{if}\ |i\!-\!j| \in\{1,n\!-\!1\}, \\
  0 & \text{otherwise.}
 \end{array}
 \right.
\end{equation}
Note that $\mat{S}$ reduces to the identity  matrix for $\delta\!=\!0$, while for $\delta\!=\!\frac{1}{2}$, it yields the signless Laplacian for the ring graph in the left panel of \cref{fig:graphs}. For this graph, we shall see that its low-dimensional structure is immediately revealed by a subzero matrix completion of its signless Laplacian. We note furthermore that the matrix $\mat{S}$ in \cref{eq:ring} has no positive submatrix with more than two rows and columns, and hence \cref{thm:rankL_bound} provides a lower bound of 2 on the rank of any subzero matrix completion. The next result shows that this lower bound is very nearly realized; this construction is previously known~\cite{lee1999-learning,saul2022-nmd}, but we include the proof as a stepping stone for later results.

\vspace{1ex}
\noindent\fcolorbox{black}{shade}{\parbox{0.985\textwidth}{\vspace{-2ex}
\begin{proposition}
\label{thm:ring}
Let $\mat{S}$ be the sparse $n\!\times\! n$ matrix in \cref{eq:ring}, and suppose $\delta\!\in\!\left[0,\frac{1}{2}\right]$. Then there exists a matrix $\mat{L}$, satisfying $\mat{S}=\max(0,\mat{L})$, such that $\textup{rank}(\mat{L})\leq 3$.
\end{proposition}
\vspace{-2ex}}}

\begin{proof}
The high rank of $\mat{S}$ in \cref{eq:ring} belies a simple geometric picture: the ring graph can be easily visualized in the plane by placing its $n$ nodes at equispaced points on the circle, as in the left panel of \cref{fig:graphs}. In particular, suppose we define the two-dimensional vectors
\begin{equation}
\label{eq:ringvecs}
 \vec{\alpha}_k = r\left(\cos\tfrac{2\pi k}{n},\,\sin\tfrac{2\pi k}{n}\right)
\end{equation} 
to lie at equispaced points on a circle of radius $r$. Then \cref{thm:tsm1} suggests that we should seek a subzero matrix completion, of rank 3 or less, that interprets $\mat{S}$ as the $n\!\times\! n$ thresholded similarity matrix with elements
\begin{equation}
\label{eq:tsm-sym}
  S_{ij} = \max\big(0,\vec{\alpha}_i\!\cdot\!\vec{\alpha}_j\!-\! \tau\|\vec{\alpha}_i\|\|\vec{\alpha_j}\|\big).
\end{equation}
To match the elements in \cref{eq:ring}, we use the freedom to choose the radius~$r$ in \cref{eq:ringvecs} as well as the threshold~$\tau$ in \cref{eq:tsm-sym}. Specifically, to match the diagonal elements of $\mat{S}$ in \cref{eq:ring}, we require $1=r^2(1\!-\!\tau)$, and to match the positive off-diagonal elements, we require $\delta=r^2(\cos\frac{2\pi}{n}\!-\!\tau)$. Solving these equations, we find
\begin{align}
\label{eq:tau}
\tau &= (1\!-\!\delta)^{-1}(\cos\tfrac{2\pi}{n}-\delta), \\
\label{eq:r2}
r^2 &= (1\!-\!\delta)(1-\cos\tfrac{2\pi}{n})^{-1}.
\end{align}
This completes the proof for $n\!\leq\!3$, but for $n\!\geq\!4$, we must account for the remaining zero elements in \cref{eq:ring} that are further off the diagonal. To proceed, we make two observations: first, that non-adjacent nodes in the ring graph are separated by an angle of at least $\tfrac{4\pi}{n}$, and second, that the threshold $\tau$ in \cref{eq:tau} is a monotonically decreasing function of $\delta$, with a minimum value of  $\tau_{\text{min}}=2\cos(\frac{2\pi}{n})\!-\!1$
at $\delta\!=\!\frac{1}{2}$. From these observations, it follows that for non-adjacent nodes
\begin{displaymath}
  \vec{\alpha}_i\!\cdot\!\vec{\alpha}_j\!-\! \tau\|\vec{\alpha}_i\|\|\vec{\alpha_j}\|\
  \leq\ r^2\left[\cos\tfrac{4\pi}{n}\!-\!\tau_{\text{min}}\right]\,
 =\, r^2 \left[\cos\tfrac{4\pi}{n}\!-\!2\cos\tfrac{2\pi}{n}\!+\!1\right]\,
 =\, 2r^2 \cos\tfrac{2\pi}{n}(\cos\tfrac{2\pi}{n}\!-\!1)\, \leq\, 0,
\end{displaymath}
where in the last inequality we have used that $n\!\geq\! 4$. From the above, we see that the \textit{thresholded} cosine similarity for these off-diagonal elements is equal to zero, completing the proof.
\end{proof}

The subzero matrix completion in \cref{thm:ring} can also be extended to permutation matrices, a result we will need in \cref{sec:block-diagonal}. Low-rank parameterizations of permutation matrices have been exploited in recent work on assignment problems and shape-matching~\cite{droge2023-kissing}. The following result is also known from previous work on the signed rank of binary matrices~\cite{alon2016-sign}. 

\noindent\fcolorbox{black}{shade}{\parbox{0.985\textwidth}{\vspace{-2ex}
\begin{corollary}
\label{thm:ident}
Any $n\times $n permutation matrix (and in particular, the $n\!\times\! n$ identity matrix) has a subzero matrix completion of rank 3 or less. 
\end{corollary}
\vspace{-2ex}}}
\begin{proof}
The matrix in \cref{eq:ring} reduces to the identity matrix upon setting $\delta\!=\!0$. Making this substitution in \cref{eq:ringvecs,eq:tsm-sym,eq:tau,eq:r2}, we can write $\mat{I}_n = \max(0,\mat{\Omega})$, where
\begin{equation}
    \Omega_{ij} = \frac{\cos\frac{2\pi(i-j)}{n}-\cos\frac{2\pi}{n}}{1-\cos\frac{2\pi}{n}},
\label{eq:Omega}
\end{equation}
and from the previous proposition, we know that $\text{rank}(\mat{\Omega})\leq 3$. 
Finally, if $\mat{P}$ is a $n\!\times\!n$ permutation matrix, then we have $\mat{P}\! =\! \max(0,\mat{P}\mat{\Omega})$, and the rank of a matrix is preserved under permutations of its rows and columns, so that \mbox{$\text{rank}(\mat{P\Omega})\!=\!\text{rank}(\mat{\Omega})\leq 3$}.
\end{proof}

The above result also highlights an important distinction of problems in subzero matrix completion. It is well-known that the nuclear norm provides the tightest convex surrogate to the rank of a matrix, and therefore, under certain conditions, the lowest-rank completion of a real-valued matrix can be obtained by choosing the missing entries to minimize its nuclear norm~\cite{candes2009-exact}. But \textit{these conditions are generally not met for subzero matrix completion}, where the resulting low-rank solutions often have a much higher nuclear norm than the sparse matrices from which they are derived. One such example is provided by subzero matrix completions of the identity matrix.

\noindent\fcolorbox{black}{shade}{\parbox{0.985\textwidth}{\vspace{-2ex}
\begin{corollary}
\label{thm:nuclear}
Let $\mat{\Omega}$ denote the subzero matrix completion of the $n \times n$ identity matrix $\mat{I}_n$ in \cref{eq:Omega}. For large $n$, it is the case that 
$\|\mat{\Omega}\|_* \gg \|\mat{I}_n\|_* = \textup{rank}(\mat{I_n}) \gg \textup{rank}(\mat{\Omega})$.
\end{corollary}
\vspace{-2ex}}}
\begin{proof}
The matrix $\mat{\Omega}$ is circulant with Fourier modes as eigenvectors. The lowest three modes account for the nonzero eigenvalues, which are given by $-n(1\!-\!\cos\phi)^{-1}$ and $\frac{n}{2}\cos\phi\,(1\!-\!\cos\phi)^{-1}$ 
(with multiplicity two) where $\phi=\frac{2\pi}{n}$. Since $\cos\phi \approx 1-\frac{2\pi^2}{n^2}$ for large $n$, it follows in this limit~that
\begin{equation}
\|\mat{\Omega}\|_* = n\left(\frac{1+|\!\cos\phi|}{1-\cos\phi}\right) \sim \frac{n^3}{\pi^2},
\end{equation}
and we see that the nuclear norm of the subzero matrix completion $\mat{\Omega}$ grows \textit{cubically} with $n$. This is true even though the matrix $\mat{\Omega}$ itself is always of rank 3 or less.
\end{proof}


\subsection{Example: hypercube graph}

Consider a sparse nonnegative matrix $\mat{S}$ of size $2^d\!\times\!2^d$ whose row and column indices enumerate the vertices of a hypercube in $d$ dimensions. Also, let $\vec{v}_i\!\in\!\{-1,+1\}^d$ denote the vertex indexed by the $i^\text{th}$ row or column, and suppose that the elements of $\mat{S}$ are given by
\begin{equation}
\label{eq:hypercube}
S_{ij} = \left\{
 \begin{array}{ll}
  1 & \mbox{if $i\!=\!j$}, \\
  \delta & \mbox{if}\ \|\vec{v}_i\!-\!\vec{v}_j\| = 2, \\
  0 & \text{otherwise.}
 \end{array}
 \right.  
\end{equation}
where $\delta\!>\!0$. Note that for $\delta\!=\!\frac{1}{d}$, the matrix in \cref{eq:hypercube} is equal to the signless Laplacian of the hypercube graph in the middle panel of \cref{fig:graphs}, while for $\delta\!=\!0$, it reduces to the identity matrix. In general the rank of~$\mat{S}$ is exponential in~$d$, suffering from the curse of dimensionality. But we can lift this curse and obtain a much lower rank by replacing the zeros of~$\mat{S}$ in \cref{eq:hypercube} with negative values.

\vspace{1ex}
\noindent\fcolorbox{black}{shade}{\parbox{0.985\textwidth}{\vspace{-2ex}
\begin{proposition}
\label{thm:hypercurbe}
Let $\mat{S}$ be the sparse $2^d\!\times\! 2^d$ matrix in \cref{eq:hypercube}, and suppose $\delta\!\in\!\left(0,\frac{1}{d}\right)$. Then there exists a matrix $\mat{L}$, satisfying $\mat{S}=\max(0,\mat{L})$, with $\textup{rank}(\mat{L})\leq d\!+\!1$.
\end{proposition}
\vspace{-2ex}}}

\begin{proof}
Again we appeal to \cref{thm:tsm1} and identify vectors $\vec{\alpha}_i\!\in\!\mathbb{R}^d$ and a threshold $\tau\!\in\!(-1,1)$ such that the elements of $\mat{S}$ are given by the thresholded similarities in \cref{eq:tsm-sym}. In this case, it is natural to point these vectors at the vertices of the hypercube, with $\vec{\alpha}_i\! =\! r\vec{v}_i$ for some $r\!>\!0$. Then a similar exercise shows that \cref{eq:tsm-sym,eq:hypercube} define the same matrix elements upon setting $r^2 = \frac{1}{2}(1\!-\!\delta)$ and $\tau = 1-\frac{2}{d}(1\!-\!\delta)^{-1}$.
\end{proof}


\subsection{Example: block diagonal matrices}
\label{sec:block-diagonal}

Our final example builds in a different way on the ring graph and illustrates another way that subzero matrix completions can significantly lower the rank of a sparse matrix. Suppose in particular that $\mat{S}$ is a \textit{block diagonal} matrix of the form
\begin{equation}
\label{eq:blocks}
\mat{S} = \left[
\begin{array}{ccccc}
\mat{S}_1 & \mat{0} & \mat{0} & \cdots & \mat{0} \\
\mat{0} & \mat{S}_2 & \mat{0} & \cdots & \mat{0} \\
\mat{0} & \mat{0} & \mat{S}_3 & \cdots & \mat{0} \\
\vdots & \vdots & \vdots & \ddots & \vdots  \\
\mat{0} & \mat{0} & \mat{0} & \cdots & \mat{S}_k
\end{array}
\right]
\end{equation}
where each block $\mat{S}_b$ is a (possibly dense) $m_b\!\times\! n_b$ nonnegative matrix. There are many reasons to study this type of idealized structure: for example, if the adjacency matrix of a symmetric graph has this form (or if the rows and columns can be permuted to yield this form), then its nonzero blocks would correspond to the graph's \textit{connected components}, as shown in the right panel of \cref{fig:graphs}.

The matrix in \cref{eq:blocks}
has a rank equal to the sum of the ranks of its non-zero blocks:
namely, \mbox{$\text{rank}(\mat{S}) = \sum_{b=1}^k \text{rank}(\mat{S}_b)$.}
Not surprisingly, for such matrices one can also construct subzero matrix completions in a blockwise fashion. In particular, suppose that there exist matrices $\mat{L}_b$ satisfying \mbox{$\mat{S}_b = \max(0,\mat{L}_b)$} for \mbox{$1\!\leq\!b\leq\!k$}, and let
\begin{equation}
\label{eq:Ladd}
\mat{L} = \left[
\begin{array}{ccccc}
\mat{L}_1 & \mat{0} & \mat{0} & \cdots & \mat{0} \\
\mat{0} & \mat{L}_2 & \mat{0} & \cdots & \mat{0} \\
\mat{0} & \mat{0} & \mat{L}_3 & \cdots & \mat{0} \\
\vdots & \vdots & \vdots & \ddots & \vdots  \\
\mat{0} & \mat{0} & \mat{0} & \cdots & \mat{L}_k
\end{array}
\right].
\end{equation}
The above provides one possible subzero matrix completion of the block-diagonal matrix in \cref{eq:blocks}, with $\text{rank}(\mat{L}) = \sum_{b=1}^k \text{rank}(\mat{L}_b)$, but in general \textit{this is not the one of lowest rank}. To show this, we start by considering the special case where all of the blocks are of the same size.

\vspace{1ex}
\noindent\fcolorbox{black}{shade}{\parbox{0.985\textwidth}{\vspace{-2ex}
\begin{lemma}
\label{lemma:block}
Suppose that each block $\mat{S}_b$ in \cref{eq:blocks} is of equal size with $m$ rows and $n$ columns. Then there exists a matrix $\mat{L}$, satisfying $\mat{S} = \max(0,\mat{L})$, with 
\mbox{$\textup{rank}(\mat{L}) \leq 3\min(m,n)$}.
\end{lemma}
\vspace{-2ex}}}
\vspace{0.5ex}

Note that the lemma asserts the existence of a subzero matrix completion whose rank is \textit{independent of the number of blocks}, $k$. Furthermore, if each block $S_b$ has strictly positive elements and is of full rank, then this latent rank is within a constant factor of the lower bound in \cref{thm:rankL_bound}.

\begin{proof}
Without lack of generality, we suppose that $m\!\leq\!n$. (If instead $n\!\leq\!m$, we can apply the same reasoning to the transpose of $\mat{S}$.) We start with the simplest cases to build intuition. If $k\!=\!1$, then we can trivially take $\mat{L}\!=\!\mat{S}_1$ so that $\text{rank}(\mat{L})\!\leq\! m$. Likewise, if $k\!=\!2$, we consider the subzero matrix completion given by
\begin{equation}
\label{eq:row_flip}
\mat{L} = \left[
\begin{array}{rr}
\mat{S}_1 & -\mat{S}_2 \\
-\mat{S_1} & \mat{S}_2
\end{array}
\right].\,
\end{equation}
where the last $m$ rows of $\mat{L}$ are related to its first $m$ rows by a flipped sign. So in this case again, we see that $\text{rank}(\mat{L})\!\leq\! m$, satisfying the claim of the theorem.

For $k\!\geq\!3$, the proof is based on a similar construction as in the example of the ring graph. At a high level, the proof has two steps. First, we exhibit a subzero matrix completion $\mat{L}$, satisfying $\mat{S}=\max(0,\mat{L})$, from a combination of Kronecker and Hadamard products. Then we use the properties of these products to bound the rank of $\mat{L}$.

The form of \cref{eq:row_flip} suggests how to proceed for the more general case. Let $\mathbb{1}_m$ denote the $m$-dimensional column vector of all ones, and consider the matrix
\begin{equation}
\mat{L}\ =\ \left(\mathbb{1}_m\otimes
\left[\begin{array}{llcr}
  \!\!\mat{S}_1\! & \mat{S}_2\! & \cdots\! & \mat{S}_k\!\!
  \end{array}\right]\right)\
  \odot\ \left(\mat{\Omega}\, \otimes\, \mathbb{1}_m{\mathbb{1}_n}^{\!\!\!\!\top}\right)
\label{eq:hadamard}
\end{equation}
where $\otimes$ and $\odot$ denote, respectively, the Kronecker and Hadamard products, and $\mat{\Omega}$ is the subzero matrix completion of the $k\!\times\!k$ identity matrix in \cref{eq:Omega}. The factors of the Hadamard product in this construction are each of size $km\times kn$, consisting of $k^2$ blocks of size $m\!\times\! n$. More explicitly, we can write this construction~as
\begin{equation}
\mat{L}\ =\ 
  \left[
\begin{array}{rrrrr}
\blue{\mat{S}_1} & \blue{\mat{S}_2} & \blue{\mat{S}_3} & \cdots & \blue{\mat{S}_k} \\
\blue{\mat{S}_1} & \blue{\mat{S}_2} & \blue{\mat{S}_3} & \cdots & \blue{\mat{S}_k} \\
\vdots\hspace{1ex} & \vdots\hspace{1ex} & \vdots\hspace{1ex} & \ddots & \vdots\hspace{1ex} \\
\blue{\mat{S}_1} & \blue{\mat{S}_2} & \blue{\mat{S}_3} & \cdots & \blue{\mat{S}_k}
\end{array}
\right]\
\odot\
 \left[
\begin{array}{rrrrr}
\blue{\Omega_{11}\mathbb{1}} & \red{\Omega_{12}\,\mathbb{1}} & \red{\Omega_{13}\mathbb{1}} & \cdots & \red{\Omega_{1k}\mathbb{1}} \\
\red{\Omega_{21}\mathbb{1}} & \blue{\Omega_{22}\mathbb{1}} & \red{\Omega_{23}\mathbb{1}} & \cdots & \red{\Omega_{2k}\mathbb{1}} \\
\vdots\hspace{2ex} & \vdots\hspace{2ex} & \vdots\hspace{2ex} & \ddots & \vdots\hspace{2ex} \\
\red{\Omega_{k1}\mathbb{1}} & \red{\Omega_{k2}\mathbb{1}} & \red{\Omega_{k3}\mathbb{1}} & \cdots & \blue{\Omega_{kk}\mathbb{1}} \end{array}
\right],
\label{eq:hadamard2}
\end{equation}
where we have colored the nonnegative blocks of these factors in blue and the nonpositive blocks in red; here we have also used $\mathbb{1}=\mathbb{1}_m{\mathbb{1}_n}^{\!\!\top}$ as shorthand for the $m\!\times\!n$ matrix of all ones. Noting that $\Omega_{ii}\!=\!1$ along the diagonal of the right factor, we see that the Hadamard (elementwise) product for $\mat{L}$ in \cref{eq:hadamard2} yields a subzero matrix completion of the block-diagonal matrix $\mat{S}$ in \cref{eq:blocks}.

To complete the proof, we must bound the rank of $\mat{L}$. Here we exploit that (i) the rank of a Kronecker product is equal to the product of the ranks of its factors, and (ii) the rank of a Hadamard product is less than or equal to the product of the ranks of its factors. Thus from \cref{eq:hadamard} we have
\begin{align}
\text{rank}(\mat{L})\ 
&\leq\ \text{rank}\left(\mathbb{1}_m\otimes
\left[\begin{array}{llcr}
  \!\!\mat{S}_1\!\! & \mat{S}_2\!\! & \cdots\!\! & \mat{S}_k\!\!
  \end{array}\right]\right)\,
  \times\,
  \text{rank}\big(\mat{\Omega}\, \otimes\, \mathbb{1}_m{\mathbb{1}_n}^{\!\!\!\!\top}\big), \\
 &=\ \text{rank}(\mathbb{1}_m)\, \times\,
 \text{rank}\big(\left[\begin{array}{llcr}
  \!\!\mat{S}_1\!\! & \mat{S}_2\!\! & \cdots\!\! & \mat{S}_k\!\!
  \end{array}\right]\big)\,
  \times\,
  \text{rank}(\mat{\Omega})\,\times\, \text{rank}\big(\mathbb{1}_m{\mathbb{1}_n}^{\!\!\!\!\top}\big), \\
  &=\ 
 \text{rank}\left(\left[\begin{array}{llcr}
  \!\!\mat{S}_1\!\! & \mat{S}_2\!\! & \cdots\!\! & \mat{S}_k\!\!
  \end{array}\right]\right)\,
  \times\,
  \text{rank}(\mat{\Omega}), \\
  &\leq\ m\, \times\, \text{rank}(\mat{\Omega}), \\
 &\leq\ 3m,
\end{align}
yielding an upper bound on the rank of $\mat{L}$ that is independent of the number of blocks, $k$. This completes the proof.
\end{proof}

Armed with this lemma, we now consider the more general case for block diagonal matrices whose $k$ blocks are of different sizes. The following result provides one possible generalization.

\vspace{1ex}
\noindent\fcolorbox{black}{shade}{\parbox{0.985\textwidth}{\vspace{-2ex}
\begin{proposition}
\label{theorem:block}
Let $\mat{S}$ denote the sparse matrix in \cref{eq:blocks} with blocks of size $m_b\!\times\!n_b$ for \mbox{$1\!\leq\!b\leq\!k$}, and let $m_*$ and $n_*$ denote, respectively, the maximum number of rows and columns in any block: \mbox{$m_*\!=\!\max_b m_b$} and \mbox{$n_*\!=\!\max_b n_b$}. Then there exists a matrix $\mat{L}$, satisfying \mbox{$\mat{S} = \max(0,\mat{L})$}, with
\mbox{$\textup{rank}(\mat{L}) \leq 3\min(m_*,n_*)$}.
\end{proposition}
\vspace{-2ex}}}
\vspace{0.5ex}

\begin{proof}
For each block $\mat{S}_b$ in \cref{eq:blocks}, let $\mat{Z}_b$ define the \textit{zero-padded} matrix with $m_*$ rows and~$n_*$ columns that has $\mat{S}_b$ as its upper left $m_b\!\times\! n_b$ block. Let $\mat{Z}$ denote the $km_*\!\times\!kn_*$ block-diagonal matrix that has the blocks $\mat{Z}_b$ along its diagonal. By the previous lemma, there exists a matrix $\mat{K}$, satisfying $\mat{Z}=\max(0,\mat{K})$, with $\text{rank}(\mat{K})\leq 3\min(m_*,n_*)$. Finally, we obtain a matrix $\mat{L}$ satisfying $\mat{S}=\max(0,\mat{L})$ by deleting rows and columns from $\mat{K}$, and by constructing $\mat{L}$ in this way it will be the case that $\text{rank}(\mat{L})\leq\text{rank}(\mat{K})$.
\end{proof}


\subsection{Counterexample: relative primes}
\label{sec:gcd}

Not all sparse nonnegative matrices are conducive to subzero matrix completions of low rank. The following counterexample is instructive. Consider the $n\!\times\! n$ matrix with elements
\begin{equation}
\label{eq:logGCD}
S_{ij} = \log \mbox{\textsc{gcd}}(i,j),
\end{equation}
where $\mbox{\textsc{gcd}}(i,j)$ indicates the greatest common divisor of the integers $i$ and $j$; thus the $ij^{\text{th}}$ element is zero whenever $i$ and $j$ are relatively prime. It is a well-known result from number theory~\cite[Theorem 332]{hardy1975-intro}
that in the limit $n\!\rightarrow\!\infty$ two randomly chosen integers between $1$ and $n$ (inclusive) are relatively prime with probability $\frac{6}{\pi^2} \approx 0.60793$. Thus for large $n$ this matrix is over 60\% sparse. 

For this example, it is possible not only to determine the rank of $\mat{S}$, but also to provide a lower bound on the rank of any subzero matrix completion. It would be surprising---given the seemingly random distribution of the primes~\cite{cramer1936-order}, the difficulty of factoring large integers~\cite{crandall2005-prime}, and the computational complexity of long-established algorithms for calculating \textsc{gcd}s~\cite{moenck1973-fast}---to discover an asymptotically low-dimensional embedding of the integers that reproduces the elements of $\mat{S}$ in \cref{eq:logGCD}. The next results show that this skepticism is justified: not only does the rank of $\mat{S}$ grow in lockstep with $n$, but so does the rank of any subzero matrix completion of $\mat{S}$.

\vspace{1ex}
\noindent\fcolorbox{black}{shade}{\parbox{0.985\textwidth}{\vspace{-2ex}
\begin{proposition}
\label{thm:rankS_logGCD}
Let $\mat{S}$ denote the $n\!\times\! n$ matrix with elements $S_{ij} = \log\mbox{\textup{\textsc{gcd}}}(i,j)$, and let $\Pi(n)$ denote the number of prime powers less than or equal to~$n$. Then $\textup{rank}(\mat{S}) = \Pi(n)$.
\end{proposition}
\vspace{-2ex}}}
\vspace{-2ex}

\begin{proof}
We prove the result in two steps---first by showing that $\text{rank}(\mat{S}) \leq \Pi(n)$, and then by showing that $\text{rank}(\mat{S}) \geq \Pi(n)$. To prove the upper bound, we examine the number of linearly independent rows in $\mat{S}$. Note that the first row of~$\mat{S}$ consists of all zeros, so consider the $m^{\text th}$ row where $m\!>\!1$. Suppose that $m$ is \textit{not} a prime power. Then we can write $m=k\ell$ where $k$ and $\ell$ are relatively prime. It follows that $\text{gcd}(m,j)=\text{gcd}(k,j)\,\text{gcd}\,(\ell,j)$, and taking logs, we see that every row in $\mat{S}$ not corresponding to a prime power is equal to the sum of two preceding rows. It follows that $\mat{S}$ has at most $\Pi(n)$ linearly independent rows, or equivalently, that $\text{rank}(\mat{S})\leq\Pi(n)$. 
 
To prove the lower bound, we consider the submatrix $\mat{S}_p$ that includes only the rows and columns of $\mat{S}$ whose indices are powers of some prime~$p$. This submatrix has an especially simple structure; its $ij^\text{th}$ element is given by $(\log p) \min(i,j)$, and by row reduction, it can be shown that $\det\mat{S}_p\!>\! 0$, or equivalently, that $\mat{S}_p$ is of full rank. Now consider the larger submatrix $\mat{\Phi}$ derived from the rows and columns of $\mat{S}$ whose indices are powers of \textit{any} prime. By permuting the rows and columns of $\mat{\Phi}$, we can create a block diagonal submatrix whose individual blocks are the matrices~$\mat{S}_p$. It follows that $\text{rank}(\mat{S}) \geq \text{rank}(\mat{\Phi}) = \sum_p \text{rank}(\mat{S}_p) = \Pi(n)$.
\end{proof}

It can be deduced from the Prime Number Theorem~\cite{hardy1975-intro} that $\Pi(n)\!\sim\! \frac{n}{\log n}$. Hence from \cref{thm:rankS_logGCD} we conclude that $\text{rank}(\mat{S})$ grows at a nearly linear rate with its number of rows and columns. The next proposition considers \textit{subzero matrix completions} of $\mat{S}$ and shows that their rank, too, must grow at a similar rate.

\vspace{1ex}
\noindent\fcolorbox{black}{shade}{\parbox{0.985\textwidth}{\vspace{-2ex}
\begin{proposition}
\label{thm:rankL_logGCD}
Let $\mat{S}$ denote the $n\times n$ matrix with elements $S_{ij} = \log\mbox{\textup{\textsc{gcd}}}(i,j)$, and let $\pi(n)$ denote the number of primes less than or equal to~$n$. If $\mat{L}$ is any subzero matrix completion of $\mat{S}$, satisfying $\mat{S}=\max(0,\mat{L})$, then $\textup{rank}(\mat{L}) \geq \pi(\frac{n}{2})$.
\end{proposition}
\vspace{-2ex}}}
\vspace{-2ex}

\begin{proof}
We prove the result by showing that $\mat{S}$ has a positive submatrix of rank $\pi(\frac{n}{2})$. Let $p_k$ denote the $k^\text{th}$ prime, and consider the submatrix $\mat{E}$, with $\pi(\frac{n}{2})$ rows and columns, that includes only the \textit{even} rows and columns of $\mat{S}$ whose indices are equal to $2p_k$ for some prime $p_k$. This submatrix takes an especially simple form: its diagonal elements are given by $E_{kk}\! =\! \log 2 p_k$, while all of its off-diagonal elements are equal to~$\log 2$. Let $\vec{v}$ be any non-zero vector of length $\pi(\frac{n}{2})$. Then
\begin{equation}
\label{eq:vEv}
\vec{v}^{\!\top}\mat{E}\,\vec{v}\, =\, \sum_{k} v_k^2\, \log p_k +  \Big({\sum}_k v_k\Big)^2(\log 2)\, >\, 0,
\end{equation}
showing that $\mat{E}$ is positive-definite and of full rank. Finally, we note that all of the elements of $\mat{E}$ are positive, so it follows from \cref{thm:rankL_bound} that $\text{rank}(\mat{L}) \geq \text{rank}(\mat{E}) = \pi(\frac{n}{2})$.
\end{proof}

Again, from the Prime Number Theorem~\cite{hardy1975-intro}, we know that $\pi(n)\!\sim\! \frac{n}{\log n}$. Hence the previous propositions show that subzero matrix completions cannot reduce the rank of the matrix~$\mat{S}$ in \cref{eq:logGCD} by more than a constant factor as $n\!\rightarrow\!\infty$. In fact, since $\Pi(n)\!\sim\! \pi(n)$, we see that asymptotically they can at most reduce the rank by a factor of $\frac{1}{2}$.


\subsection{Sparse datasets}

The idealized matrices in previous sections provide intuition for the type of latent low-rank structure we might discover in real-world datasets. In this paper we look for this structure in three matrices of empirical interest---all of them large, sparse, and nonnegative. The first is a sparse similarity matrix with 70K rows and columns derived from the \textsc{mnist} image dataset of handwritten digits~\cite{lecun98-gradient}. The second is the synaptic count matrix for the connectome of a female fruit fly~\cite{dorkenwald2024-flywire} with over 139K brain cells. The third is a stochastic matrix derived from bigram counts for the 250K most common tokens in Wikipedia text~\cite{TC-wikipedia}. In all of these matrices, fewer than $1\%$ of the elements are positive, and they can all be comfortably stored in memory as sparse arrays (though not as dense ones).
Further details on these matrices can be found in appendix \ref{app:data}.

\Cref{tab:data} summarizes several basic properties of these matrices; these include the numbers of rows and columns $(m,n)$ with at least one nonzero element and the fraction of nonzero elements $\rho_+$ in these rows and columns. It is possible, using sparse matrix-vector multiplies, to compute a large number of leading singular values $(\sigma_1,\sigma_2,\ldots)$ for these matrices. The table also shows the relative errors $\varepsilon_k$ for $k\!\in\!\{10^1,10^2,10^3\}$, computed as
\begin{equation}
\varepsilon_k = \sqrt{1-\frac{\sum_{j=1}^{k} \sigma_j^2}{\|\mat{S}\|^2_F}},
\label{eq:svd_err}
\end{equation}
when each matrix is approximated via a truncated singular value decomposition (SVD). Note that the errors in \cref{eq:svd_err} lie between zero and one since $\|\mat{S}\|^2_F = \sum_j \sigma_j^2$, where~$\|\mat{S}\|_F$ denotes the Frobenius norm of $\mat{S}$. From the values of $\varepsilon_k$ in the table, we see that each matrix has a very slowly decaying spectrum of singular values. 

\begin{table}[b]
\begin{center}
\begin{tabular}{cccclccclccc}
\toprule
$\mat{S}$ &  
$m$ & $n$ & 
$\rho_+$ &  &
\hspace{1ex}$\varepsilon_{10}$\hspace{1ex} &
\hspace{1ex}$\varepsilon_{100}$\hspace{1ex} &
$\varepsilon_{1000}$ & &
\hspace{1ex}$m_+$\hspace{1ex} & $n_+$ & $\!$rank($\mat{S}_+$)$\!$ \\ \midrule
digits & \hspace{1ex}70,000 & \hspace{1ex}70,000 & 0.00023 & & 
  0.998 & 0.987 & 0.923 & & 11 & 9 & 9 \\ 
connectome & 139,003 & 138,955 & 0.00102 & &
   0.964 & 0.866 & 0.662 & & 61 & 59 & 59 \\ 
bigrams & 249,992 & 249,987 & 0.00088 & &
  0.479 & 0.379 & 0.307 & & 218 & 214 & 214 \\
\bottomrule
\end{tabular}
\end{center}
\caption{Sparse nonegative matrices analyzed numerically in this paper. For each matrix $\mat{S}$, the table shows the number of rows and columns $(m,n)$ with at least one nonzero element, the fraction of nonzero elements $\rho_+$, and the relative errors $\varepsilon_k$ in \cref{eq:svd_err} from truncated singular value decompositions of rank $k\!\in\!\{10,100,1000\}$. The table also shows the number of rows and columns $(m_+,n_+)$ and rank of the largest positive submatrix~$\mat{S}_+$ found by a randomized greedy procedure.}
\label{tab:data}
\end{table}

We observed in \cref{thm:rankL_bound} that the rank of any positive submatrix of $\mat{S}$ provides a lower bound on the rank of any subzero matrix completion. Thus it is also informative, for each matrix in \cref{tab:data}, to identify its positive submatrix of largest rank, or even more simply, its largest positive submatrix. But this latter problem, though simpler, is still not simple; it is equivalent to finding the maximal edge biclique in a bipartite graph, which is known to be NP-complete~\cite{peeters2003-maximum}. In lieu of an exhaustive approach, we used a simple greedy procedure to search for large positive submatrices. The procedure starts from a randomly sampled nonzero element of~$\mat{S}$ and iteratively adds rows and columns until no larger positive submatrix can be found; we also used heuristics to accelerate these searches by capitalizing on the results of previous ones. With this procedure, we were able to identify large positive submatrices, with $m_+$ rows and $n_+$ columns, of the matrices in \cref{tab:data}. For each matrix, the table also shows the dimensions and rank of the largest positive submatrix~$\mat{S}_+$ found in this~way. In practice these largest positive submatrices were also of full rank.

The relatively small values of these ranks suggest that subzero matrix completions may be useful to find latent low-rank structure in these matrices. Latent low-rank structure is also suggested, in the thresholded similarity matrix of \textsc{mnist} digits, by the approximately block diagonal structure that emerges from a simple permutation of its rows and columns; see \cref{fig:digit-blocks}. The ten blocks along the diagonal in this figure record the thresholded similarities of images from the same digit class, while the sparser off-diagonal blocks record those from easily distinguished classes (e.g., \textsc{zeros} versus \textsc{fours}). Unlike the examples in previous sections, however, the matrices in \cref{tab:data} do not have any perfect symmetries or telltale patterns that can be exploited to guess a subzero matrix completion of low rank. We therefore turn to a study of numerical methods.

\begin{figure}[t]
\centerline{
  \includegraphics[height=1.5in]{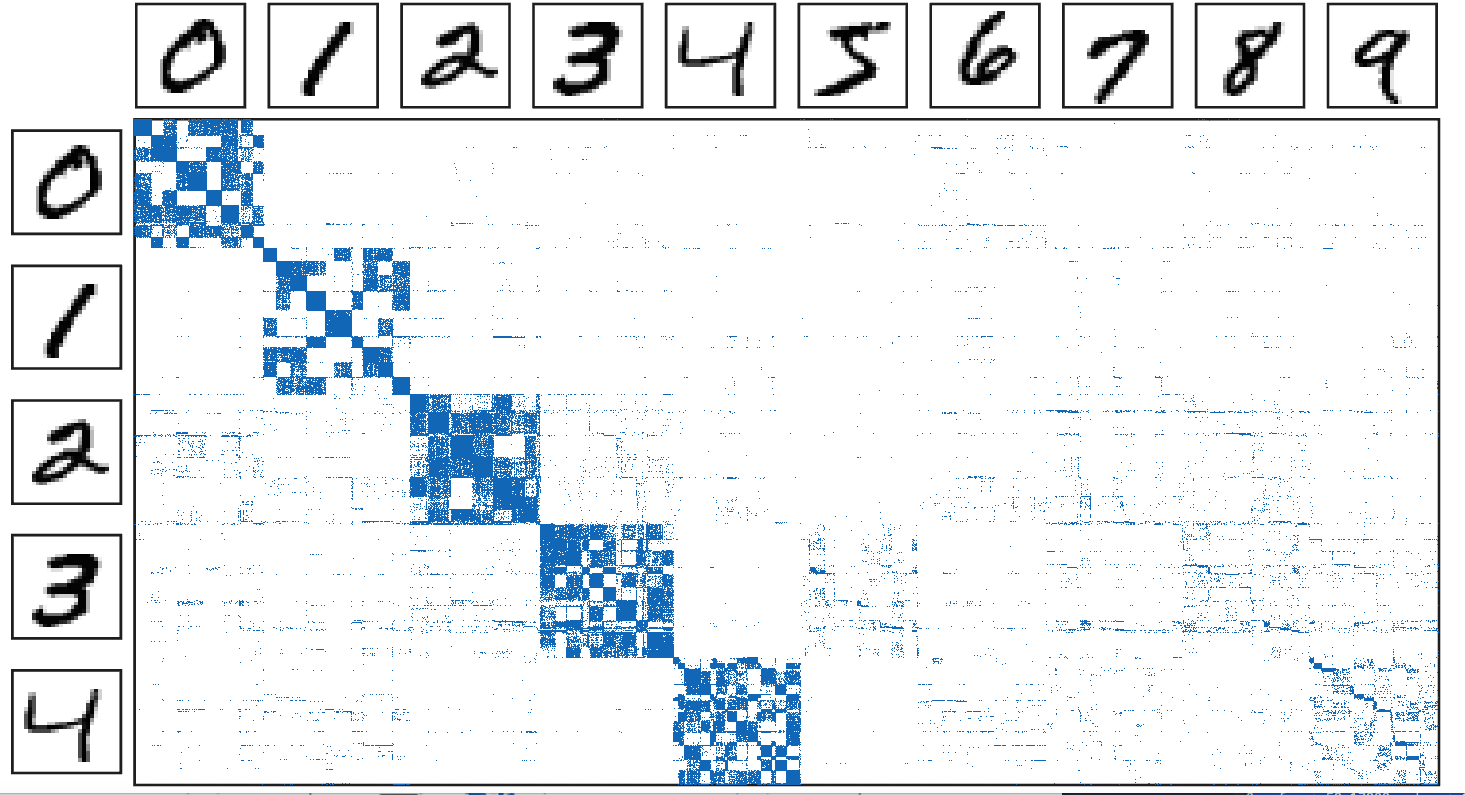}
  \hspace{5ex}
  \includegraphics[height=1.5in]{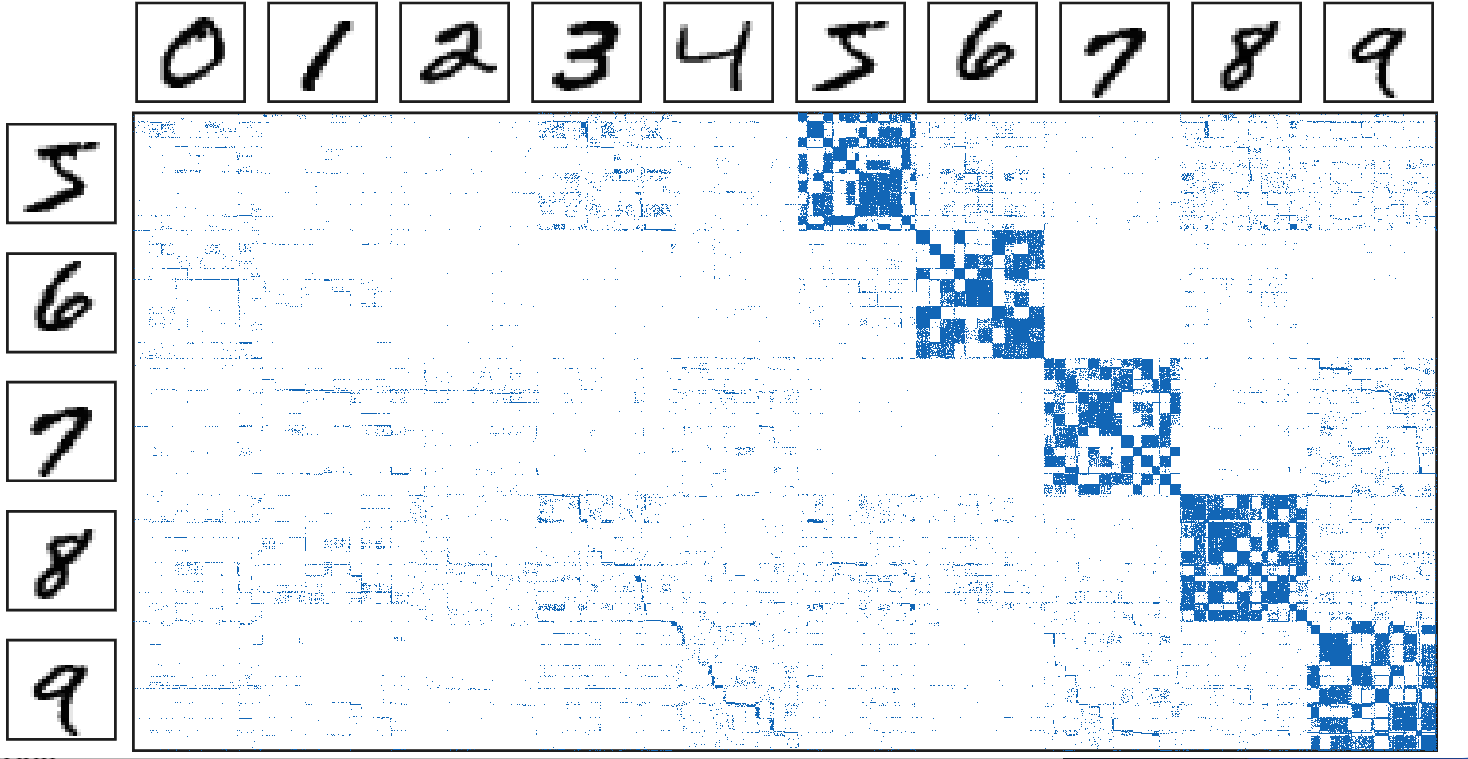}}
\caption{The sparse similarity matrix for \textsc{MNIST} digit images in \cref{tab:data} exhibits an approximately block diagonal structure if its rows and columns are appropriately permuted. Above, shown side-by-side, are the upper and lower halves of this reconstituted matrix.}
\label{fig:digit-blocks}
\end{figure}



\section{Algorithm}
\label{sec:alg}

Given an $m\!\times\!n$ sparse nonnegative matrix $\mat{S}$, we now turn to the problem of computing a low-rank matrix $\mat{L}$ such that $S_{ij}\approx\max(0,L_{ij})$. The next sections review the basic framework in which this problem has been most commonly tackled.

\subsection{Objective function}
\label{sec:loss}

Suppose that $\mat{L}$ is close but not exactly equal to a subzero matrix completion of~$\mat{S}$. There are many ways to measure how well $\mat{S}$ can be reconstructed from $\mat{L}$, or equivalently, how closely $\mat{S}$ matches the rectified matrix
\begin{equation}
    \hat{\mat{S}} = \max(0,\mat{L}).
\label{eq:rectify}
\end{equation}
We start by reviewing two widely adopted measures of differences between nonnegative matrices. The first is the \textit{normalized root-mean-squared error} \textup{(RMSE)}, which we denote by
\begin{equation}
\label{eq:RMSE}
\mathcal{E}(\mat{S},\hat{\mat{S}})\ =\ \frac{\|\mat{S}-\hat{\mat{S}}\|_F}{\|\mat{S}\|_F},
\end{equation}
and which is always greater than or equal to zero but not bounded above. Another is the \textit{weighted Jaccard distance} \textup{(WJD)}, which we denote by 
\begin{equation}
\label{eq:WJD}
\mathcal{J}_W(\mat{S},\hat{\mat{S}})\ =\ 1\, -\, \frac{\sum_{ij} \min\big(S_{ij},\hat{S}_{ij}\big)}{\sum_{ij}\max\big(S_{ij},\hat{S}_{ij}\big)},
\end{equation}
and which is bounded between zero and one. Both the RMSE and WJD are calibrated so that the all-zero approximation gives an error of one: i.e., $\mathcal{E}(\mat{S},\mat{0})\! =\! \mathcal{J}_W(\mat{S},\mat{0})\! =\! 1$. Likewise, both vanish if and only if $\mat{S} = \hat{\mat{S}}$.

Despite their familiarity, the losses in \cref{eq:RMSE,eq:WJD} are not the most commonly used for subzero matrix completion. One reason is that they are not especially well-suited to hill-climbing methods: their partial derivatives with respect to the elements of $\mat{L}$ vanish whenever these elements are negative. Another reason is that in general we seek to discover approximate completions that are interpretable in the following sense: \textit{the more negative an element~$L_{ij}\!<\!0$, the less likely it is matched to a \textit{positive} element~$S_{ij}\!>\!0$.} Note, however, that the RMSE \cref{eq:RMSE} and WJD in \cref{eq:WJD} only depend on $\mat{L}$ through the \textit{rectified} values of its elements. Hence these criteria are equally partial to errors in which a positive element ($S_{ij}\!>\!0$) is modeled by a zero element ($L_{ij}\!=\!0$) or a strongly negative element ($L_{ij}\!\ll\!-S_{ij})$.

Instead, following many previous works~\cite{saul2022-nmd,saul2022-geometrical,seraghiti2023-accelerated,awari2024-coordinate,wang2024-momentum,wang2025-accelerated}, we consider how to optimize an objective function that involves, in addition to the low-rank matrix $\mat{L}$, an auxiliary matrix (or latent variable)~$\mat{Z}$ of the same size. This optimization takes the form
\begin{equation}
\label{eq:auxCost}
\min_{\mat{L},\mat{Z}}\frac{\|\mat{Z}\!-\!\mat{L}\|_F}{\|\mat{S}\|_F} 
 \quad \mbox{such that}\quad
  \left\{ \begin{array}{l}
      \mat{S} = \max(0,\mat{Z}), \\
      \text{rank}(\mat{L}) \leq r,
      \end{array}\right.
\end{equation}
where the prescribed rank $r$ is typically far less than either the number of rows or columns of~$\mat{S}$. Note that the objective in \cref{eq:auxCost} is bounded below by zero, and it is only equal to zero when $\mat{Z}\!=\!\mat{L}$, which together with the first constraint, implies that $\mat{S}=\max(0,\mat{L})$. In addition, unlike the RMSE and WJD, this objective more severely penalizes errors in which a positive element in $\mat{S}$ is modeled by a strongly negative element of~$\mat{L}$.
The optimization in \cref{eq:auxCost} can be derived from the deterministic limit~\cite{saul2022-nmd} of an Expectation-Maximization algorithm for maximum likelihood estimation~\cite{dempster1977-maximum}, and for this reason (as we shall see in the next section), it lends itself very well to alternating minimization procedures. Finally, the following result relates the objective in \cref{eq:auxCost} to the RMSE in \cref{eq:RMSE}.

\vspace{1ex}
\noindent\fcolorbox{black}{shade}{\parbox{0.985\textwidth}{\vspace{-2ex}
\begin{proposition}
The objective in \cref{eq:auxCost} at any feasible point of the optimization provides an upper bound on the \text{RMSE} in \cref{eq:RMSE}; specifically, if $\mat{S}\!=\!\max(0,\mat{Z})$ and $\hat{\mat{S}}\!=\!\max(0,\mat{L})$, then \mbox{$\|\mat{Z}\!-\!\mat{L}\|_F \geq \big\|\mat{S}\!-\!\hat{\mat{S}}\big\|_F$}.
\end{proposition}
\vspace{-2ex}}}

\begin{proof}
Note that the function $f(\xi) = \max(0,\xi)$ is 1-Lipschitz because it is continuous and $|f'(\xi)|\leq 1$ for all nonzero $\xi$. From Lipschitz continuity, it follows that 
\begin{equation}
\big|S_{ij}\!-\!\hat{S}_{ij}\big| = \big|\!\max(0,Z_{ij})-\max(0,L_{ij})\big| \leq |Z_{ij}\!-\!L_{ij}|
\label{eq:lipschitz}
\end{equation}
for each element of these matrices, and summing over the squares of differences in \cref{eq:lipschitz}, we see that $\|\mat{S}-\hat{\mat{S}}\|_F^2 \leq \|\mat{Z}-\mat{L}\|_F^2$.
\end{proof}

This result improves by a constant factor on a previous upper bound~\cite{gillis2025-extrapolated} relating the objective in \cref{eq:auxCost} to the RMSE in \cref{eq:RMSE}. From this upper bound, it follows that we may perform the optimization in \cref{eq:auxCost} with a reasonable expectation of converging to solutions with low (though not necessarily minimal) values of the RMSE.


\subsection{Alternating minimization}
\label{sec:altmin}

The objective in \cref{eq:auxCost} can be minimized by an iterative procedure that alternates between updating the low-rank matrix $\mat{L}$ for fixed $\mat{Z}$ and updating the auxiliary matrix $\mat{Z}$ for fixed $\mat{L}$. Each update is derived by minimizing the objective with respect to one of these matrices, and thus the objective decreases monotonically from one iteration to the next. For fixed~$\mat{Z}$, the objective is optimized by choosing $\mat{L}$ to be its best approximation (in the least-squares sense) of rank $r$:
\begin{equation}
\mat{L}\ =\ \arg\!\!\!\!\min_{\!\!\!\!\!\!\!\!\text{rank}(\mat{X})\leq r}\! \big\|\mat{Z}\!-\!\mat{X}\big\|.
\label{eq:tsvd}
\end{equation}
This approximation can be computed from a \textit{truncated} singular value decomposition (SVD) of $\mat{Z}$, keeping only its leading $r$ components. 
Likewise, for fixed~$\mat{L}$, the objective is optimized by setting
\begin{equation}
Z_{ij} = \left\{\!\!
\begin{array}{ll}
  S_{ij} & \mbox{if $S_{ij}\!>\!0$}, \\
  L_{ij}\!-\!\hat{S}_{ij} & \mbox{if $S_{ij}\!=\!0$},
\end{array}
\right.
\label{eq:updateZ}
\end{equation}
where $\hat{S}_{ij}$ is the rectified value of $L_{ij}$ given by \cref{eq:rectify}. Alternating procedures of this form were among the earliest to be explored for this optimization~\cite{saul2022-nmd,saul2022-geometrical}.

More recent work~\cite{seraghiti2023-accelerated,wang2024-momentum,wang2025-accelerated,wang2025-efficient,gillis2025-extrapolated}, has focused on explicitly representing the low-rank matrix $\mat{L}\!\in\!\mathbb{R}^{m\times n}$ as the product of two smaller matrices; in many ways this approach permits a finer degree of control over the optimization. We do so by writing
\begin{align}
\label{eq:ABt}
\mat{L} &= \mat{A}\mat{B}^\top, \\
\label{eq:reluAB}
\hat{\mat{S}} &= \max\!\big(0,\mat{A}\mat{B}^{\!\top}\big),
\end{align}
where $\mat{A}\!\in\!\mathbb{R}^{m\times r}$ and $\mat{B}\!\in\!\mathbb{R}^{n\times r}$, and
as in \cref{eq:rectify}, we use $\mat{\hat{S}}$ to denotes the rectified low-rank product with which ones hope to approximate the sparse matrix~$\mat{S}$. Here we parameterize both factors $\mat{A}$ and $\mat{B}$ as tall and skinny matrices with~$r$ columns. Given this parameterization, the objective in \cref{eq:auxCost} can be optimized by a three-way alternating procedure~\cite{seraghiti2023-accelerated} that repeatedly updates one of the matrices $\mat{Z}$, $\mat{A}$, or $\mat{B}$ while holding the other two of these matrices fixed. In this case, the auxiliary matrix $\mat{Z}$ continues to be updated by \cref{eq:updateZ}, while the factors $\mat{A}$ and $\mat{B}$ are optimized by the \textit{least-squares} updates~\cite{hastie2015-matrix}:
\begin{align}
\label{eq:lsqA}
\mat{A} &\leftarrow \mat{Z}\ \mat{B}(\mat{B}^{\!\top}\mat{B})^{-1},\\
\label{eq:lsqB}
\mat{B} &\leftarrow \mat{Z}^{\!\top}\!\mat{A}(\mat{A}^{\!\top}\!\mat{A})^{-1}.
\end{align}
These least-squares updates can also be further tuned by introducing additional hyperparameters~\cite{seraghiti2023-accelerated,wang2024-momentum,gillis2025-extrapolated} into the optimization (e.g., for momentum, regularization, extrapolation).

It will be useful to make a change of variables and rewrite the above updates in a slightly different way. In particular, let
\begin{equation}
\mat{\Delta} = \mat{Z}-\mat{A}\mat{B}^{\!\top},
\label{eq:Delta}
\end{equation}
so that $\mat{\Delta}$ measures the \textit{difference} between the dense matrix~$\mat{Z}$ and the low-rank matrix~$\mat{A}\mat{B}^{\!\top}$. We proceed to rewrite the updates in \cref{eq:updateZ,eq:lsqA,eq:lsqB} in terms of $(\mat{\Delta},\mat{A},\mat{B})$ instead of $(\mat{Z},\mat{A},\mat{B})$. In particular, suppose that we have updated $\mat{Z}$ according to \cref{eq:updateZ}. Then
from the update in \cref{eq:updateZ}, and the factorization in \cref{eq:ABt}, it follows that
\begin{equation}
\mat{\Delta}\, =\, \mathbb{1}_{\mat{S}} \odot \big(\mat{S}\!  +\! \hat{\mat{S}} -\! \mat{A}\mat{B}^{\!\top}\!\big)\, -\, \hat{\mat{S}},
\label{eq:updateD}    
\end{equation}
where $\odot$ denotes the Hadamard product and $\mathbb{1}_\mat{S}$ denotes the sparse \textit{binary} matrix with the same zeros as $\mat{S}$, but whose nonzero values are all equal to one. Finally, substituting $\mat{\Delta}+\mat{A}\mat{B}^{\!\top}$ for~$\mat{Z}$ in \cref{eq:lsqA,eq:lsqB}, we can rewrite these least squares updates as
\begin{align}
\label{eq:spLsqA}
\mat{A} &\leftarrow \mat{A} + \mat{\Delta} \mat{B}(\mat{B}^{\!\top}\mat{B})^{-1},\\
\label{eq:spLsqB}
\mat{B} &\leftarrow \mat{B} + \mat{\Delta}^{\!\!\top}\!\mat{A}(\mat{A}^{\!\top}\!\mat{A})^{-1}.
\end{align}
The updates in this form make explicit that the factors $\mat{A}$ and $\mat{B}$ are shifted by an amount that is linear in the matrix~$\mat{\Delta}$. We will examine the properties of this matrix further in \cref{sec:sparse}.


\section{Scaling to large problems}
\label{sec:scaling}

In this section we show how to scale the alternating minimization over the matrices $\mat{\Delta}$, $\mat{A}$, and~$\mat{B}$ in \cref{eq:updateD,eq:spLsqA,eq:spLsqB} up to very large problems in subzero matrix completion. We have in mind applications where the dense matrix $\mat{Z}$ in \cref{eq:updateZ} is much too large to be stored all at once in memory, thereby exceeding the limits of current hardware. There are also computational bottlenecks that come to the fore in this regime. We address each of these challenges in turn.

\subsection{Memory management}

A memory-bottleneck arises whenever it is too expensive to store $\mat{Z}$ in \cref{eq:updateZ}, or alternatively~$\mat{\Delta}$ in \cref{eq:updateD}, as a dense matrix. We begin by describing two simple and complementary ways to conserve memory in these problems.

\vspace{-3ex}
\paragraph{Parallel least-squares.}
\label{sec:parallel}
The simplest memory-efficient implementation is obtained by exploiting the parallel structure of the updates. Such parallelism is often leveraged in optimizations based on alternating least-squares~\cite{zhou2008-large,yu2014-parallel}. Let~$I$ denote a subset of contiguous row indices between 1 and $m$ inclusive, and let $\mat{A}_{I}$
denote the submatrix of $\mat{A}$ consisting of \textit{rows} in $I$. Then analogous to the updates for $\hat{\mat{S}}$, $\mat{\Delta}$, and $\mat{A}$ in \cref{eq:reluAB,eq:Delta,eq:updateD,eq:spLsqA}, we can also derive the blockwise updates 
\begin{eqnarray}
  \label{eq:Shat_par}
    \hat{\mat{S}}_{I} & \!\!\!=\!\!\! &\max\big(0,\mat{A}_{I}\,\mat{B}^{\!\top}\big), \\[-0.25ex]
  \label{eq:Delta_par}
    \mat{\Delta}_{I} &\!\!\!=\!\!\! & 
      \big[\mathbb{1}_{\mat{S}} \odot \big(\mat{S}\!  -\! \hat{\mat{S}} -\! \mat{A}\mat{B}^{\!\top}\!\big)\big]_{I}\,
      -\, \mat{\hat{S}}_{I}\\[-0.25ex]
  \label{eq:lsqA_par}
    \mat{A}_{I} &\!\!\!\leftarrow\!\!\! &\mat{A}_{I}\, +\, \mat{\Delta}_{I}\, \mat{B}(\mat{B}^{\!\top}\mat{B})^{-1}.
\end{eqnarray}
Likewise, for any block row of $\mat{B}$, we can derive an analogous update that involves only the corresponding \textit{columns} of~$\mat{S}$, $\hat{\mat{S}}$, and $\mat{\Delta}$. 
These blockwise updates can be parallelized across multiple cores if they are available. On a single core, where the updates must be performed serially, it is generally most efficient to do so with the largest blocks of $\mat{\Delta}_{I}$ that fit into memory. 


\vspace{-3ex}
\paragraph{Stochastic least-squares.}
\label{sec:stochastic}

A memory-efficient implementation can also be obtained by introducing an element of randomness into the alternating minimization over $\mat{\Delta}$, $\mat{A}$, and $\mat{B}$. This idea builds on a familiar idea in statistical learning---namely, that stochastic optimizers often converge much faster on large problems by exploiting correlations across different cross-sections of the data~\cite{bottou2004-large}. For large problems in subzero matrix completion, such correlations arise whenever different submatrices of~$\mat{S}$ have similar statistics. In this case, we can use \textit{stochastic} least-squares updates (analogous to stochastic gradient-based methods) to learn the matrices $\mat{A}$ and $\mat{B}$ more quickly. Let~$J$ denote a random mini-batch of column indices between 1 and $n$ inclusive, and let $\mat{\Delta}_{:J}$ denote the submatrix of $\mat{\Delta}$ consisting of \textit{columns} in $J$. Then analogous to the updates for $\hat{\mat{S}}$, $\mat{\Delta}$, and $\mat{A}$ in \cref{eq:reluAB,eq:Delta,eq:updateD,eq:spLsqA}, we consider the stochastic updates 
\begin{eqnarray}
  \label{eq:Shat_sto}
    \hat{\mat{S}}_{:J} & \!\!\!=\!\!\! &\max\left(0,\mat{A}\mat{B}_{J}^{\!\top}\right), \\[-0.25ex]
  \label{eq:Delta_sto}
    \mat{\Delta}_{:J} &\!\!\!=\!\!\! & 
      \big[\mathbb{1}_{\mat{S}} \odot \big(\mat{S}\!  -\! \hat{\mat{S}} -\! \mat{A}\mat{B}^{\!\top}\!\big)\big]^J - \mat{\hat{S}}_{:J}, \\[-0.25ex]
  \label{eq:lsqA_sto}
    \mat{A} &\!\!\!\leftarrow\!\!\! &\mat{A}\, +\, \eta\mat{\Delta}_{:J}\, \mat{B}_{J}\!\left[\mat{B}_{J}^{\!\top}\mat{B}_{J}\right]^{-1}.
\end{eqnarray}
where $\eta\!>\!0$ is a step size. Once again, the mini-batch size is gated by the available memory to store the submatrix $\mat{\Delta}_{:J}$. Though not shown above, these stochastic updates can also be accelerated by incorporating a \textit{momentum} term~\cite{polyak1964-momentum}---namely, adding a multiple of the previously executed update to the right side of \cref{eq:lsqA_sto} and using an additional hyperparameter to control the amount~of~momentum.

\textbf{Complementarity.} It is worth contrasting the different ways that memory is conserved by the least-squares updates in \cref{eq:lsqA_par,eq:lsqA_sto}. In both cases, it is the elements of~$\mat{A}$ that are being updated, but in the former, it is the rows of $\mat{A}$ and $\mat{\Delta}$ that are being subsampled---to be updated, if possible, in parallel---while in the latter, it is the rows of $\mat{B}$ and $\mat{\Delta}^{\!\top}$ that are being subsampled to yield a stochastic update for all the rows of~$\mat{A}$. 
While both updates conserve memory by working with smaller submatrices of~$\mat{\Delta}$, they otherwise involve different trade-offs. The parallel updates have the advantage that they monotonically decrease the objective function and do not require additional hyperparameters, such as step sizes; the stochastic updates have the advantage that they converge much faster when different submatrices of~$\mat{S}$ have similar statistics. For very large problems, both strategies can also be combined into a single update that subsamples rows and columns of $\mat{\Delta}$.

\subsection{Computational throughput}
\label{sec:GPU}

Though memory can be managed by working with blocks of the matrices $\mat{\hat{S}}$ and $\mat{\Delta}$, there remains the computational bottleneck of evaluating and multiplying these blocks for the parallel or stochastic updates in the previous section. The main bottleneck is to calculate large blocks of the rectified matrix \mbox{$\hat{\mat{S}} = \max(0,\mat{A}\mat{B}^{\!\top})$} in eq.~(\ref{eq:reluAB}). When $\mat{S}$ has up to hundreds of thousands of rows and columns, this cost can be managed reasonably well by multiplying the factors $\mat{A}$ and $\mat{B}^{\!\top}$ (or blocks thereof) on modern GPUs. 
The two software implementations that we provide---one in Python, and one in MATLAB---perform these matrix multiplications on NVIDIA GPUs when that option is available. We have also taken further steps to accelerate these GPU-based implementations, mimicking other aspects of hardware-aware algorithm design that have been highly impactful in other areas of machine learning~\cite{dao2022-flash1,dao2024-flash2,shah2024-flash3,zadouri2026-flash4}. We describe these next.

\begin{figure}[t]
\centerline{\includegraphics[width=.48\textwidth]{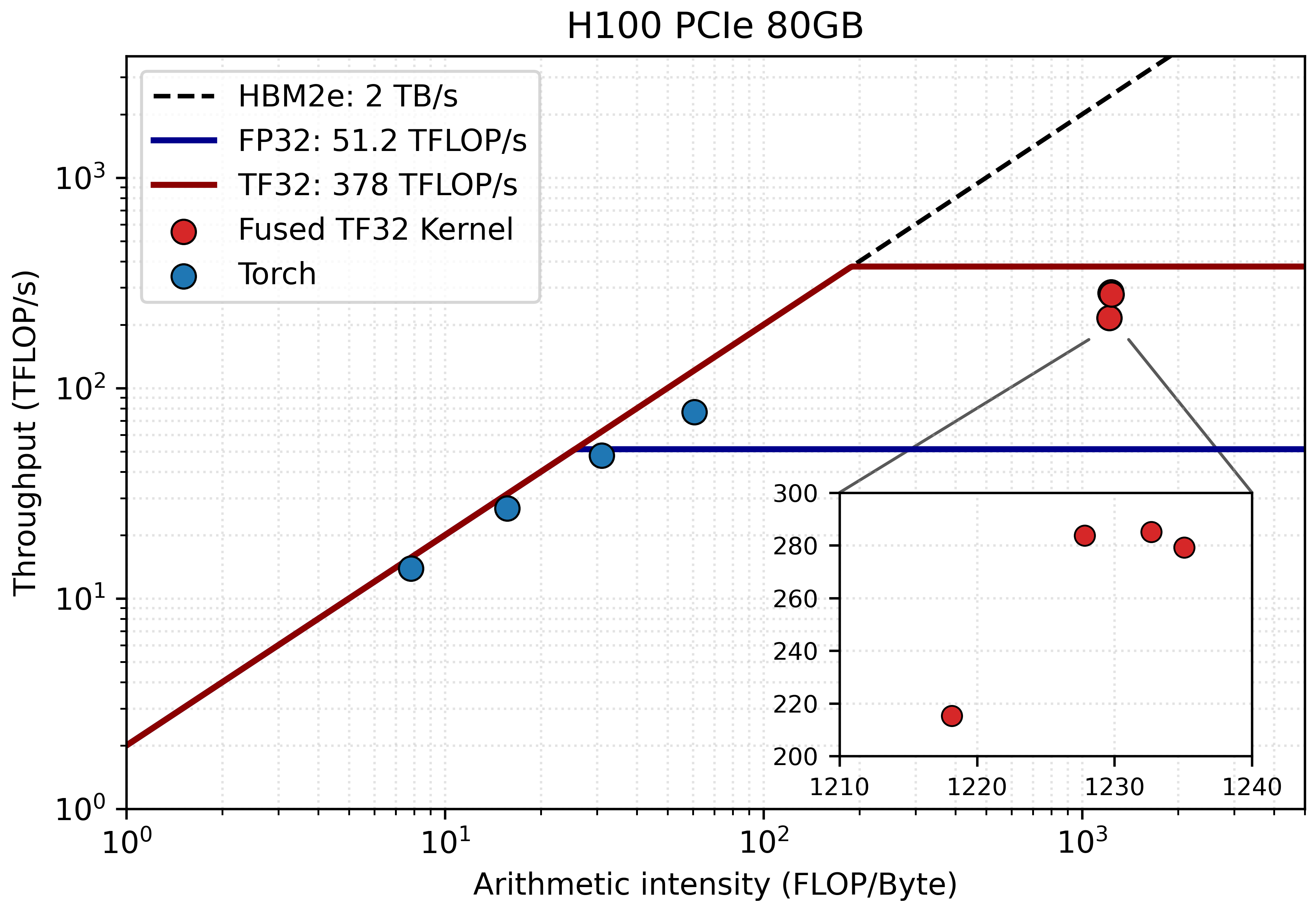}
\hspace{2ex}
\includegraphics[width=.48\textwidth]{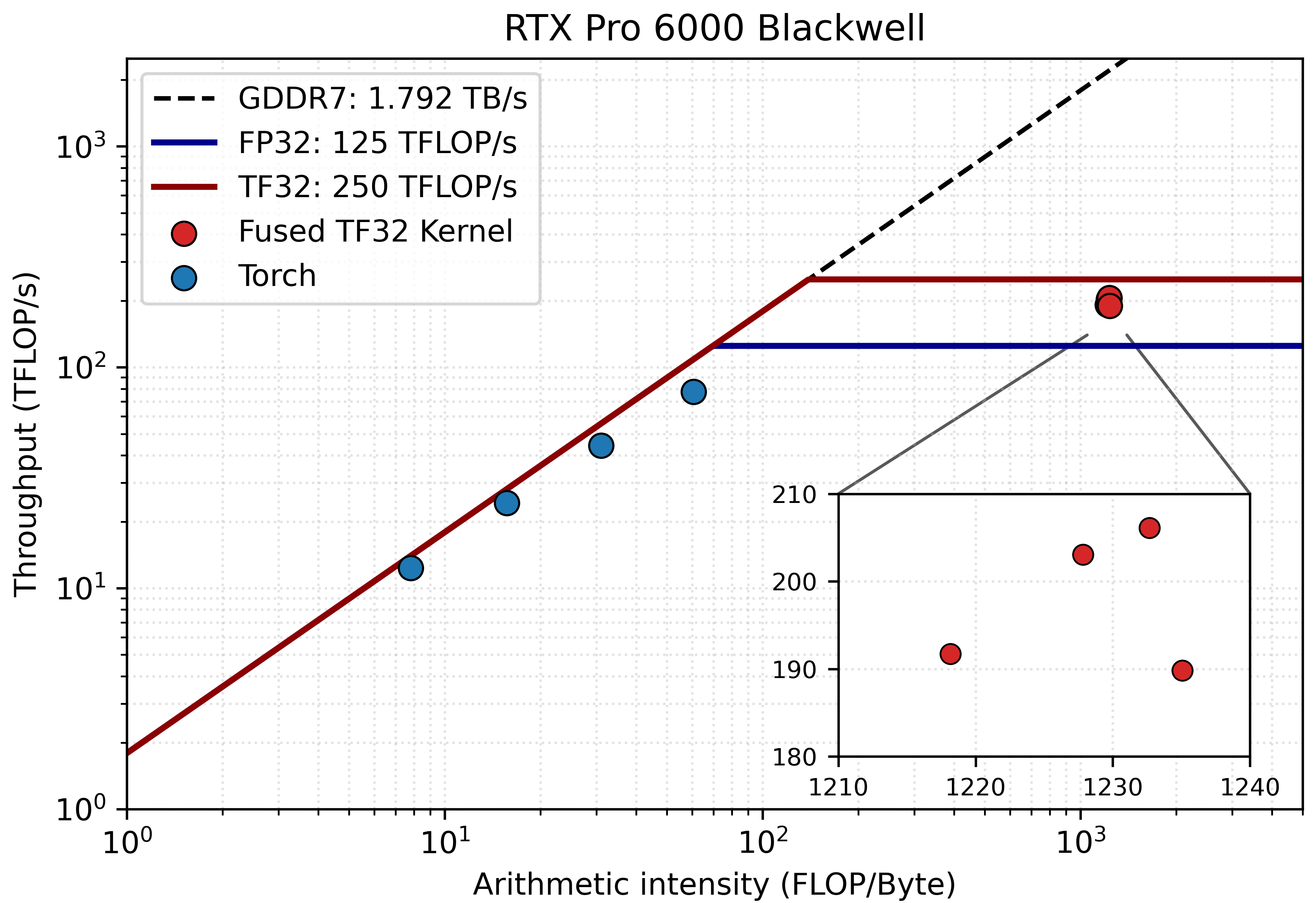}}
\caption{Roofline plots showing the theoretical peak FP32 and TF32 throughput ceilings for the H100 and RTX Pro 6000 Blackwell GPUs as a function of arithmetic intensity. The dots show the achieved performance of our fused TF32 kernels (in red) and a two-kernel epilogue-fused PyTorch baseline (in blue); see \cref{tab:kernel-throughput} for details.}
\label{fig:roofline}    
\end{figure}

\vspace{-3ex}
\paragraph{Kernel fusion.}
We have written custom, I/O-aware CUDA kernels that fuse multiple operations on the GPU that would otherwise be handled inefficiently by separate kernels. Consider for example the operations required by the full-batch updates in \cref{eq:updateD,eq:spLsqA}. One of our fused kernels takes matrices $\mat{A}$, $\mat{B}$, and $\mat{B}^\dagger$ (the pseudoinverse) as input and, in one call to the GPU, computes the~matrix
\begin{equation}
\max\big(0,\mat{A}\mat{B}^{\!\top}\big)\big(\mat{B}^\dagger\big)^\top
\label{eq:kernel}
\end{equation}
that appears implicitly on the right side of \cref{eq:spLsqA}. Note that without a fused kernel this calculation would require at least two separate calls to the GPU---the first to multiply the factors~$\mat{A}$ and~$\mat{B}^{\!\top}$ and rectify their product via epilogue fusion, then a second to multiply the rectified product by the transpose of~$\mat{B}^{\dagger}$.  When $\mat{A}$ and $\mat{B}$ are tall, skinny matrices, the matrix multiplications in \cref{eq:kernel} are gated by memory bandwidth rather than the GPU's peak floating-point throughput. Specifically, in a two-kernel implementation, the large intermediate matrix $\mat{A}\mat{B}^{\!\top}$ is written to GPU memory only to be read again by the subsequent kernel that multiplies it by the pseudoinverse of~$\mat{B}^{\!\top}$. Our fused kernel eliminates this memory inefficiency. 

\vspace{-3ex}
\paragraph{Arithmetic intensity.}
Performance gains from kernel fusion can be quantified by the \textit{arithmetic intensity}, measured in FLOPs/byte, and defined as the ratio of floating-point operations performed by the GPU to the total bytes transferred to and from memory. Let $I_\text{product}$ denote the arithmetic intensity of the single GPU call for calculating the matrix product $\mat{A}\mat{B}^\top\!$ (where
$\mat{A}\!\in\!\mathbb{R}^{m\times r}$ and $\mat{B}\!\in\!\mathbb{R}^{n\times r}$), and let $I_{\text{fused}}$ denote the arithmetic intensity of the fused kernel for calculating \cref{eq:kernel}. In appendix~\ref{app:gpu}, we show that
\begin{equation}
\frac{I_\text{fused}}{I_\text{product}}\ \geq\ 1 + \frac{mn}{(m\!+\!n)r},
\label{eq:AIratio}
\end{equation}
so that the fused kernel always yields a much higher arithmetic intensity in the typical regime for subzero matrix completion where $m,n\!\gg\! r$. Finally, we note that the gains are smaller but still quite significant when the kernel is used to multiply smaller sub-blocks of the matrices $\mat{\Delta}$ and $\mat{B}$, as in the stochastic least-square update of \cref{eq:lsqA_sto}. In this case, the ratio of arithmetic intensities is computed by replacing the number of columns~$n$ in \cref{eq:AIratio} by the smaller block size, $n_\text{block}$. For a typical problem in subzero matrix completion, with (say) a latent rank of $r\!=\!10$, a sparse matrix with $m\!=\!n\!=\!10^5$ rows and columns, and a block size of $n_\text{block}\!=\!10^3$, the arithmetic intensity of the fused kernel is still larger by a factor of over 100.

\vspace{-3ex}
\paragraph{Tensor Cores.}
With fused kernels, the stochastic updates for subzero matrix completion are no longer memory-bound. In this regime, further speedups can be obtained from Tensor Cores, specialized hardware units on modern NVIDIA GPUs that perform matrix multiply-accumulate instructions significantly faster than conventional floating-point compute units. The gains are made possible by the TensorFloat32 (TF32) reduced-precision floating-point format~\cite{stosic2021-accelerating} at the cost of a slight but acceptable loss in numerical precision. We 
provide TF32 implementations of the kernel in \cref{eq:kernel} for multiple Tensor Core generations; see Appendix~\ref{app:gpu} for further details. In our Python implementation, we use autotuning to select various optimal parameters of these kernels (e.g., tile sizes, operand mode, instruction shapes) based on the sizes of their matrix inputs. We also combine this autotuning with the just-in-time (JIT) compilation of NVIDIA's runtime compiler and include an option to back off to FP32 fused kernels.

\vspace{-3ex}
\paragraph{Performance gains.} Our custom fused TF32 kernels lead to significant speedups over a JIT-compiled PyTorch baseline that uses two kernel launches with epilogue fusion to compute \cref{eq:kernel}. \Cref{fig:roofline} compares the computational throughput of these kernels to theoretical performance rooflines as a function of arithmetic intensity for matrix sizes $m\!=\!250000$, $n_\text{block}\!=\! 2500$ and ranks $r\!\in\!\{16,32,64,128\}$. Here we see that our custom kernels 
achieve up to 75\% and 82\% of peak theoretical TF32 throughput, respectively, on the H100 PCIe 80GB and RTX Pro 6000 Blackwell Workstation GPUs. \Cref{tab:kernel-throughput} reveals the breakdown by rank, and for small ranks ($r\!<\!32$) we see that the custom fused kernels achieve order-of-magnitude increases in throughput. Moreover, for larger ranks---a regime in which the GPU is hardly idle---the increased throughput translates even more directly into shorter, overall wall-clock times.

\begin{table}[t]
\centering
\begin{tabular}{c} \\
\textbf{H100} \\
\textbf{PCIe} \\
\textbf{80GB}
\end{tabular}
\begin{tabular}{ccc}
\toprule
rank & fused TF32 & PyTorch \\ \midrule
16  & 215.30 & 13.89 \\
32  & 283.73 & 26.84 \\
64  & 285.10 & 47.83 \\
128 & 279.19 & 77.05 \\
\bottomrule
\end{tabular}
\hspace{7ex}
\begin{tabular}{c} \\
\textbf{RTX} \\
\textbf{Pro 6000} \\
\textbf{Blackwell}
\end{tabular}
\begin{tabular}{ccc}
\toprule
rank & fused TF32 & PyTorch \\ \midrule
16  & 191.72 & 12.35 \\
32  & 203.06 & 24.33 \\
64  & 206.11 & 44.32 \\
128 & 189.80 & 77.48 \\
\bottomrule
\end{tabular}
\caption{Achieved throughputs in TFLOP/s on H100 (\textit{left}) and RTX GPUs (\textit{right}) for the fused TF32 kernels versus the two-kernel PyTorch implementation with an option for TF32 acceleration.}
\label{tab:kernel-throughput}
\end{table}

\subsection{Exploiting sparsity}
\label{sec:sparse}

In the previous sections, we have addressed memory and computation-related bottlenecks for the alternating minimization in \cref{sec:altmin}. These bottlenecks arise from the fact that the latent matrix $\mat{Z}$ in \cref{eq:updateZ}, though of same large size as the matrix~$\mat{S}$, is neither sparse \textit{nor} low-rank. The stochastic updates conserve memory by working with randomized blocks of this matrix, and the custom CUDA kernels yield further speedups by optimizing the power of modern GPUs. In this section we describe the starting point for a rather different approach.

\vspace{-3ex}
\paragraph{Origins of sparsity.} Though the matrix $\mat{Z}$ is neither sparse nor low-rank, interestingly \textit{it can often be written as the sum of a sparse matrix and a low-rank matrix}. We have, in fact, already hinted at this decomposition in \cref{eq:Delta}, for if the matrix $\mat{\Delta}$ is sparse, then it follows at once that $\mat{Z}=\mat{\Delta}+\mat{A}\mat{B}^{\!\top}$ is sparse plus low-rank. Furthermore, from the update in \cref{eq:updateD}, we see that $\mat{\Delta}$ cannot have more nonzero elements than the \textit{combined} number of nonzero elements in~$\mat{S}$ and $\hat{\mat{S}}$, where the latter (repeated here for convenience) is equal to the rectified low-rank matrix
$$\mat{\hat{S}} = \max(0,\mat{A}\mat{B}^{\!\top}).$$
In sum, $\mat{Z}$ will be sparse plus low-rank whenever $\mat{\Delta}$ is sparse, and $\mat{\Delta}$ will be sparse whenever $\mat{\hat{S}}$ is sparse. But when do expect the latter to be the case?
The answer is \textit{almost always}---at least, almost always if the matrices $\mat{A}$ and $\mat{B}$ are sensibly initialized, if the rank $r$ is not too small, and if the matrix~$\mat{S}$ is highly sparse. In particular, suppose that we \textit{initialize $\mat{A}$ as a strictly positive matrix and~$\mat{B}$ as a strictly negative matrix}. With this choice of initialization, the rectified matrix $\hat{\mat{S}}$ is equal to the all-zero matrix (i.e., completely sparse) at the outset of the alternating minimization. Suppose furthermore that the model of subzero matrix completion has sufficient capacity (i.e., a sufficiently high rank $r$) to explain the pattern of sparsity in $\mat{S}$. Then from this initialization, one also expects $\hat{\mat{S}}$ to remain fairly sparse as a consequence of minimizing the objective in \cref{eq:auxCost}, because otherwise it could hardly provide a reasonable approximation to $\mat{S}$. 

\vspace{-3ex}
\paragraph{Sparse matrix multiplication.}
Consider the two primary computational bottlenecks addressed by the custom CUDA kernels of \cref{sec:GPU}. Both involve matrix multiplications: the first is to multiply the factors $\mat{A}$ and $\mat{B}^\top$ (or blocks thereof) that appear in \cref{eq:kernel}, and the second is to multiply $\mat{\hat{S}}$ by the pseudoinverse of $\mat{B}^\top\!$. We can immediately exploit the sparsity of $\hat{\mat{S}}$ to accelerate the second of these operations, with an expectation that the speedup should be commensurate with the sparsity of $\mat{\hat{S}}$. One can also see this more directly in the least-squares updates of \cref{eq:spLsqA,eq:spLsqB}; if $\mat{\hat{S}}$ is sparse, then from \cref{eq:updateD} so is $\mat{\Delta}$, and as observed previously, the least-squares updates for the factors $\mat{A}$ and $\mat{B}$ are both linear in the matrix $\mat{\Delta}$.

\vspace{-3ex}
\paragraph{Fast similarity search.}
\label{sec:fss}
Though less obvious, sparsity can also be exploited to address the primary computational bottleneck of calculating $\mat{\hat{S}}$ in \cref{eq:reluAB} from the factors $\mat{A}$ and $\mat{B}$. Note that in order to compute $\hat{\mat{S}}$, it is necessary to know the magnitudes of the positive elements in $\mat{A}\mat{B}^{\!\top}$, but not the magnitudes of the negative ones. In particular, suppose there was an oracle that provided just the set of indices, $\Omega$, where $\mat{A}\mat{B}^{\!\top}$ has strictly positive elements: 
\begin{equation}
\Omega = \left\{(i,j)\,\big|\,(\mat{A}\mat{B}^{\!\top})_{ij}>0\right\}.
\label{eq:nonzeros}
\end{equation}
Then to compute $\hat{\mat{S}}$, it would not be necessary to compute the whole matrix product of $\mat{A}\mat{B}^\top$, but only to compute the very small fraction of nonzero elements indicated by $\Omega$. Such an oracle might seem fanciful, as it would need to solve the \textit{needles-in-the-haystack} problem of identifying the positive elements of $\mat{A}\mat{B}^{\!\top}$ without having to compute~$\mat{A}\mat{B}^{\!\top}$ itself. But there is, in fact, a large literature on algorithms to solve this problem, particularly when $\mat{A}$ and~$\mat{B}$ are tall skinny matrices with a small number of columns. To make the connection to earlier work, let~$\vec{a}_i$ and~$\vec{b}_j$ denote, respectively, the $i^{\text{th}}$ row of $\mat{A}$ and the $j^\text{th}$ row of $\mat{B}$, and let $\hat{\vec{a}}_i\!=\!\vec{a}_i/\|\vec{a}_i\|$ and $\hat{\vec{b}}_j\!=\!\vec{b}_j/\|\vec{b}_j\|$ denote the normalized unit vectors for these rows. Then we can rewrite the set of indices in \cref{eq:nonzeros}~as
\begin{equation}
\Omega = \left\{(i,j)\,\Big|\,\big\|\hat{\vec{a}}_i\!-\!\hat{\vec{b}}_j\big\|^2<2\right\}.
\label{eq:range-search}
\end{equation}
Thus the oracle mentioned above can also be viewed as solving the following problem: given two large collections of vectors, identify those pairs of vectors---consisting of one in each collection---that are less than two units of distance apart. This is the problem of \textit{range search}, closely related to the problem of \textit{nearest-neighbor search}, and many algorithms have been proposed to tackle it efficiently for vectors that lie in a low-dimensional space~\cite{bentley1975-multidimensional,agarwal2017-range}. 

\vspace{-3ex}
\paragraph{Contention with GPU-based approaches.}
We experimented with storing $\hat{\mat{S}}$ in memory as a sparse matrix and with customized oracles to perform the range search in \cref{eq:range-search}. There is often contention, however, between these optimizations for storing and computing sparse matrices and the dense matrix optimizations for  the custom CUDA kernels in the previous section. Thus to some extent it becomes an empirical question---for problems of different sizes and degrees of sparsity, for implementations on different architectures, and for updates at different stages of the alternating minimization---whether the gains from these sparse optimizations combine or compete with GPU-based speedups. We summarize these findings in the next section.


\section{Experimental results}
\label{sec:experiment}

In this section we analyze the large sparse matrices in \cref{tab:data} with the more scalable updates for alternating minimization in \cref{sec:scaling}. In \cref{sec:recipes} we describe the particular combination of updates that was most effective in practice and compare its performance to that of leading gradient-based approaches. Next, in \cref{sec:approximation}, we compare the reconstruction accuracy of different low-rank models as measured by mean-squared error and weighted Jaccard distance. Finally, in \cref{sec:fly}, we present a downstream application of these models and show that cell descriptors of varying specificity can be accurately predicted from low-rank representations of the fly connectome.

\subsection{SUMAC and SALSA}
\label{sec:recipes}

We experimented with a large number of different combinations of the updates in \cref{sec:scaling}. We give only the briefest summary of these experiments: in the end, what worked best was the combination of updates that extracted random blocks of $\mat{S}$ for the stochastic least-squares updates in \cref{sec:stochastic} and maximized the gains from GPUs using the custom kernels in \cref{sec:GPU}. For the problem sizes in \cref{tab:data}, the stochastic least-squares updates were sufficiently memory-conserving that the parallel updates in \cref{sec:parallel} were not needed; likewise, the GPU-based speedups were much easier to realize throughout the course of the optimization than the gains from fast similarity search in \cref{sec:fss}. We have released software packages in Python and MATLAB, under the name \textbf{SUMAC}~\cite{sumac2026-github}, to solve the problem of \textbf{su}bzero \textbf{ma}trix \textbf{c}ompletion in this way, and at the core of these packages are the custom kernels (written in C\texttt{++} and inline assembly) for the \textbf{s}tochastic \textbf{a}lternating \textbf{l}east-\textbf{s}quares \textbf{a}lgorithm (\textbf{SALSA}) described in the previous section.

As a baseline, we compare the performance of these packages to those for more general-purpose optimizations. The problem of subzero matrix completion can also be solved by existing packages for gradient-based learning in neural networks. To do so, we simply convert the objective function in \cref{sec:loss} into one that can be optimized by standard approaches. We do this in two steps. First we use \cref{eq:updateZ} to eliminate the latent variable~$\mat{Z}$ from \cref{eq:auxCost} and obtain a squared loss,
\begin{equation}
\label{eq:SLloss}
    \|\mat{Z}-\mat{L}\|_F^2\ =\  \sum_{S_{ij}=0} \max\left(0,L_{ij}\right)^2 + \sum_{S_{ij}>0} 
    (S_{ij}\!-\!L_{ij})^2,   
\end{equation}
 that is expressed entirely in terms of the sparse matrix $\mat{S}$ and the latent low-rank matrix $\mat{L}$. Next we substitute $\mat{L}\!=\!\mat{A}\mat{B}^\top$ into \cref{eq:SLloss} to obtain a squared loss
that is expressed entirely in terms of the sparse matrix $\mat{S}$ and the factors $\mat{A}$ and $\mat{B}$. Once the loss is expressed in this way, it can be minimized by standard gradient-based packages, such as the ADAM optimizer~\cite{kingma2015-adam} in PyTorch, that alternately view the rows and columns of $\mat{S}$ as outputs of a rectified-linear neural network~\cite{mazumdar2019-learning,saul2022-nmd} and the factors $\mat{A}$ and $\mat{B}$ as inputs or weight matrices.

\begin{figure}
\centerline{\includegraphics[width=0.9\textwidth]{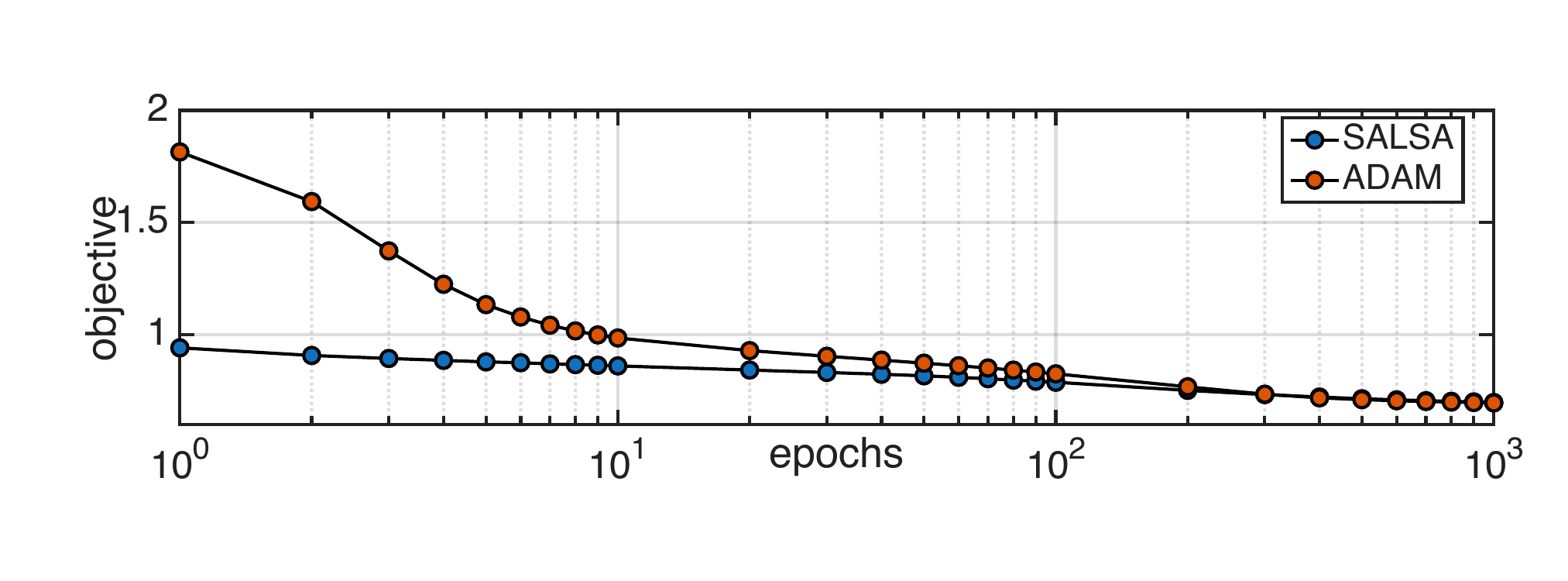}}
\vspace{-2ex}
\caption{Objective in \cref{eq:auxCost} versus the number of epochs, averaged over five different random seeds, for models of subzero matrix completion of rank 16 by SALSA and ADAM on the fly connectome matrix. SALSA converges more rapidly at the outset (and in shorter wall-clock times), while ADAM performs slightly better in the tail; see text for details.}
\label{fig:costs}
\end{figure}

\Cref{fig:costs} compares the convergence of the SALSA optimizer versus the ADAM optimizer in PyTorch with default hyperparameters. The plot shows the objective in \cref{eq:auxCost} for the fly connectome as a function of the number of epochs (i.e., loops over the rows and columns of $\mat{S}$), and the results have been averaged over five different initializations for a subzero matrix completion of rank 16. Both algorithms were initialized in the same way, so that the cost at epoch zero (not shown in the figure) was equal to~2, and both also used the same mini-batch size and momentum. The results show that SALSA decreases the cost more quickly at the outset while ADAM performs similarly or marginally better in the tail. 
A more striking contrast emerges, however, if we examine the convergence in terms of \textit{wall-clock time}. The ADAM optimizer benefits from the GPU acceleration in PyTorch, whereas the SALSA optimizer benefits from the custom CUDA kernels for its core routines. For the experiments shown in \cref{fig:costs}, the latter yield significantly shorter wall-clock times. In particular, on a single H100 GPU, these 1000 iterations take 9 minutes of wall-clock time for the ADAM optimizer in PyTorch versus 2-3 minutes for the SUMAC packages that we provide in Python and MATLAB. We remark that the ADAM optimizer in PyTorch can be further accelerated by using \emph{multiple} GPUs, and this support for SUMAC is a natural next step.

\subsection{Matrix approximation}
\label{sec:approximation}

We have already observed that the sparse matrices in \cref{tab:data} have slowly decaying spectra of singular values. In this section we compare the low-rank models obtained from truncated SVDs of these matrices to those obtained from SUMAC. Most of the SUMAC models were trained for 1000 epochs with default parameters of $\eta\!=\!1$ for the learning rate, $\gamma\!=\!0.9$ for the momentum, and a mini-batch size equal to $(1/100)^\text{th}$ of the number of rows or columns in the matrix $\mat{S}$. Further details of these experiments can be found in Appendix~\ref{app:eval}.

\Cref{fig:approx} compares the matrix approximation errors from SVD and SUMAC, as measured by the normalized root-mean-squared errors (RMSEs) in \cref{eq:RMSE} and the weighted Jaccard distances (WJDs) in \cref{eq:WJD}. The errors for SUMAC in these plots were averaged over models initialized by five different random seeds. (The standard deviations of these errors are not shown as they are smaller than the size of the plot markers.) For each of the analyzed matrices, we observe a crossover value of the rank above which the approximation errors from SVD and SUMAC are comparable and below which the errors from SUMAC decay much more rapidly. Furthermore, in all cases this crossover value is extremely small relative to the size of these matrices---less than 10 for the digits and connectome matrices, and less than 100 for the bigram counts. Interestingly, in each matrix this crossover value is also commensurate with the rank of the largest positive submatrix identified in \cref{tab:data}. 
The results show that SUMAC can reveal latent low-rank structure that purely linear models cannot.

\begin{figure}
\centerline{\hspace{1.7in}\includegraphics[width=4.8in]{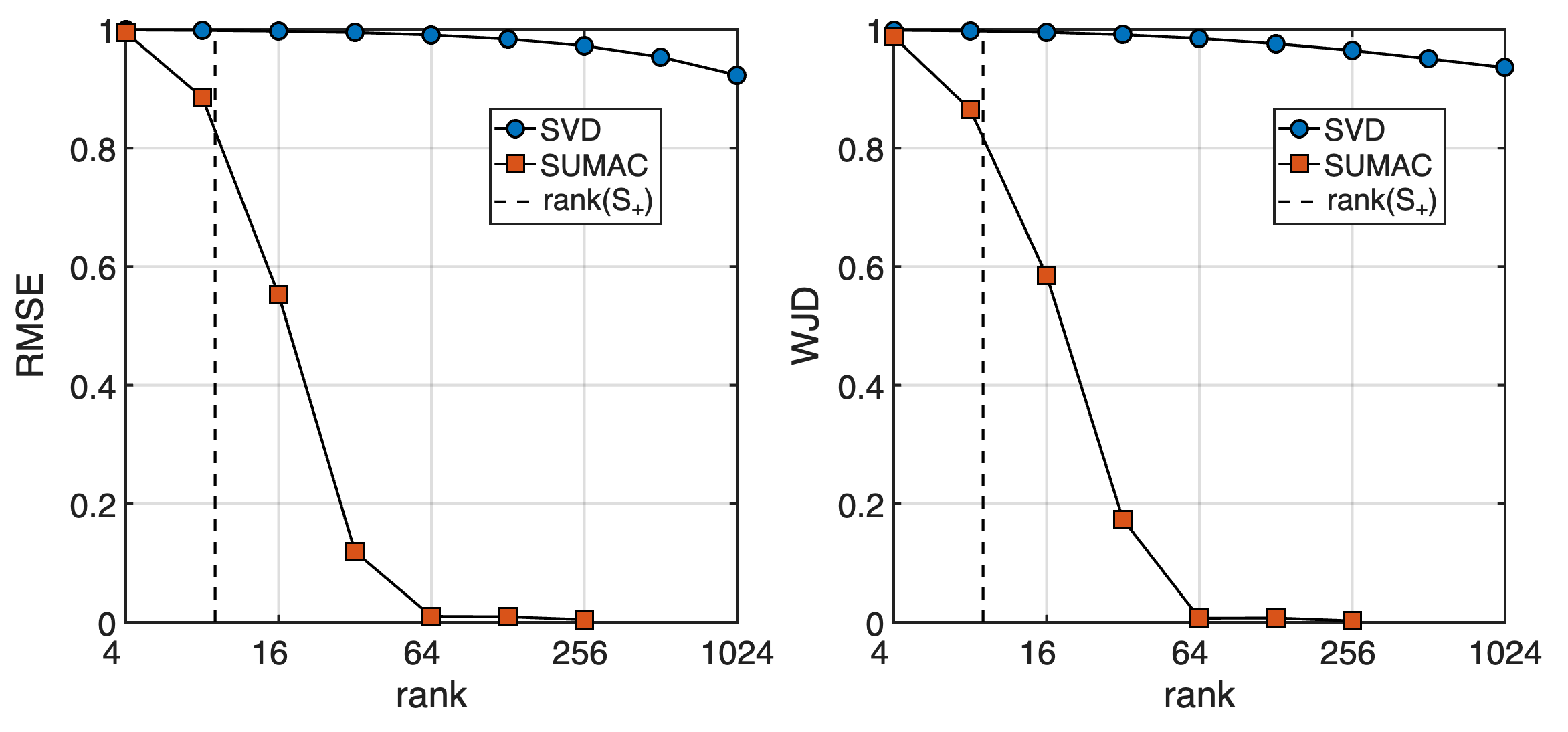}}
\vspace{2ex}
\centerline{\hspace{1.7in}\includegraphics[width=4.8in]{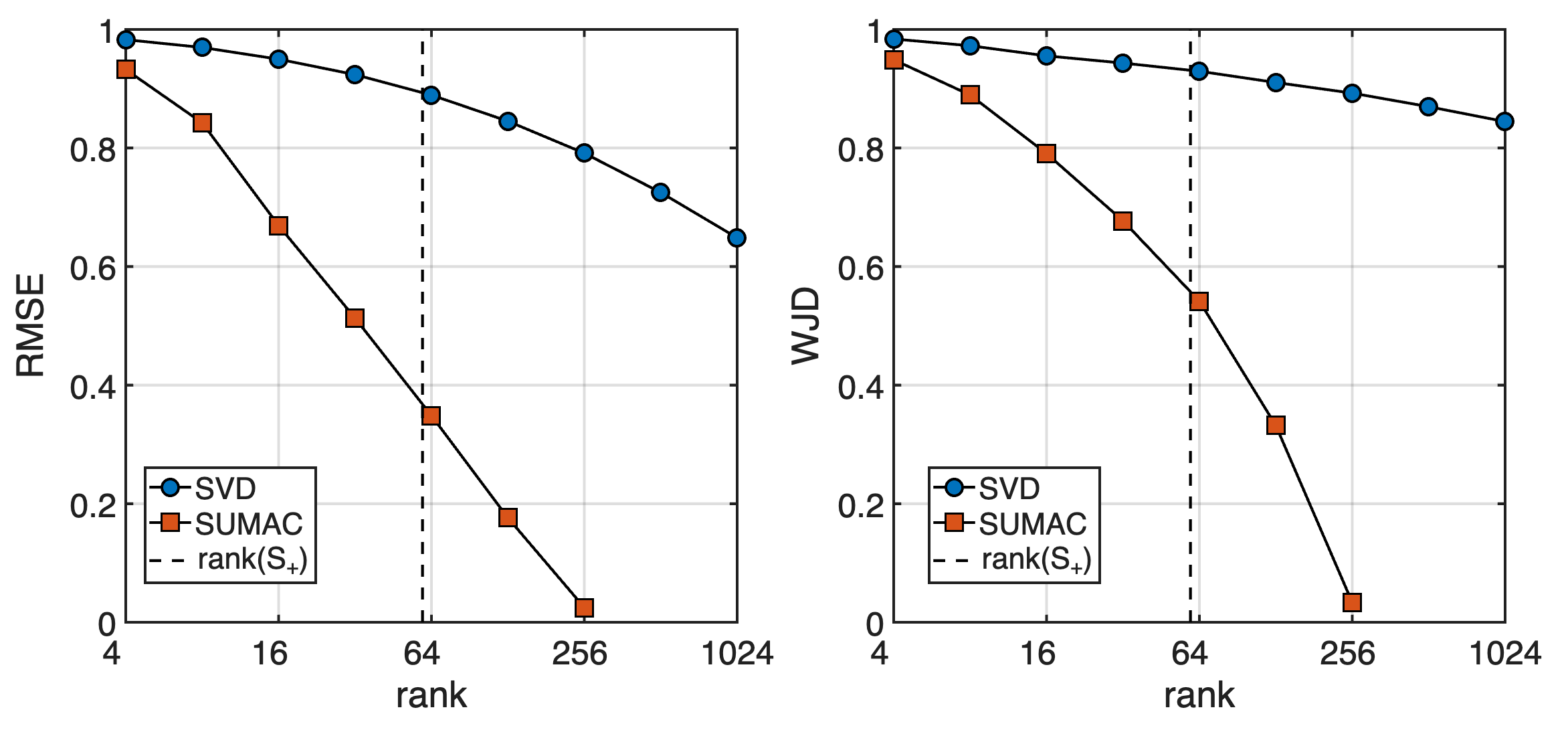}}
\vspace{2ex}
\centerline{\hspace{1.7in}\includegraphics[width=4.8in]{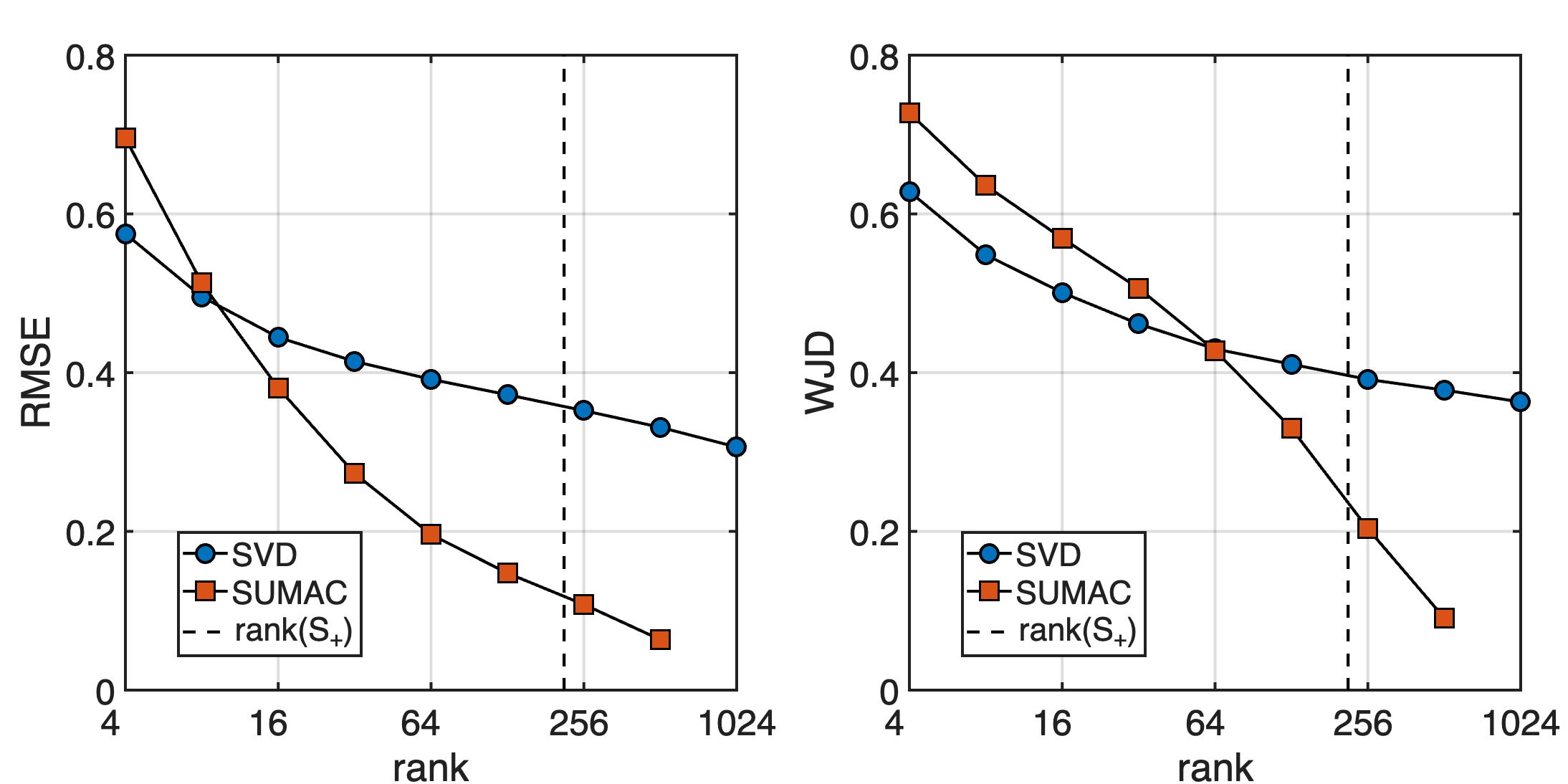}}
\vspace{-6.5in}\textbf{DIGITS} \\[0.5ex] 
$m\!=\!70,000$ \\ $\mbox{\hspace{0.5ex}}n\!=\!70,000$ \\[1.8in]
\textbf{CONNECTOME} \\[0.5ex]
$m\!=\!139,003$ \\ $\mbox{\hspace{0.5ex}}n\!=\!138,955$ \\[1.9in]
\textbf{BIGRAMS}\\ [0.5ex]
$m\!=\!249,992$ \\ $\mbox{\hspace{0.5ex}}n\!=\!249,987$ \\[0.8in]
\label{fig:approx}
\caption{Comparison of matrix approximation errors from truncated SVD and SUMAC with stochastic alternating least-squares updates. \textit{Left:} normalized root-mean-squared errors from \cref{eq:RMSE}. \textit{Right:} weighted Jaccard distances  from \cref{eq:WJD}. The dashed lines show the rank of the largest positive submatrix found by a randomized greedy search.} 
\end{figure}

Curiously we observe the errors from SUMAC to exceed those from SVD in some of the most severely underparameterized models (e.g., of rank $r\leq 32$ for the matrix of normalized bigram counts with 250K rows and columns). Presumably this occurs due to the existence of shallow local minima in the vicinity of the initialized values of the factors~$\mat{A}$ and~$\mat{B}$; recall that these factors were not initialized from the truncated SVD, but rather in a way that favors the sparsest possible updates in \cref{eq:spLsqA,eq:spLsqB}. In addition, other experiments suggest that these underparameterized models are less amenable to stochastic updates, mainly because the leading subspaces they compute are not stable across different mini-batches of rows and columns of the sparse matrix $\mat{S}$.


\subsection{Cell classification}
\label{sec:fly}

In this section we more closely examine the results of these models for the fly connectome. These models were learned entirely from the counts of synapses between cells, but the dataset also provides labels from human annotators that classify the majority of cells into categories of varying specificity~\cite{schlegel2024-whole}; see \cref{tab:classes}. For example, all of the 139,255 cells are labeled by \textsc{superclass} with one of the following nine labels: \texttt{optic}, \texttt{central}, \texttt{sensory}, \texttt{visual-projection}, \texttt{ascending}, \texttt{descending}, \texttt{visual-centrifugal}, \texttt{motor}, or \texttt{endocrine}. Most of the cells have also been more specifically labeled by \textsc{class} (with 27 labels), by \textsc{subclass} (with 96 labels), and by \textsc{cell type} (with 8547 labels). For each cell, it is known that many of these labels can be predicted from its synaptic connectivity with other cells~\cite{schwartzman2026-NTAC}. Here we investigate whether these labels can also be predicted by the latent low-rank structure of the connectome.

There are many ways in principle that one could explore this possibility. Here we adopt a simple approach based on nearest-neighbor classification in a low-dimensional embedding space. 
In particular, let $\mat{L}$ denote a low-rank matrix derived in some way from the sparse fly connectome $\mat{S}$. If~$\mat{L}$ is of rank $r$, then from the singular value decomposition
\begin{equation}
\mat{L} = \mat{U}\mat{\Sigma}\mat{V}^\top,
\label{eq:Lsvd}
\end{equation}
we can extract a diagonal $r\!\times\! r$ matrix $\mat{\Sigma}$ and semi-orthogonal matrices $\mat{U}$ and $\mat{V}$ of sizes $n\!\times r$, where $n$ is the number of cells in the connectome. We use these matrices to derive a low-dimensional embedding of the cells, where each cell is represented by a unit vector in $\mathbb{R}^{2r}$. In particular, to the $i$th cell, we associate the vector $\vec{\xi}_i$ with elements
\begin{equation}
\vec{\xi}_i \propto 
\left[\begin{array}{c} \!\!\vec{u}_i\!\! \\ \!\!\vec{v}_i\!\!\end{array}\right]
\mat{\Sigma}^{\frac{1}{2}},
\label{eq:embed}
\end{equation}
where we have used $\vec{u}_i$ and $\vec{v}_i$ to denote the $i$th rows of $\mat{U}$ and $\mat{V}$, respectively, and where the constant of proportionality is chosen so that $\|\xi_i\|=1$. Note that the first $r$ elements of $\vec{\xi}_i$ store information about the outgoing synapses of the $i$th cell, while the last~$r$ elements store information about the incoming ones. There are many possible variations of \cref{eq:embed} that can also be used to compute a low-dimensional embedding of cells (e.g., with different weightings or normalizations of the vectors~$\vec{u}_i$ and~$\vec{v}_i$), but for brevity we focus on the above.

\begin{table}[t]
\begin{center}
\begin{tabular}{lrrr} 
\toprule
 category\hspace{3ex} & \hspace{2ex}cardinality & \hspace{5ex}labeled & \hspace{5ex}missing \\ \midrule
superclass & 9 & 139,255 & 0 \\
class & 27  & 106,969 & 32,286 \\
subclass & 96 & 94,946 & 44,309 \\
cell type & 8,547 & 137,677 & 1,578 \\ \bottomrule
\end{tabular}
\end{center}
\caption{Numbers of labeled cells in the connectome for categories of differing specificity.}
\label{tab:classes}
\end{table}

Next we consider how the embedding in \cref{eq:embed} can be used to probe whether low-rank structure in the connectome, as captured by $\mat{L}$, is correlated with the cell labels in \cref{tab:classes}. If there are strong correlations, then one would expect similarly labeled cells $(i,i')$ to be separated by shorter pairwise distances $\|\vec{\xi}_i-\vec{\xi}_{i'}\|$ in the embedding space, and conversely, one would expect differently labeled cells to be separated by larger pairwise distances. One simple measure of this tendency is given by the held-out one-nearest-neighbor (1\textsc{nn}) error rate. In particular, for each category in \cref{tab:classes}, and for each labeled cell in this category, we can locate the nearest labeled cell in the embedding and observe if this nearest neighbor shares the same category label. The 1\textsc{nn} error rate is given by the proportion of disagreements.

\cref{tab:1nn-svd} shows the 1\textsc{nn} error rates obtained in this way from truncated SVDs that extract the \textit{linear} low-rank structure of the connectome matrix. These error rates are far lower than what would be obtained from random guessing, showing that (a) the cell labels are highly correlated with patterns of synaptic connectivity, and (b) the linear low-rank structure provides a summary description of these patterns that is considerably more compact than the connectome matrix itself. At the same time, it is also true that across all cell categories, the lowest error rates are obtained from the embeddings that are not especially low-dimensional. In particular, they are obtained from the embeddings in the last two rows of the table, where each cell in the connectome is represented by a vector with hundreds or thousands of elements.

By comparison, \cref{tab:1nn-sumac} shows the 1\textsc{nn} error rates, averaged over ten different random seeds, obtained from subzero matrix completions that look for \textit{nonlinear} low-rank structure in the connectome matrix; i.e., they seek a low-rank matrix $\mat{L}$ in \cref{eq:Lsvd} for which $\mat{S}\approx\max(0,\mat{L})$. The most striking difference here is that comparable or lower error rates are obtained by embeddings of much lower dimensionality. The results in \cref{tab:1nn-sumac} are shown for models of ranks between 8 and 16, while the results in \cref{tab:1nn-svd} are shown for truncated SVDs of ranks up to 1024.

We observed a surprising trend in these results as well. The more accurate results in \cref{tab:1nn-sumac} were obtained from models of subzero matrix completion that were only optimized for five epochs. This is a positive but mysterious result that we do not fully understand---positive because such models can be estimated very quickly, but mysterious because such models seem rather crude as measured by metrics such as the RMSE, WJD, and cost in \cref{eq:RMSE,eq:WJD,eq:auxCost}. To provide further context, \cref{tab:iter} shows the \textsc{superclass} 1\textsc{nn} error rates from subzero matrix completions that were optimized for different ranks and for different numbers of epochs; qualitatively similar results (not shown) were also obtained for the other cell categories. It is understandable that the highest-rank models, trained for the largest numbers of epochs, suffer from some sort of \textit{overfitting} to the objective in \cref{eq:auxCost} and therefore do not yield the lowest 1\textsc{nn} error rates for cell classification. Nevertheless, it is mysterious that the lowest error rates are obtained by low-rank models of subzero matrix completion that are hardly optimized at all, whose factors $\mat{A}$ and $\mat{B}$ in \cref{eq:ABt} are estimated by 10 or fewer epochs. By contrast, the embeddings obtained from truncated SVDs do not appear to suffer from this type of overfitting (or at least, not for the ranks up to 1024 shown in \cref{tab:1nn-svd}). We leave a more careful study of these overfitting issues to future work.

\begin{table}[t]
\begin{center}
\begin{tabular}{c}
\\
\textbf{SVD} \\[0.5ex]
\textsc{1nn} error \\ rate
(\%) \\
\end{tabular}
\hspace{4ex}
\begin{tabular}{rrrrr}
  \toprule
  rank & 
  \hspace{3ex}superclass & \hspace{3ex}class & \hspace{3ex}subclass & \hspace{3ex}cell type \\ \midrule
  4 & 13.3 & 10.6 & 43.4 & 64.6 \\
  16  & 3.96 & 2.29 & 15.7  & 35.0 \\
  64 & 2.48 & 1.20 & 6.41 & 22.2 \\
  256 & 1.89 & 0.94 & \textbf{4.30} & 16.1 \\
  1024 & \textbf{1.70}  & \textbf{0.83} & 4.79 & \textbf{15.1} \\
\bottomrule
\end{tabular}
\hspace{20ex}
\end{center}
\caption{\textsc{1nn} error rates (\%) from row and column embeddings obtained from different truncated SVDs of the connectome matrix.}
\label{tab:1nn-svd}
\end{table}

\begin{table}[t]
\vspace{2ex}
\begin{center}
\begin{tabular}{c}
\\
\textbf{SUMAC} \\[0.5ex]
\textsc{1nn} error \\ rate
(\%) \\
\end{tabular}
\hspace{4ex}
\begin{tabular}{rrrrr}
  \midrule  rank & 
  \hspace{5ex}superclass & \hspace{10ex}class & \hspace{7ex}subclass & \hspace{6ex}cell type \\ \midrule
   8  & $1.49 \pm 0.07$ & $0.85 \pm 0.06$ & $6.91 \pm 0.48$ & $21.3 \pm 0.49$ \\
   10  & $1.31\pm 0.06$ &  $0.74\pm 0.05$ & $5.57\pm 0.34$ & $19.1\pm 0.56$ \\
   12  & $1.21\pm 0.04$ & $0.68\pm 0.04$ &  $4.84\pm 0.48$ &  $18.1\pm 0.55$ \\
   14  & $1.15\pm 0.03$ &  $0.63 \pm 0.03$ &  $4.25\pm 0.26$  &  $17.3 \pm 0.27$ \\
   16  & $\textbf{1.13}\pm \textbf{0.02}$ & $\textbf{0.63}\pm \textbf{0.02}$ & $\textbf{4.19}\pm \textbf{0.14}$ &  $\textbf{17.2}\pm \textbf{0.26}$ \\
   \bottomrule
\end{tabular}
\end{center}
\caption{\textsc{1nn} error rates (\%) from row and column embeddings obtained from different subzero matrix completions of the connectome matrix. Each subzero matrix completion was optimized for only five iterations. The table shows the mean and standard deviation of the error rates over ten different random seeds.}
\label{tab:1nn-sumac}
\end{table}

\begin{table}[h]
\begin{center}
\begin{tabular}{c}
\\
\textbf{SUMAC} \\
superclass \\
\textsc{1nn} error \\
rate (\%) \\
\end{tabular}
\hspace{4ex}
\begin{tabular}{rrrrrr}
   \toprule
  rank & 
 \hspace{2ex}epoch=2 & \hspace{2ex}epoch=5 & \hspace{1ex}epoch=10 & \hspace{1ex}epoch=20 & epoch=1000 \\ \midrule
   004 &      6.40 &      4.15 &      4.02 &      4.65 &     5.49 \\
   008 &      2.14 &      1.42 &      1.44 &      1.68 &     3.26 \\
   016 &      1.39 &      \textbf{1.13} &      1.21 &      1.30 &     2.69 \\
   032 &      1.20 &      1.21 &      1.34 &      1.56 &     2.11 \\
   064 &      1.35 &      1.54 &      1.65 &      1.67 &     2.11 \\
   128 &      1.54 &      1.59 &      1.56 &      1.63 &     2.07 \\
   256 &      1.89 &      1.83 &      1.97 &      2.27 &     2.60 \\
 \bottomrule
\end{tabular}
\end{center}
\caption{\textsc{1nn} \textit{superclass} error rate (\%) from row and column embeddings obtained from different subzero matrix completions of the connectome matrix.}
\label{tab:iter}
\end{table}


\section{Related work}
\label{sec:related}

Low-rank matrix factorizations are a foundational tool in statistical learning and high dimensional data analysis~\cite{markovsky2012-lowrank,wright2021-hdda}. The basic problem, arising across many disciplines~\cite{deerwester1990-indexing,brunet2004-gene,koren2009-matrix,hu2022-lora,turk1991-eigenfaces}, is to approximate a large real-valued matrix by a product of two smaller factors. If the elements of the large matrix are fully observed, then this problem can be solved by a truncated singular value decomposition~\cite{eckart1936-svd}; otherwise, a low-rank completion can be used to recover the missing elements, a setting in which many theoretical guarantees have been proven~\cite{candes2009-exact,candes2010-matrix,keshavan2010-matrix,sun2016-guaranteed,nguyen2019-low,chatterjee2020-deterministic,zilber2022-GNMR}. 
These basic problems of matrix factorization and completion have been generalized on a number of different fronts~\cite{collins2002-generalization,gordon2003-GLM,singh2008-unified,bartholomew2011-latent,udell2016-generalized,hoff2021-additive}. Individual elements of the factors can be required to be binary~\cite{meeds2007-modeling}, discrete~\cite{kolda1998-semidiscrete}, bounded~\cite{lee1996-unsupervised,ding2020-convex,thanh2023-bounded}, or nonnegative~\cite{paatero1994-positive,lee1999-nmf,gillis2021-nmf}; in addition, one or both factors can be constrained as a whole to be sparse~\cite{zou2006-sparse,daspremont2008-optimal,witten2009-pmd,soni2016-noisy}, stochastic~\cite{lee1996-unsupervised,saul1997-aggregate,hofmann1999-probabilistic,thanh2023-bounded}, or of low norm~\cite{srebro2005-maximum}. Likewise, further constraints can apply to the matrix to be encoded or completed. Previous works have considered settings where its elements are binary~\cite{tipping1999-probabilistic,hoff2002-latent,schein2003-generalized,davenport2014-1bit,bhaskar2015-1bit}, discrete or categorical~\cite{cao2013-categorical,lafond2014-probabilistic,bhaskar2016-probabilistic}, bounded~\cite{kannan2014-bounded}, integer~\cite{gopalan2014-poisson}, or nonnegative~\cite{mazumdar2019-learning}; the completed matrix as a whole can also be constrained to be nonnegative~\cite{song20-nonnegative} or positive-semidefinite~\cite{bishop2014-deterministic,liu2024-symmetric}. Uncertainty in the observed elements can be incorporated by probabilistic models that are fitted by maximum likelihood estimation, maximum a posteriori estimation, or Bayesian methods~\cite{saul1997-aggregate,hofmann1999-probabilistic,hoff2002-latent,hoff2005-bilinear,meeds2007-modeling,pmf-mnih08,handcock2007-model,bhaskar2016-probabilistic,ma2020-universal,hoff2021-additive,saul2022-nmd}. It is common to enforce elementwise constraints on the completed matrix by passing the low-rank product of factors through an elementwise quantization or nonlinearity~\cite{ganti2015-monotonic,mazumdar2019-learning,saul2022-nmd,awari2025-alternating}; in probabilistic models, and particularly in generalized linear models, these elementwise nonlinearities arise from the inverse of statistical link functions~\cite{collins2002-generalization,gordon2003-GLM}.  Many authors have proven theoretical guarantees and derived statistical rates of convergence for the recovery of latent low-rank structure in these more general settings~\cite{cao2013-categorical,davenport2014-1bit,lafond2014-probabilistic,bhaskar2015-1bit,bhaskar2016-probabilistic,soni2016-noisy,mazumdar2019-learning,ganti2015-monotonic,ma2020-universal,liu2024-symmetric}. Finally, the factors in these models are commonly learned via some form of alternating minimization~\cite{schein2003-generalized,srebro2003-wlr,singh2008-unified,wen2012-solving,udell2016-generalized,tanner2016-low,chi2019-nonconvex}.

Our work builds directly on earlier attempts to discover low-rank encodings of sparse nonnegative matrices via a rectified linearity~\cite{saul2022-nmd,seraghiti2023-accelerated,wang2024-momentum,wang2025-accelerated,gillis2025-extrapolated}. Previous studies have explored a number of alternating minimizations for this problem. Saul~\cite{saul2022-nmd} introduced a probabilistic latent variable model where the factors are estimated by an Expectation-Maximization algorithm~\cite{dempster1977-maximum}; he also described the alternating minimization in \cref{eq:updateZ,eq:tsvd}, based on truncated SVDs, as a particular limiting case. Seraghiti \textit{et al}~\cite{seraghiti2023-accelerated} showed that faster solutions could be obtained by parameterizing the low-rank matrix $\mat{L}$ in \cref{eq:auxCost} as a product of smaller factors and iteratively re-estimating these factors with least-squares updates; Gillis \textit{et al}~\cite{gillis2025-extrapolated} proved the convergence of this scheme as well as that of an even faster variant with an extrapolation step. In parallel work, Wang \textit{et al}~\cite{wang2024-momentum,wang2025-accelerated} introduced additional regularizing terms into the updates for the factors and proved convergence of an accelerated alternating minimization based on Bregman proximal gradient methods. Several related works have considered the specialized setting where the low-rank matrix in \cref{eq:auxCost} is constrained to be symmetric~\cite{saul2022-geometrical,wang2025-efficient} or positive-definite~\cite{liu2024-symmetric}. Finally, it has been observed that minima of the objective in \cref{eq:auxCost} with nonzero error do not generally coincide with those of the RMSE in \cref{eq:auxCost}~\cite{gillis2025-extrapolated}, and subsequent works have explored algorithms to minimize the latter based on coordinate-descent~\cite{awari2024-coordinate} and the alternating direction method of multipliers~\cite{awari2025-alternating}.

Similar connections between sparse and low-rank matrices arise in the problem of manifold learning. The goal of manifold learning is to discover a faithful mapping of high-dimensional inputs into a low-dimensional embedding space. Some algorithms discover this mapping by using a sparse matrix to encode nearest-neighbor relations and constructing a low-rank Gram matrix from which the embedding is derived~\cite{tenenbaum2000-global,weinberger2006-mvu}. Manifold structure can also be sparsely encoded by locally linear reconstruction weights~\cite{roweis2000-nonlinear,chen2018-sparse} and thresholded similarity matrices~\cite{sengupta2018-tiling,saul2022-geometrical}. 

Sparse boolean matrices of zeros and ones are a special case of sparse nonnegative matrices. The \textit{sign rank} of a sparse boolean matrix $\mat{S}$ is equal to the minimal rank of all matrices $\mat{L}$ such that $\mat{S}=\Theta(\mat{L})$, where~$\Theta$ is the Heaviside step function (whose value at zero is left undefined). Properties of the sign rank are of great interest to researchers in learning theory and communication complexity~\cite{paturi1986-probabilistic,alon2016-sign,hatami2022-sign}. It is known, for example, that the $n\!\times\! n$ identity matrix with $n\!\geq 3$ has a sign rank of~3, and this result~\cite{alon2016-sign} provides a corresponding lower bound to the result in \Cref{thm:ident}. It is also known that there exist very sparse matrices whose sign rank is large~\cite{sherstov2023-near}.


\section{Discussion}
\label{sec:discuss}

In this paper we have analyzed when sparse nonnegative matrices can be recovered from a real-valued matrix of much lower rank by zeroing out its negative elements. We have also described a scalable algorithm to discover this latent low-rank structure in sparse matrices with hundreds of thousands of rows and columns.

There remain several directions for future work. The stochastic updates in this paper scale to large problem sizes by operating on random blocks of the factors~$\mat{A}$ and $\mat{B}$ in \cref{eq:ABt}. It is likely that further speedups could be obtained by adaptive strategies for choosing the momentum and step sizes~\cite{wang2024-momentum,gillis2025-extrapolated,kingma2015-adam} and/or selectively invoking the sparse optimizations of \cref{sec:sparse}. It should also be possible to adapt randomized methods for SVD~\cite{mahoney2010-random,halko2011-finding,tropp2017-practical} to this problem, exploiting the decomposition of $\mat{Z}$ as a sparse plus low-rank matrix in \cref{eq:Delta}. Though the stochastic updates in this paper work well in practice, they can diverge for poorly chosen batch sizes (e.g., too small) or step sizes (e.g., too large); it would be desirable to have a formal proof of convergence (as in~\cite{gillis2025-extrapolated,wang2025-accelerated}) that informed these choices. Another direction for study is suggested by the results in \cref{sec:fly}, where the lowest 1\textsc{nn} error rates were obtained from subzero matrix completions that were optimized for only a few iterations. It seems that the performance in downstream tasks may be susceptible to overfitting, either because the objective in \cref{eq:auxCost} is a poor surrogate for those tasks or because the sparse matrix~$\mat{S}$ itself is a single noisy sample from an underlying generative process. It would be worthwhile to explore the use of regularization~\cite{wang2024-momentum,wang2025-accelerated} in this context as a more principled alternative to early stopping. Another hope is that the theoretical results in \cref{sec:motivation}, for various idealized templates of sparse matrices, could better inform how to initialize the optimizations for subzero matrix completion on real-world datasets. Finally, we mention that many so-called spectral methods for clustering~\cite{ng2001-spectral,vonluxburg2007-tutorial} and graph-matching~\cite{umeyama1988-eigen} are based on computing the eigenvalues and eigenvectors of sparse matrices. It is natural to wonder if these problems could be reformulated in terms of latent low-rank decompositions that encode~this~sparsity.


\section*{Acknowledgements}
LS thanks Sebastian Seung for his suggestion and help to analyze connectome data with these methods, Charlie Epstein for several fruitful discussions about relative primes and the matrix in \cref{eq:logGCD}, Ryan O'Donnell and Mohan Paturi for pointers to related work on the sign rank of matrices, and Nicolas Gillis and  Giovanni Seraghiti for several illuminating exchanges about nonlinear matrix decompositions during a visit to Mons in the fall of 2025.



\appendix

\section{Details on sparse datasets}
\label{app:data}

In this appendix we describe how the sparse nonnegative matrices in \cref{tab:data} were constructed from publicly available datasets.

\vspace{-2ex}
\paragraph{Digits.} The first matrix in \cref{tab:data} was constructed from the training and test images in the $\textsc{mnist}$ dataset of handwritten digits~\cite{lecun98-gradient}. The matrix has 70K rows and columns, and the $ij^\text{th}$ element is nonzero if and only if the $j^\text{th}$ image is one of the $k$ nearest neighbors of the $i^\text{th}$ image with $k\!=\!16$; here nearest neighbors were computed using cosine distance, and each digit was counted among its own~$k$ nearest neighbors. The values of elements were computed as follows. Let~$\vec{x}_i$ denote the 784-dimensional pixel vector of the $i^\text{th}$ image, and let $j_i$ and $j'_i$ denote the indices of the $k^\text{th}$ and $(k\!+\!1)^\text{th}$ nearest neighbors to this image. For each image, we defined a threshold, given~by
\begin{equation}
\tau_i = \tfrac{1}{2}\left[\cos\left(\vec{x}_i,\vec{x}_{j_i}\right) + \cos\big(\vec{x}_i,\vec{x}_{j'_i}\big)\right]
\end{equation}
for the $i$th image, so that its $k$-nearest neighbors were precisely those closer in angle than $\cos^{-1}(\tau_i)$. In terms of this threshold, we computed the matrix elements in the $i^\text{th}$ row as
\begin{equation}
S_{ij} = \max\left(0,\vec{x}_i\!\cdot\!\vec{x}_j-\tau_i\|\vec{x}_i\|\|\vec{x}_j\|\right).
\end{equation}
The matrix constructed in this way is over 99.9\% sparse with exactly 16 nonzeros in each of its 70000 rows. Note that it is also asymmetric. 

\vspace{-2ex}
\paragraph{Connectome.} The second matrix in \cref{tab:data} was constructed from the recently published connectome of a female fruit fly~\cite{dorkenwald2024-flywire,matsliah2023-codex}. Connectome weights and cell labels were downloaded from \url{https://codex.flywire.ai/api/download?dataset=fafb} on September 16, 2025. The $ij^\text{th}$ element in this matrix records the number of synapses between the $i^\text{th}$ and $j^\text{th}$ cell in the connectome. The matrix has 139255 rows and columns, but a handful of these rows and columns have no nonzero elements; only the rows and columns with at least one nonzero element are counted in \cref{tab:data}. The experiments in \cref{sec:fly} are based on cell labels from this same dataset. All of the 139255 cells in the connectome are labeled by \textsc{superclass}, while most of the cells are more specifically labeled by \textsc{class}, \textsc{subclass}, and \textsc{cell type}. Within each of these categories there is a set of possible labels: this set contains 9 labels for superclasses, 28 for classes, 97 for subclasses, and 8547 for cell types. \cref{tab:labels} shows the most frequent ten labels in each category. Note that the labels are not distributed evenly, and within each category the dominant label often accounts for a large percentage. Even so, the error rates in \cref{sec:fly} are much lower than what would be obtained by always predicting the dominant label (e.g., \texttt{optic} for \textsc{superclass}, \texttt{optic-lobe-intrinsic} for \textsc{class}).

\vspace{-2ex}
\paragraph{Bigrams.} The third matrix in \cref{tab:data} is a stochastic matrix derived from bigram counts for the 250K most frequently observed tokens in Wikipedia text~\cite{TC-wikipedia}. Rows of the matrix containing nonzero bigram counts were normalized so that they summed to one. A handful of rows and columns have no nonzero elements, corresponding to tokens that were not observed to be followed or preceded by any of the 250K most frequently observed tokens. Only the rows and columns with at least one nonzero element are counted in \cref{tab:data}.

\begin{table}
\scriptsize
\begin{center}
\begin{tabular}{lrclrclrclr} 
\toprule
{\sc superclass} & & \hspace{1ex} & {\sc class} & & \hspace{1ex} & {\sc subclass} & & \hspace{1ex} & \multicolumn{2}{l}{\sc cell type} \\
\midrule
\texttt{optic} & 77849 & & \texttt{optic-lobe-intrinsic} & 76928 & & \textit{unlabeled} & 44309 & & \texttt{R1-6} & 8525 \\
\texttt{central} & 32388 & & \textit{unlabeled} & 32286 & & \texttt{transmedullary} & 17947 & & \texttt{KCg-m} & 2190 \\
\texttt{sensory} & 16981 & &\texttt{visual} & 11469 & & \texttt{photo-receptor} & 9777 & & \texttt{T2a} & 1774 \\
\texttt{visual-projection} & 7665 & & \texttt{Kenyon-Cell} & 5177 & & \texttt{medulla-intrinsic} & 7859 & & \texttt{Tm3} & 1756 \\
\texttt{ascending} & 2362 & & \texttt{CX} & 2869 & & \texttt{lamina-monopolar} & 7584 & & \texttt{T4c} & 1710 \\
\texttt{descending} & 1303 & & \texttt{mechanosensory} & 2648 & & \texttt{distal-medulla} & 6538 & & \texttt{T3} & 1676 \\
\texttt{visual-centrifugal} & 521 & & \texttt{AN} & 2362 & & \texttt{t4-neuron} & 6243 & & \texttt{KCab} & 1643 \\
\texttt{motor} & 106 & & \texttt{olfactory} & 2281 & & \texttt{t5-neuron} & 6002 & & \texttt{L1} & 1596 \\
\texttt{endocrine} & 80 & & \texttt{ALPN} & 685 & & \texttt{transmedullary-y} & 5278 & & \texttt{L2} & 1593  \\
\textit{\it unlabeled} & 0 & & \texttt{LHLN} & 514 & & \texttt{t2 neuron} & 3240 & & \texttt{Mi1} & 1581 \\
\bottomrule
\end{tabular}
\end{center}
\caption{Top 10 labels per cell category in the connectome data set.}
\label{tab:labels}
\end{table}


\section{Experimental settings}
\label{app:eval}

In \cref{sec:approximation} we presented experimental results for subzero matrix completions of the sparse nonnegative matrices in \cref{tab:data}. Most of these models were trained for 1000 epochs with default parameters of $\eta\!=\!1$ for the step size, $\gamma\!=\!0.9$ for the momentum, and 100 mini-batches per epoch (i.e., a mini-batch size equal to $(1/100)^\text{th}$ of the number of rows or columns in the matrix $\mat{S}$). However, some of the models required smaller step sizes, fewer (but larger) mini-batches, and/or more iterations to converge. \Cref{tab:hyper} lists these hyperparameters for all of the experiments in \cref{sec:approximation}. From the bottom rows of the table, we see that smaller step sizes and larger mini-batches are required for the smallest and most heavily \textit{underparameterized} models to converge, presumably because in these case there is high variance in the leading subspaces identified across different stochastic updates. Likewise, we see that more iterations (4000 versus 1000) are generally required for convergence on the largest sparse matrix of bigram counts (with nearly 250K rows and columns).  This trend is also to be expected for larger problem sizes.

\begin{table}[t]
\begin{center}
\begin{tabular}{rccccccc}
\toprule
\textbf{dataset} & \hspace{7ex}\textbf{rank}\hspace{7ex} & \textbf{step-size} & 
\hspace{2ex}\textbf{iterations}\hspace{2ex} & \textbf{mini-batches} \\
\midrule
digits & $2^2,2^3,\ldots,2^8$ & 1 & 1000 & 100 \\
connectome & $2^4,2^5,\ldots,2^8$ & 1 & 1000 & 100 \\
bigrams & $2^4,2^5,\ldots,2^9$ & 0.1 & 4000 & 100 \\
connectome & $2^2$ & 1 & 1000 & 25\\
connectome & $2^3$ & 1 & 1000 & 50\\
bigrams & $2^2$ & 0.01 & 1000 & 50 \\
bigrams & $2^3$ & 0.1 & 4000 & 50 \\
\bottomrule
\end{tabular}
\end{center}
\caption{Hyperparameter settings for the models of subzero matrix completion in \cref{sec:approximation}.}
\label{tab:hyper}
\end{table}


\section{GPU optimizations}
\label{app:gpu}

In this appendix we provide more details on the GPU optimizations in \cref{sec:GPU}. We begin by deriving \cref{eq:AIratio} for the ratio of arithmetic intensities. Let $\mat{A}\!\in\!\mathbb{R}^{m\times r}$ and $\mat{B}\!\in\!\mathbb{R}^{n\times r}$ with $r\!<\!\min(m,n)$. It requires $mnr$ scalar multiplications and $mn(r\!-\!1)$ scalar additions to compute the matrix product $\mat{A}\mat{B}^\top$, and thus the arithmetic intensity to calculate this product is given by
\begin{equation}
    I_\text{product} = \frac{mn(2r\!-\!1)}{w(mr+nr+mn)}\frac{\mathrm{FLOP}}{\mathrm{Byte}},
\label{eq:matmul_arithmetic_intensity}
\end{equation}
where $w$ is the number of bytes per word (e.g., $w=4$ for 32-bit floating-point values). Note that the denominator of this expression scales with the size of the matrix product $\mat{A}\mat{B}^\top$ since $wmn$ bytes must be transferred from the GPU at the end of the calculation. The \textit{split-kernel} computation of \cref{eq:kernel} accordingly suffers from the $2wmn$-byte memory round-trip of this matrix product. Next we compute the arithmetic intensity $I_{\mathrm{fused}}$ of the fused kernel for calculating \cref{eq:kernel}. This calculation involves two matrix products, so that its arithmetic intensity is given by
\begin{equation}
    I_{\text{fused}} = \frac{mn(2r\!-\!1)+mr(2n\!-\!1)}{2w(mr+nr)}\frac{\mathrm{FLOP}}{\mathrm{Byte}}.
    \label{eq:fused_arithmetic_intensity}
\end{equation}
Note that the denominator in this expression scales with the sizes of the factors $\mat{A}$ and $\mat{B}$ but not with the size of their product. In particular, the fused kernel avoids the $2wmn$ bytes memory round-trip of the split-kernel computation. Taking the ratio of \cref{eq:fused_arithmetic_intensity,eq:matmul_arithmetic_intensity}, we find
\begin{equation}
\frac{I_\text{fused}}{I_\text{product}}\,
 =\, \frac{1}{2}\left[1 + \frac{r(2n\!-\!1)}{n(2r\!-\!1)}\right] \left[1 + \frac{mn}{(m\!+\!n)r}\right]\,
 >\, 1 + \frac{mn}{(m\!+\!n)r},
\end{equation}
where the inequality follows from the earlier assumption that $r\!<\!n$. In this way we recover the lower bound in \cref{eq:AIratio}.

Next we provide more details of our TF32 implementations for the kernel in \cref{eq:kernel}. We provide two implementations of this kernel: one is a warp-synchronous Tensor Core implementation that uses the \textbf{mma} instruction family introduced with the A100 generation of NVIDIA GPUs, and the other is a warp-group asynchronous Tensor Core version that uses the \textbf{wgmma} instruction family introduced with the H100 generation.
A CUDA warp contains 32 threads, and on Hopper GPUs, \textbf{wgmma} instructions are issued by a warpgroup, a cooperative group of four consecutive warps (i.e., 128 threads) within the same thread block. As shorthand in what follows, we use $\mat{C}$ to denote the pseudoinverse of $\mat{B}^\top$ in \cref{eq:kernel} and $\mat{Y}$ to denote the kernel output $\max(0,\mat{A}\mat{B}^\top)\mat{C}$.

Our TF32 kernels decompose $\mat{Y}$ into tiles. Let $i$ index a tile of $m_T$ output rows (where $m_T<m$), and let $j$ index a tile of $n_T$ rows (where $n_T<n$) from $\mat{B}$ and $\mat{C}$ along their shared dimension. We write these tiles as $\mat{A}_i\in\mathbb{R}^{m_T\times r}$, $\mat{Y}_i\in\mathbb{R}^{m_T\times r}$, and $\mat{B}_j,\mat{C}_j\in\mathbb{R}^{n_T\times r}$. Before the main compute kernel, a short packing kernel converts $\mat{B}$ and $\mat{C}$ from FP32 to TF32 and rearranges each tile into the operand layout expected by the Tensor Core instructions.
The main kernel then uses a one-dimensional thread-block grid over the output rows: each thread block owns one $\mat{A}_i$ tile and the corresponding output tile $\mat{Y}_i$. Since $\mat{A}_i$ is reused for every $j$ tile, it is loaded once at kernel startup,
either into registers in the \textbf{mma} kernel, or into registers or shared memory depending on the \textbf{wgmma} operand mode. The compute loop then streams over the $\mat{B}_j,\mat{C}_j$ tiles. For each tile $j$, the packed $\mat{B}_j$ and $\mat{C}_j$ data are staged from global
memory into shared memory using an asynchronous copy pipeline. The \textbf{mma} kernel uses multistage \textbf{cp.async} copies: while the Tensor Core instructions consume the current shared-memory tile,
later tiles are already in flight. The \textbf{wgmma} kernel uses a producer/consumer pattern and utilizes the specialized Tensor Memory Accelerator (TMA) units for global-to-shared memory copies: one warpgroup issues \textbf{cp.async.bulk} transfers into staged shared-memory buffers,
while the compute warpgroups consume ready stages with \textbf{wgmma} instructions. 
Within the compute phase, a series of \textbf{mma}/\textbf{wgmma} instructions performs the first contraction to compute the local intermediate $\mat{L_{ij}}=\mat{A}_i\mat{B}_j^{\!\top}$. Since TF32 Tensor Core instructions produce FP32 outputs, we then need to convert $\mat{S}_{ij}$ again to TF32 to feed into the second contraction that accumulates $\mat{Y}_i\mathrel{+}=\max(0,\mat{L_{ij}})\mat{C}_j$. In both cases, the intermediate $\mat{S}_{ij}$ tile is never written to global memory. The overlapping copy/compute pipelines for both kernels are sketched in Figure~\ref{fig:kernel-pipeline}.

In our Python implementation, we use autotuning to scan, select, and cache the tile sizes $m_T$ and~$n_T$, the number of rows handled per warp or warpgroup and the number of pipeline stages that maximize the kernel performance for the given input matrix shapes. For the \textbf{wgmma} kernel, the operand mode and the \textbf{wgmma} instruction shapes are also part of the autotuning space. (The TF32 \textbf{wgmma.mma\_async} PTX instruction supports varying shapes for the $m$, $n$, and $k$ dimensions of its three input matrix operands: $m=64$ and $k=8$ are fixed, but $n$ may vary between 8 and 256 in steps of 8.) We allow for different \textbf{wgmma} instruction shapes in the first and second contraction of the compute phase. 

This autotuning is combined with JIT compilation using NVIDIA's runtime compiler in order to avoid having to precompile a large number of parameter combinations ahead of time. In addition to the TF32 kernels, we also provide plain FP32 fused kernels.

\begin{figure}
\centering
\resizebox{0.98\textwidth}{!}{%
\begin{tikzpicture}[
    >=Latex,
    x=0.58cm,
    y=0.72cm,
    rowlabel/.style={font=\scriptsize, anchor=east, align=right},
    title/.style={font=\small\bfseries, anchor=west},
    stage/.style={draw, rounded corners=1pt},
    copy/.style={stage, fill=blue!10, draw=blue!55!black},
    wait/.style={stage, fill=black!6, draw=black!40},
    ready/.style={stage, fill=green!10, draw=green!45!black},
    first/.style={stage, fill=orange!18, draw=orange!75!black},
    relu/.style={stage, fill=red!8, draw=red!55!black},
    second/.style={stage, fill=purple!10, draw=purple!65!black},
    bartext/.style={font=\tiny, align=center},
]
    \node[title] at (0.0,6.6) {\textbf{mma} kernel:};

    \node[rowlabel] at (1.9,5.55) {cp.async};
    \node[rowlabel] at (1.9,4.57) {warp, tile $i$};

    \draw[copy] (2.2,5.28) rectangle (7.8,5.82);
    \node[bartext] at (4.6,5.55) {Copy $j+1$ in flight};
    \draw[copy] (12.65,5.28) rectangle (19.25,5.82);
    \node[bartext] at (14.75,5.55) {Copy $j+2$ in flight};
    \node[bartext] at (19.85,5.55) {$\cdots$};

    \draw[wait] (2.2,4.30) rectangle (3.8,4.84);
    \node[bartext] at (2.95,4.57) {Wait $j$};
    \draw[first] (3.9,4.30) rectangle (6.05,4.84);
    \node[bartext] at (5,4.57) {MMA $L_{ij}$};
    \draw[relu] (6.2,4.30) rectangle (7.45,4.84);
    \node[bartext] at (6.825,4.57) {ReLU};
    \draw[second] (7.6,4.30) rectangle (9.75,4.84);
    \node[bartext] at (8.65,4.57) {MMA $Y_i$};
    \draw[copy] (9.95,4.30) rectangle (12.5,4.84);
    \node[bartext] at (11.2,4.57) {Issue $j+2$};

    \draw[wait] (12.65,4.30) rectangle (14.75,4.84);
    \node[bartext] at (13.7,4.57) {Wait $j+1$};
    \draw[first] (14.9,4.30) rectangle (17.75,4.84);
    \node[bartext] at (16.25,4.57) {MMA $L_{i,j+1}$};
    \draw[relu] (17.8,4.30) rectangle (19.25,4.84);
    \node[bartext] at (18.5,4.57) {ReLU};
    \node[bartext] at (19.85,4.57) {$\cdots$};
    \draw[->, thin, draw=black!65] (4.6,5.28) -- (13.7,4.84);
    \draw[->, draw=black!65] (2.1,3.95) -- (20.9,3.95)
        node[anchor=west, font=\scriptsize] {time};

    \node[title] at (0.0,3.25) {\textbf{wgmma} kernel:};

    \node[rowlabel] at (1.9,2.15) {producer WG};
    \node[rowlabel] at (1.9,1.17) {consumer WG 0, tile $i$};
    \node[rowlabel] at (1.9,0.37) {consumer WG 1, tile $i+1$};

    \draw[copy] (2.3,1.88) rectangle (4.4,2.42);
    \node[bartext] at (3.35,2.15) {Copy $j+1$};
    \draw[wait] (4.5,1.88) rectangle (11,2.42);
    \node[bartext] at (7.725,2.15) {Wait Stage $j$ completion};
    \draw[copy] (11.15,1.88) rectangle (13.55,2.42);
    \node[bartext] at (12.35,2.15) {Copy $j+2$};
    \draw[wait] (13.7,1.88) rectangle (19.55,2.42);
    \node[bartext] at (16.625,2.15) {Wait Stage $j+1$ completion};

    \draw[ready] (2.35,0.90) rectangle (3.95,1.44);
    \node[bartext] at (3.15,1.17) {Ready $j$};
    \draw[first] (4.1,0.90) rectangle (6.45,1.44);
    \node[bartext] at (5.275,1.17) {MMA $L_{ij}$};
    \draw[relu] (6.6,0.90) rectangle (8.0,1.44);
    \node[bartext] at (7.3,1.17) {ReLU};
    \draw[second] (8.15,0.90) rectangle (10.1,1.44);
    \node[bartext] at (9.125,1.17) {MMA $Y_i$};
    \draw[ready] (11.1,0.90) rectangle (13.55,1.44);
    \node[bartext] at (12.365,1.17) {Ready $j+1$};
    \draw[first] (13.65,0.90) rectangle (16.6,1.44);
    \node[bartext] at (15.05,1.17) {MMA $L_{i,j+1}$};
    \node[bartext] at (17.3,1.17) {$\cdots$};

    \draw[densely dashed, draw=black!55] (11.065,-0.05) -- (11.065,2.55)
        node[pos=1, above, font=\scriptsize, text=black!65] {stage $j$ done};
    \draw[->, thin, draw=black!65] (3.35,1.88) -- (12.475,1.44);
    \draw[->, thin, dashed, draw=black!65] (12.35,1.88) -- (20.25,1.35);

    \draw[ready] (2.7,0.10) rectangle (4.3,0.64);
    \node[bartext] at (3.5,0.37) {Ready $j$};
    \draw[first] (4.45,0.10) rectangle (6.95,0.64);
    \node[bartext] at (5.7,0.37) {MMA $L_{i+1,j}$};
    \draw[relu] (7.1,0.10) rectangle (8.5,0.64);
    \node[bartext] at (7.8,0.37) {ReLU};
    \draw[second] (8.65,0.10) rectangle (11,0.64);
    \node[bartext] at (9.8,0.37) {MMA $Y_{i+1}$};
    \draw[ready] (11.45,0.10) rectangle (14.0,0.64);
    \node[bartext] at (12.725,0.37) {Ready $j+1$};
    \draw[first] (14.15,0.10) rectangle (17.6,0.64);
    \node[bartext] at (15.875,0.37) {MMA $L_{i+1,j+1}$};
    \node[bartext] at (18.1,0.37) {$\cdots$};
    \draw[->, draw=black!65] (2.1,-0.3) -- (20.9,-0.3)
        node[anchor=west, font=\scriptsize] {time};
\end{tikzpicture}%
}
\caption{Schematic asynchronous copy/compute pipeline overlap for the fused TF32 kernels with two staging buffers. Box widths and offsets are schematic and do not represent actual timings. For each tile, the first MMA block forms $\mat{L_{ij}}$, the ReLU block rectifies $\mat{L_{ij}}$ and converts it back to TF32, and the second MMA block accumulates $\mat{Y_i}$. In the WGMMA kernel, the dashed vertical line marks completion of the stage holding tile $j$, after which the producer reuses that buffer for tile $j\!+\!2$.}
\label{fig:kernel-pipeline}
\end{figure}
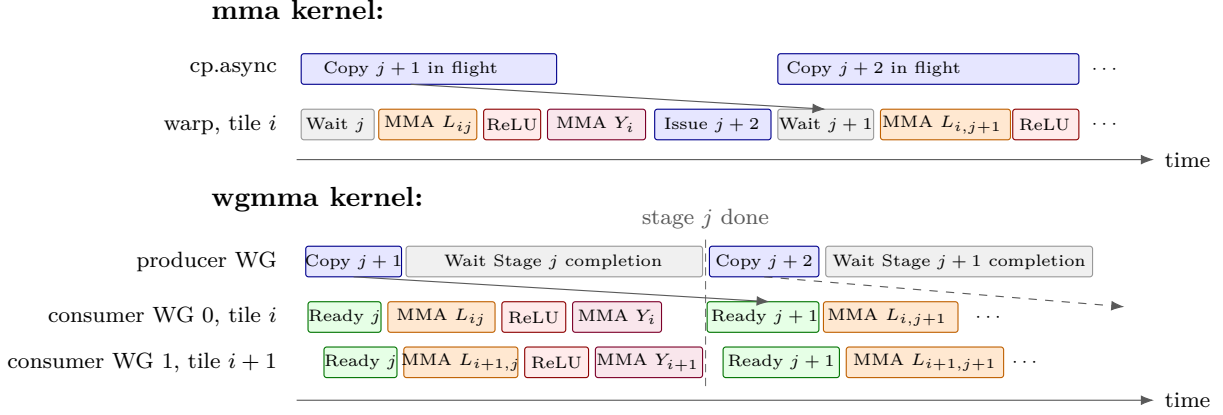



\begin{thebibliography}{100}

\bibitem{newman2003-networks}
M.~E.~J. Newman.
\newblock The structure and function of complex networks.
\newblock {\em SIAM Review}, 45:167--256, 2003.

\bibitem{michel2011-ngrams}
J.-B. Michel, Y.~K. Shen, A.~P. Aiden, A.~Veres, M.~K. Gray, J.~P. Pickett, D.~Hoiberg, D.~Clancy, P.~Norvig, J.~Orwant, S.~Pinker, M.~A. Nowak, and E.~L. Aiden.
\newblock Quantitative analysis of culture using millions of digitized books.
\newblock {\em Science}, 331(6014):176--182, 2011.

\bibitem{olshausen1996-emergence}
B.~A. Olshausen and D.~J. Field.
\newblock Emergence of simple-cell receptive field properties by learning a sparse code for natural images.
\newblock {\em Nature}, 381:607--609, 1996.

\bibitem{papadopoulos2004-nearest}
A.~N. Papadopoulos and Y.~Manolopoulos.
\newblock {\em Nearest Neighbor Search: A Database Perspective}.
\newblock Springer, New York, NY, 2004.

\bibitem{udell2019-rank}
M.~Udell and A.~Townsend.
\newblock Why are big data matrices approximately low rank?
\newblock {\em SIAM Journal on Mathematics of Data Science}, 1(1):144--160, 2019.

\bibitem{budzinskiy2025-rank}
S.~Budzinskiy.
\newblock When big data actually are low-rank, or entrywise approximation of certain function-generated matrices.
\newblock {\em {SIAM} Journal on Mathematics of Data Science}, 7(3):1098--1122, 2025.

\bibitem{deerwester1990-indexing}
S.~Deerwester, S.~T. Dumais, G.~W. Furnas, T.~K. Landauer, and R.~Harshman.
\newblock Indexing by latent semantic analysis.
\newblock {\em Journal of the Association for Information Science and Technology}, 41(6):391--407, 1990.

\bibitem{brunet2004-gene}
J.-P. Brunet, P.~Tamayo, T.~R. Golub, and J.~P. Mesirov.
\newblock Metagenes and molecular pattern discovery using matrix factorization.
\newblock {\em Proceedings of the National Academy of Sciences U.S.A.}, 101(12):4164--4169, 2004.

\bibitem{koren2009-matrix}
Y.~Koren, R.~Bell, and C.~Volinsky.
\newblock Matrix factorization techniques for recommender systems.
\newblock {\em Computer}, 42(8):30--37, 2009.

\bibitem{hu2022-lora}
E.~J. Hu, Y.~Shen, P.~Wallis, Z.~Allen-Zhu, Y.~Li, S.~Wang, and W.~Chen.
\newblock {LoRA}: Low-rank adaptation of large language models.
\newblock In {\em Proceedings of the 10th International Conference on Learning Representations (ICLR-2022)}, 2022.

\bibitem{turk1991-eigenfaces}
M.~Turk and A.~Pentland.
\newblock Eigenfaces for recognition.
\newblock {\em Journal of Cognitive Neuroscience}, 3:71--86, 1991.

\bibitem{jolliffe2016-pca}
I.~T. Jolliffe and J.~Cadima.
\newblock Principal component analysis: a review and recent developments.
\newblock {\em Philosophical Transactions of the Royal Society A: Mathematical, Physical and Engineering Sciences}, 374(2065):20150202, 04 2016.

\bibitem{saul2022-nmd}
L.~K. Saul.
\newblock A nonlinear matrix decomposition for mining the zeros of sparse data.
\newblock {\em SIAM Journal on Mathematics of Data Science}, 4(2):431--463, 2022.

\bibitem{saul2022-geometrical}
L.~K. Saul.
\newblock A geometrical connection between sparse and low-rank matrices and its application to manifold learning.
\newblock {\em Transactions on Machine Learning Research}, December 2022.

\bibitem{alon2016-sign}
N.~Alon, S.~Moran, and A.~Yehudayoff.
\newblock Sign rank versus {VC} dimension.
\newblock In {\em Proceedings of the 29th Annual Conference on Learning Theory}, pages 47--80, 2016.

\bibitem{seraghiti2023-accelerated}
G.~Seraghiti, A.~Awari, A.~Vandaele, M.~Porcelli, and N.~Gillis.
\newblock Accelerated algorithms for nonlinear matrix decomposition with the {ReLU} function.
\newblock In {\em Proceedings of the 2023 IEEE 33rd International Workshop on Machine Learning for Signal Processing (MLSP)}, pages 1--6. IEEE, 2023.

\bibitem{liu2024-symmetric}
H.~Liu, P.~Wang, L.~Huang, Q.~Qu, and L.~Balzano.
\newblock Symmetric matrix completion with {ReLU} sampling.
\newblock In {\em Proceedings of the 41st International Conference on Machine Learning}, pages 32015--32040, 2024.

\bibitem{wang2024-momentum}
Q.~Wang, C.~Cui, and D.~Han.
\newblock A momentum accelerated algorithm for {ReLU}-based nonlinear matrix decomposition.
\newblock {\em IEEE Signal Processing Letters}, 31:2865--2869, 2024.

\bibitem{awari2025-alternating}
A.~Awari, N.~Gillis, and A.~Vandaele.
\newblock Alternating direction method of multipliers for nonlinear matrix decompositions.
\newblock {\em arXiv preprint arXiv:2512.17473}, 2025.

\bibitem{gillis2025-extrapolated}
N.~Gillis, M.~Porcelli, and G.~Seraghiti.
\newblock An extrapolated and provably convergent algorithm for nonlinear matrix decomposition with the {ReLU} function.
\newblock {\em arXiv preprint arXiv:2503.23832}, 2025.

\bibitem{wang2025-efficient}
Q.~Wang.
\newblock An efficient alternating algorithm for {ReLU}-based symmetric matrix decomposition.
\newblock {\em arXiv preprint arXiv:2503.16846}, 2025.

\bibitem{wang2025-accelerated}
Q.~Wang, Y.~Qu, C.~Cui, and D.~Han.
\newblock An accelerated alternating partial {B}regman algorithm for {ReLU}-based matrix decomposition.
\newblock {\em Journal of Scientific Computing}, 105(48), 2025.

\bibitem{sumac2026-github}
\url{https://github.com/flatironinstitute/sumac}, 2026.

\bibitem{dorkenwald2024-flywire}
S.~Dorkenwald, A.~Matsliah, A.~R. Sterling, P.~Schlegel, S.-C. Yu, C.~E. McKellar, A.~Lin, M.~Costa, K.~Eichler, Y.~Yin, W.~Silversmith, C.~Schneider-Mizell, C.~S. Jordan, D.~Brittain, A.~Halageri, K.~Kuehner, O.~Ogedengbe, R.~Morey, J.~Gager, K.~Kruk, E.~Perlman, R.~Yang, D.~Deutsch, D.~Bland, M.~Sorek, R.~Lu, T.~Macrina, K.~Lee, J.~A. Bae, S.~Mu, B.~Nehoran, E.~Mitchell, S.~Popovych, J.~Wu, Z.~Jia, M.~A. Castro, N.~Kemnitz, D.~Ih, A.~S. Bates, N.~Eckstein, J.~Funke, F.~Collman, D.~D. Bock, G.~S. X.~E. Jefferis, H.~S. Seung, M.~Murthy, and The~FlyWire Consortium.
\newblock Neuronal wiring diagram of an adult brain.
\newblock {\em Nature}, 634:124--138, 2024.

\bibitem{matsliah2023-codex}
A.~Matsliah, A.~R. Sterling, S.~Dorkenwald, K.~Kuehner, R.~Morey, H.~S. Seung, and M.~Murthy.
\newblock Codex: Connectome data explorer.
\newblock \url{https://www.researchgate.net/publication/372464040\_Codex\_Connectome\_Data\_Explorer}, 2023.

\bibitem{candes2009-exact}
E.~Candes and B.~Recht.
\newblock Exact matrix completion via convex optimization.
\newblock {\em Foundations of Computational Mathematics}, 9:717--772, 2009.

\bibitem{candes2010-matrix}
E.~Candes and Y.~Plan.
\newblock Matrix completion with noise.
\newblock {\em Proceedings of the IEEE}, 98(6):925--936, 2010.

\bibitem{keshavan2010-matrix}
R.~H. Keshavan, A.~Montanari, and S.~Oh.
\newblock Matrix completion from a few entries.
\newblock {\em {IEEE} Transactions on Information Theory}, 56(6):2980--2998, 2010.

\bibitem{chatterjee2015-matrix}
S.~Chatterjee.
\newblock Matrix estimation by universal singular value thresholding.
\newblock {\em Annals of Statistics}, 43(1):177--214, 2015.

\bibitem{sun2016-guaranteed}
R.~Sun and Z.-Q. Luo.
\newblock Guaranteed matrix completion via non-convex factorization.
\newblock {\em IEEE Transactions on Information Theory}, 62(11):6535--6579, 2016.

\bibitem{nguyen2019-low}
L.~T. Nguyen, J.~Kim, and B.~Shim.
\newblock Low-rank matrix completion: a contemporary survey.
\newblock {\em IEEE Access}, 7:94215--94237, 2019.

\bibitem{chatterjee2020-deterministic}
S.~Chatterjee.
\newblock A deterministic theory of low rank matrix completion.
\newblock {\em IEEE Transactions on Information Theory}, 66(12):8046--8055, 2020.

\bibitem{zilber2022-GNMR}
P.~Zilber and B.~Nadler.
\newblock {GNMR}: a provable one-line algorithm for low rank matrix recovery.
\newblock {\em SIAM Journal on Mathematics of Data Science}, 4(2):909--934, 2022.

\bibitem{chung1997-spectral}
F.~R.~K. Chung.
\newblock {\em Spectral Graph Theory}.
\newblock Number~92 in CBMS Regional Conference Series in Mathematics. American Mathematical Society, Providence, RI, 1997.

\bibitem{cvetkovic2007-signless}
D.~Cvetkovi\'{c}, P.~Rowlinson, and S.~K. Simi\'{c}.
\newblock Signless {L}aplacians of finite graphs.
\newblock {\em Linear Algebra and its Applications}, 423(1):155--171, 2007.

\bibitem{lee1999-learning}
D.~D. Lee and H.~Sompolinsky.
\newblock Learning a continuous hidden variable model for binary data.
\newblock In M.~J. Kearns, S.~A. Solla, and D.~A. Cohn, editors, {\em Advances in Neural Information Processing Systems 11}, pages 515--521. MIT Press, 1999.

\bibitem{droge2023-kissing}
H.~Dr\"{o}ge, Z.~L\"{a}hner, Y.~Bahat, O.~Martorell~Nadal, F.~Heide, and M.~Moeller.
\newblock Kissing to find a match: Efficient low-rank permutation representation.
\newblock In A.~Oh, T.~Naumann, A.~Globerson, K.~Saenko, M.~Hardt, and S.~Levine, editors, {\em Advances in Neural Information Processing Systems 36}, pages 48459--48471. Curran Associates, Inc., 2023.

\bibitem{hardy1975-intro}
G.~H. Hardy and E.~M. Wright.
\newblock {\em An Introduction to the Theory of Numbers}.
\newblock Oxford University Press, 4th edition, 1975.

\bibitem{cramer1936-order}
H.~Cram\'{e}.
\newblock On the order of magnitude of the difference between consecutive prime numbers.
\newblock {\em Acta Arithmetica}, 2(1)(1):23--46, 1936.

\bibitem{crandall2005-prime}
R.~Crandall and C.~Pomerance.
\newblock {\em Prime Numbers: A Computational Perspective}.
\newblock Springer, 2nd edition, 2005.

\bibitem{moenck1973-fast}
R.~T. Moenck.
\newblock Fast computation of {GCDs}.
\newblock In {\em Proceedings of the 5th Annual ACM Symposium on Theory of Computing (STOC-1973)}, page 142–151, 1973.

\bibitem{lecun98-gradient}
Y.~LeCun, L.~Bottou, Y.~Bengio, and P.~Haffner.
\newblock Gradient-based learning applied to document recognition.
\newblock {\em Proceedings of the {IEEE}}, 86(11):2278--2324, 1998.

\bibitem{TC-wikipedia}
J.~Artiles and S.~Sekine.
\newblock Tagged and cleaned {W}ikipedia and its {N}gram.
\newblock \url{https://nlp.cs.nyu.edu/wikipedia-data/}.
\newblock Accessed: September 5, 2025.

\bibitem{peeters2003-maximum}
R.~Peeters.
\newblock The maximum edge biclique problem is {NP}-complete.
\newblock {\em Discrete Applied Mathematics}, 131(3)(3):651--654, 2003.

\bibitem{awari2024-coordinate}
A.~Awari, H.~Nguyen, S.~Wertz, A.~Vandaele, and N.~Gillis.
\newblock Coordinate descent algorithm for nonlinear matrix decomposition with the {ReLU} function.
\newblock In {\em Proceedings of the 2024 32nd European Signal Processing Conference (EUSIPCO)}, pages 2622--2626. IEEE, 2024.

\bibitem{dempster1977-maximum}
A.~P. Dempster, N.~M. Laird, and D.~B. Rubin.
\newblock Maximum likelihood estimation from incomplete data via the {EM} algorithm (with discussion).
\newblock {\em Journal of the Royal Statistical Society B}, 39:1--38, 1977.

\bibitem{hastie2015-matrix}
T.~Hastie, R.~Mazumder, J.~D. Lee, and R.~Zadeh.
\newblock Matrix completion and low-rank {SVD} via fast alternating least squares.
\newblock {\em Journal of Machine Learning Research}, 16(104):3367--3402, 2015.

\bibitem{zhou2008-large}
Y.~Zhou, D.~Wilkinson, R.~Schreiber, and R.~Pan.
\newblock Large-scale parallel collaborative filtering for the {Netflix} prize.
\newblock In R.~Fleischer and J.~Xu, editors, {\em Algorithmic Aspects in Information and Management}, pages 337--348. Springer, 2008.

\bibitem{yu2014-parallel}
H.-F. Yu, C.-J. Hsieh, S.~Si, and I.~S. Dhillon.
\newblock Parallel matrix factorization for recommender systems.
\newblock {\em Knowledge and Information Systems (KAIS)}, 41(3):793--–819, 2014.

\bibitem{bottou2004-large}
L.~Bottou and Y.~LeCun.
\newblock Large scale online learning.
\newblock In S.~Thrun, L.~K. Saul, and B.~Sch\"{o}lkopf, editors, {\em Advances in Neural Information Processing Systems 16}, pages 25--32. MIT Press, 2004.

\bibitem{polyak1964-momentum}
B.~T. Polyak.
\newblock Some methods of speeding up the convergence of iteration methods.
\newblock {\em USSR Computational Mathematics and Mathematical Physics}, 4(5):1--17, 1964.

\bibitem{dao2022-flash1}
T.~Dao, D.~Y. Fu, S.~Ermon, A.~Rudra, and C.~R\'{e}.
\newblock {FlashAttention}: fast and memory-efficient exact attention with {IO}-awareness.
\newblock In S.~Koyejo, S.~Mohamed, A.~Agarwal, D.~Belgrave, K.~Cho, and A.~Oh, editors, {\em Advances in Neural Information Processing Systems 35}, pages 16344--16359. Curran Associates Inc., 2022.

\bibitem{dao2024-flash2}
T.~Dao.
\newblock {FlashAttention}-2: Faster attention with better parallelism and work partitioning.
\newblock In {\em Proceedings of the 12th International Conference on Learning Representations (ICLR-2024)}, 2024.

\bibitem{shah2024-flash3}
J.~Shah, G.~Bikshandi, Y.~Zhang, V.~Thakkar, P.~Ramani, and T.~Dao.
\newblock {FlashAttention}-3: Fast and accurate attention with asynchrony and low-precision.
\newblock In A.~Globerson, L.~Mackey, D.~Belgrave, A.~Fan, U.~Paquet, J.~Tomczak, and C.~Zhang, editors, {\em Advances in Neural Information Processing Systems 37}, pages 68658--68685. Curran Associates Inc., 2024.

\bibitem{zadouri2026-flash4}
T.~Zadouri, M.~Hoehnerbach, J.~Shah, T.~Liu, V.~Thakkar, and T.~Dao.
\newblock Flashattention-4: algorithm and kernel pipelining co-design for asymmetric hardware scaling.
\newblock {\em arXiv preprint arXiv:2603.05451}, 2026.

\bibitem{stosic2021-accelerating}
D.~Stosic and P.~Micikevicius.
\newblock \url{https://developer.nvidia.com/blog/accelerating-ai-training-with-tf32-tensor-cores/}, 2021.

\bibitem{bentley1975-multidimensional}
J.~L. Bentley.
\newblock Multidimensional binary search trees used for associative searching.
\newblock {\em Communications of the ACM}, 18(9):509--517, 1975.

\bibitem{agarwal2017-range}
P.~K. Agarwal.
\newblock Range searching.
\newblock In J.~E. Goodman, J.~O'Rourke, and C.~D. T{\'{o}}th, editors, {\em Handbook of Discrete and Computational Geometry}, chapter~40, pages 1057--1092. Chapman and Hall/CRC, 2017.

\bibitem{kingma2015-adam}
D.~P. Kingma and J.~Ba.
\newblock Adam: A method for stochastic optimization.
\newblock In {\em Proceedings of the 3rd International Conference on Learning Representations (ICLR-2015)}, 2015.

\bibitem{mazumdar2019-learning}
A.~Mazumdar and A.~S. Rawat.
\newblock Learning and recovery in the {ReLU} model.
\newblock In {\em Proceedings of the 57th Annual Allerton Conference on Communication, Control, and Computing}, pages 108--115, 2019.

\bibitem{schlegel2024-whole}
P.~Schlegel, Y.~Yin, A.~S. Bates, S.~Dorkenwald, K.~Eichler, P.~Brooks, D.~S. Han, M.~Gkantia, M.~dos Santos, E.~J. Munnelly, G.~Badalamente, L.~S. Capdevila, V.~A. Sane, A.~M.~C. Fragniere, L.~Kiassat, M.~W. Pleijzier, T.~Stürner, I.~F.~M. Tamimi, C.~R. Dunne, I.~Salgarella, A.~Javier, S.~Fang, E.~Perlman, T.~Kazimiers, S.~R. Jagannathan, A.~Matsliah, A.~R. Sterling, S.-C. Yu, C.~E. McKellar, FlyWire Consortium, M.~Costa, H.~S. Seung, M.~Murthy, V.~Hartenstein, D.~D. Bock, and G.~S. X.~E. Jefferis.
\newblock Whole-brain annotation and multi-connectome cell typing of \textit{{D}rosophila}.
\newblock {\em Nature}, 634:139--152, 2024.

\bibitem{schwartzman2026-NTAC}
G.~Schwartzman, B.~Jourdan, D.~García-Soriano, and A.~Matsliah.
\newblock {NTAC}: Neuronal type assignment from connectivity.
\newblock {\em Nature Communications}, 17(1284):1--12, 2026.

\bibitem{markovsky2012-lowrank}
I.~Markovsky.
\newblock {\em Low Rank Approximation: Algorithms, Implementation, Applications (Communications and Control Engineering)}.
\newblock Springer, 2012.

\bibitem{wright2021-hdda}
J.~Wright and Y.~Ma.
\newblock {\em High-Dimensional Data Analysis with Low-Dimensional Models: Principles, Computation, and Applications}.
\newblock Cambridge University Press, 2021.

\bibitem{eckart1936-svd}
C.~Eckart and G.~Young.
\newblock The approximation of one matrix by another of lower rank.
\newblock {\em Psychometrika}, 1:211--218, 1936.

\bibitem{collins2002-generalization}
M.~Collins, S.~Dasgupta, and R.~E. Schapire.
\newblock A generalization of principal components analysis to the exponential family.
\newblock In T.~G. Dietterich, S.~Becker, and Z.~Ghahramani, editors, {\em Advances in Neural Information Processing Systems 14}, pages 617--624. MIT Press, 2002.

\bibitem{gordon2003-GLM}
G.~J. Gordon.
\newblock Generalized${^2}$ linear${^2}$ models.
\newblock In S.~Becker, S.~Thrun, and K.~Obermayer, editors, {\em Advances in Neural Information Processing Systems 15}, pages 593--600. MIT Press, 2003.

\bibitem{singh2008-unified}
A.~P. Singh and G.~J. Gordon.
\newblock A unified view of matrix factorization models.
\newblock In {\em Proceedings of the European Conference on Machine Learning and Knowledge Discovery in Databases (ECML/PKDD-08)}, pages \mbox{358--373}, 2008.

\bibitem{bartholomew2011-latent}
D.~J. Bartholomew, M.~Knott, and I.~Moustaki.
\newblock {\em Latent Variable Models and Factor Analysis: A Unified Approach}.
\newblock Wiley, 2011.

\bibitem{udell2016-generalized}
M.~Udell, C.~Horn, R.~Zadeh, and S.~Boyd.
\newblock Generalized low rank models.
\newblock {\em Foundation and Trends in Machine Learning}, 9(1):1--118, 2016.

\bibitem{hoff2021-additive}
P.~D. Hoff.
\newblock Additive and multiplicative effects network models.
\newblock {\em Statistical Science}, 36(1):34--50, 2021.

\bibitem{meeds2007-modeling}
E.~Meeds, Z.~Ghahramani, R.~M. Neal, and S.~T. Roweis.
\newblock Modeling dyadic data with binary latent factors.
\newblock In B.~Sch\"{o}lkopf, J.~Platt, and T.~Hofmann, editors, {\em Advances in Neural Information Processing Systems 19}, pages 977--984. MIT Press, 2007.

\bibitem{kolda1998-semidiscrete}
T.~G. Kolda and D.~P. O'Leary.
\newblock A semidiscrete matrix decomposition for latent semantic indexing in information retrieval.
\newblock {\em ACM Transactions on Information Systems}, 16(4):322--346, 1998.

\bibitem{lee1996-unsupervised}
D.~D. Lee and H.~S. Seung.
\newblock Unsupervised learning by convex and conic coding.
\newblock In M.~C. Mozer, M.~Jordan, and T.~Petsche, editors, {\em Advances in Neural Information Processing Systems 9}, pages 515--521. MIT Press, 1996.

\bibitem{ding2020-convex}
C.~H.~Q. Ding, T.~Li, and M.~I. Jordan.
\newblock Convex and semi-nonnegative matrix factorizations.
\newblock {\em IEEE Transactions on Pattern Analysis and Machine Intelligence}, 32(1):45--55, 2010.

\bibitem{thanh2023-bounded}
O.~V. Thanh, N.~Gillis, and F.~Lecron.
\newblock Bounded simplex-structured matrix factorization: algorithms, identifiability, and applications.
\newblock {\em IEEE Transactions on Signal Processing}, 71:2434--2447, 2023.

\bibitem{paatero1994-positive}
P.~Paatero and U.~Tapper.
\newblock Positive matrix factorization: A non-negative factor model with optimal utilization of error estimates of data values.
\newblock {\em Environmetrics}, 5:111--126, 1994.

\bibitem{lee1999-nmf}
D.~D. Lee and H.~S. Seung.
\newblock Learning the parts of objects by non-negative matrix factorization.
\newblock {\em Nature}, 401:788--791, 1999.

\bibitem{gillis2021-nmf}
N.~Gillis.
\newblock {\em Nonnegative matrix factorization}.
\newblock SIAM, 2021.

\bibitem{zou2006-sparse}
H.~Zou, T.~Hastie, and R.~Tibshirani.
\newblock Sparse principal component analysis.
\newblock {\em Journal of Computational and Graphical Statistics}, 15(2):265--286, 2006.

\bibitem{daspremont2008-optimal}
A.~D'Aspremont, F.~Bach, and L.~El~Ghaoui.
\newblock Optimal solutions for sparse principal component analysis.
\newblock {\em Journal of Machine Learning Research}, 9:1269--1294, 2008.

\bibitem{witten2009-pmd}
D.~M. Witten, R.~Tibshirani, and T.~Hastie.
\newblock A penalized matrix decomposition, with applications to sparse principal components and canonical correlation analysis.
\newblock {\em Biostatistics}, 10(3):515--534, 2009.

\bibitem{soni2016-noisy}
A.~Soni, S.~Jain, J.~Haupt, and S.~Gonella.
\newblock Noisy matrix completion under sparse factor models.
\newblock {\em IEEE Transactions on Information Theory}, 62(6):3636--3661, 2016.

\bibitem{saul1997-aggregate}
L.~Saul and F.~Pereira.
\newblock Aggregate and mixed-order {M}arkov models for statistical language processing.
\newblock In {\em Proceedings of the 2nd Conference on Empirical Methods in Natural Language Processing (EMNLP-97)}, pages 81--89, 1997.

\bibitem{hofmann1999-probabilistic}
T.~Hofmann.
\newblock Probabilistic latent semantic analysis.
\newblock In {\em Proceedings of the 15th Conference on Uncertainty in Artificial Intelligence (UAI-99)}, pages 289--296, 1999.

\bibitem{srebro2005-maximum}
N.~Srebro, J.~Rennie, and T.~S. Jaakkola.
\newblock Maximum-margin matrix factorization.
\newblock In L.~K. Saul, Y.~Weiss, and L.~Bottou, editors, {\em Advances in Neural Information Processing Systems 17}, pages 1329--1336. MIT Press, 2005.

\bibitem{tipping1999-probabilistic}
M.~E. Tipping.
\newblock Probabilistic visualisation of high-dimensional binary data.
\newblock In M.~J. Kearns, S.~A. Solla, and D.~A. Cohn, editors, {\em Advances in Neural Information Processing Systems 11}, pages 592--598. MIT Press, 1989.

\bibitem{hoff2002-latent}
P.D. Hoff, A.E. Raftery, and M.S. Handcock.
\newblock Latent space approaches to social network analysis.
\newblock {\em Journal of the American Statistical Association}, 97(460):1090--1098, 2002.

\bibitem{schein2003-generalized}
A.~I. Schein, L.~K. Saul, and L.~H. Ungar.
\newblock A generalized linear model for principal component analysis of binary data.
\newblock In {\em Proceedings of the 9th Workshop on Artificial Intelligence and Statistics (AISTATS-03)}, pages 14--21, 2003.

\bibitem{davenport2014-1bit}
M.~A. Davenport, Y.~Plan, E.~van~den Berg, and M.~Wooters.
\newblock 1-bit matrix completion.
\newblock {\em Information and Inference: A Journal of the {IMA}}, 3(3):189--223, 2014.

\bibitem{bhaskar2015-1bit}
S.~A. Bhaskar and A.~Javanmard.
\newblock 1-bit matrix completion under exact low-rank constraint.
\newblock In {\em Proceedings of the 49th Annual Conference on Information Sciences and Systems (CISS-15)}, pages 1--6, 2015.

\bibitem{cao2013-categorical}
Y.~Cao and Y.~Xie.
\newblock Categorical matrix completion.
\newblock In {\em Proceedings of the 6th IEEE International Workshop on Computational Advances in Multi-Sensor Adaptive Processing (CAMSAP-15)}, pages 369--372, 2013.

\bibitem{lafond2014-probabilistic}
J.~Lafond, O.~Klopp, E.~Moulines, and J.~Salmon.
\newblock Probabilistic low-rank matrix completion on finite alphabets.
\newblock In Z.~Ghahramani, M.~Welling, C.~Cortes, N.~Lawrence, and K.~Q. Weinberger, editors, {\em Advances in Neural Information Processing Systems 27}, pages 1727--1735. Curran Associates, Inc., 2014.

\bibitem{bhaskar2016-probabilistic}
S.~A. Bhaskar.
\newblock Probabilistic low-rank matrix completion from quantized measurements.
\newblock {\em Journal of Machine Learning Research}, 17:1--34, 2016.

\bibitem{kannan2014-bounded}
R.~Kannan, M.~Ishteva, and H.~Park.
\newblock Bounded matrix factorization for reocommender systems.
\newblock {\em Knowledge and Information Systems}, 39(3):491--511, 2014.

\bibitem{gopalan2014-poisson}
P.~Gopalan, L.~Charlin, and D.~Blei.
\newblock Content-based recommendations with {P}oisson factorization.
\newblock In Z.~Ghahramani, M.~Welling, C.~Cortes, N.~Lawrence, and K.~Q. Weinberger, editors, {\em Advances in Neural Information Processing Systems 27}, pages 3176--3184. Curran Associates, Inc., 2014.

\bibitem{song20-nonnegative}
G.-J. Song and M.~K. Ng.
\newblock Nonnegative low rank matrix approximation for nonnegative matrices.
\newblock {\em Applied Mathematics Letters}, 105:106300, 2020.

\bibitem{bishop2014-deterministic}
W.~E. Bishop and B.~M. Yu.
\newblock Deterministic symmetric positive semidefinite matrix completion.
\newblock In Z.~Ghahramani, M.~Welling, C.~Cortes, N.~D. Lawrence, and K.~Q. Weinberger, editors, {\em Advances in Neural Information Processing Systems 27}, pages 2762--2770. Curran Associates, Inc., 2014.

\bibitem{hoff2005-bilinear}
P.~D. Hoff.
\newblock Bilinear mixed-effects models for dyadic data.
\newblock {\em Journal of the American Statistical Association}, 100:286--295, 2005.

\bibitem{pmf-mnih08}
A.~Mnih and R.~R. Salakhutdinov.
\newblock Probabilistic matrix factorization.
\newblock In J.~C. Platt, D.~Koller, Y.~Singer, and S.~T. Roweis, editors, {\em Advances in Neural Information Processing Systems 20}, pages 1257--1264. Curran Associates, Inc., 2008.

\bibitem{handcock2007-model}
M.~S. Handcock, A.~E. Raftery, and J.~M. Tantrum.
\newblock Model-based clustering for social networks.
\newblock {\em Journal of the Royal Statistical Society, Series A}, 170:301–354, 2007.

\bibitem{ma2020-universal}
Z.~Ma, Z.~Ma, and H.~Yuan.
\newblock Universal latent space model fitting for large networks with edge covariates.
\newblock {\em Journal of Machine Learning Research}, 21(4):1--67, 2020.

\bibitem{ganti2015-monotonic}
R.~S. Ganti, L.~Balzano, and R.~Willett.
\newblock Matrix completion under monotonic single index models.
\newblock In C.~Cortes, N.~Lawrence, D.~Lee, M.~Sugiyama, and R.~Garnett, editors, {\em Advances in Neural Information Processing Systems 28}, pages 1864--1872. Curran Associates, Inc., 2015.

\bibitem{srebro2003-wlr}
N.~Srebro and T.~Jaakkola.
\newblock Weighted low-rank approximations.
\newblock In {\em Proceedings of the 20th International Conference on Machine Learning (ICML-03)}, pages 720--727, 2003.

\bibitem{wen2012-solving}
Z.~Wen, W.~Yin, and Y.~Zhang.
\newblock Solving a low-rank factorization model for matrix completion by a nonlinear successive over-relaxation algorithm.
\newblock {\em Mathematical Programming Computation}, 4:333--361, 2012.

\bibitem{tanner2016-low}
J.~Tanner and K.~Wei.
\newblock Low rank matrix completion by alternating steepest descent methods.
\newblock {\em Applied and Computational Harmonic Analysis}, 40(2):417--429, 2016.

\bibitem{chi2019-nonconvex}
Y.~Chi, Y.~M. Lu, and Y.~Chen.
\newblock Non-convex optimization meets low-rank matrix factorization: an overview.
\newblock {\em {IEEE} Transactions on Signal Processing}, 67(20):5239--5269, 2019.

\bibitem{tenenbaum2000-global}
J.~B. Tenenbaum, V.~de~Silva, and J.~C. Langford.
\newblock A global geometric framework for nonlinear dimensionality reduction.
\newblock {\em Science}, 290:2319--2323, 2000.

\bibitem{weinberger2006-mvu}
K.~Q. Weinberger and L.~K. Saul.
\newblock An introduction to nonlinear dimensionality reduction by maximum variance unfolding.
\newblock In {\em Proceedings of the 21st National Conference on Artificial Intelligence (AAAI-06)}, pages 1683--1686, 2006.

\bibitem{roweis2000-nonlinear}
S.~T. Roweis and L.~K. Saul.
\newblock Nonlinear dimensionality reduction by locally linear embedding.
\newblock {\em Science}, 290:2323--2326, 2000.

\bibitem{chen2018-sparse}
Y.~Chen, D.~M. Paiton, and B.~A. Olshausen.
\newblock The sparse manifold transform.
\newblock In S.~Bengio, H.~Wallach, H.~Larochelle, K.~Grauman, N.~Cesa-Bianchi, and R.~Garnett, editors, {\em Advances in Neural Information Processing Systems 31}, pages 10534--10545. Curran Associates, Inc., 2018.

\bibitem{sengupta2018-tiling}
A.~M. Sengupta, C.~Pehlevan, M.~Tepper, A.~Genkin, and D.~B. Chklovskii.
\newblock Manifold-tiling localized receptive fields are optimal in similarity-preserving neural networks.
\newblock In S.~Bengio, H.~Wallach, H.~Larochelle, K.~Grauman, N.~Cesa-Bianchi, and R.~Garnett, editors, {\em Advances in Neural Information Processing Systems 31}, pages 7080--7090. Curran Associates, Inc., 2018.

\bibitem{paturi1986-probabilistic}
R.~Paturi and J.~Simon.
\newblock Probabilistic communication complexity.
\newblock {\em Journal of Computer and System Sciences}, 33(1):106–--123, 1986.

\bibitem{hatami2022-sign}
H.~Hatami, K.~Hosseini, and S.~Lovett.
\newblock Sign rank vs discrepancy.
\newblock {\em Theory of Computing}, 18(19):1--18, 2022.

\bibitem{sherstov2023-near}
A.~A. Sherstov and P.~Wu.
\newblock Near-optimal lower bounds on the threshold degree and sign-rank of {AC}$^0$.
\newblock {\em SIAM Journal on Computing}, 52(2):\texttt{STOC19}--1--\texttt{STOC19}--86, 2023.

\bibitem{mahoney2010-random}
M.~W. Mahoney.
\newblock Randomized algorithms for matrices and data.
\newblock {\em Foundations and Trends in Machine Learning}, 3(2):123--224, 2010.

\bibitem{halko2011-finding}
N.~Halko, P.~G. Martinsson, and J.~A. Tropp.
\newblock Finding structure with randomness: probabilistic algorithms for constructing approximate matrix decompositions.
\newblock {\em {SIAM} Review}, 53(2):217--288, 2011.

\bibitem{tropp2017-practical}
J.~A. Tropp, A.~Yurtsever, M.~Udell, and V.~Cevher.
\newblock Practical sketching algorithms for low-rank matrix approximation.
\newblock {\em SIAM Journal on Matrix Analysis and Applications}, 38(4):1454--1485, 2017.

\bibitem{ng2001-spectral}
A.~Ng, M.~Jordan, and Y.~Weiss.
\newblock On spectral clustering: analysis and an algorithm.
\newblock In T.~Dietterich, S.~Becker, and Z.~Ghahramani, editors, {\em Advances in Neural Information Processing Systems 14}, pages 849--856, 2001.

\bibitem{vonluxburg2007-tutorial}
U.~Von~Luxburg.
\newblock A tutorial on spectral clustering.
\newblock {\em Statistics and Computing}, 17(4):395--416, 2007.

\bibitem{umeyama1988-eigen}
S.~Umeyama.
\newblock An eigendecomposition approach to weighted graph matching problems.
\newblock {\em IEEE Transactions on Pattern Analysis and Machine Intelligence}, 10(5):695--703, 1988.

\end{thebibliography}
\end{document}